\documentclass[11pt]{article} 

\usepackage[letterpaper,margin=1in]{geometry}

\usepackage[letterpaper,margin=1in]{geometry}

\newif\ifred
\redtrue 

\usepackage[T1]{fontenc}
\usepackage{microtype}

\usepackage{titlesec}
\titleformat*{\paragraph}{\bfseries}

\usepackage{amsthm,amsmath,amssymb,amsfonts,amssymb,mathtools}
\usepackage{amsmath,amssymb,amsfonts,amssymb,mathtools}
\usepackage{xcolor}
\usepackage{empheq}
\usepackage{dsfont}
\usepackage{mathrsfs}
\usepackage{graphicx}
\usepackage{enumitem}
\usepackage{aliascnt} 
\usepackage{xspace}
\usepackage{bbm}

\usepackage{caption}
\usepackage{tikz}
\usetikzlibrary{patterns}
\usetikzlibrary{patterns}
\usetikzlibrary{arrows,shapes,automata,backgrounds,petri,positioning}
\usetikzlibrary{shadows}
\usetikzlibrary{calc}
\usetikzlibrary{spy}
\usetikzlibrary{angles, quotes}
\usetikzlibrary{matrix}
\usepackage{pgf,pgfplots}
\usepackage{pgfmath,pgffor}
\pgfplotsset{compat=1.17}

\usepackage{framed}
 \usepackage{algorithm}
\usepackage[noend]{algpseudocode}

\definecolor[named]{ACMBlue}{cmyk}{1,0.1,0,0.1}
\definecolor[named]{ACMYellow}{cmyk}{0,0.16,1,0}
\definecolor[named]{ACMOrange}{cmyk}{0,0.42,1,0.01}
\definecolor[named]{ACMRed}{cmyk}{0,0.90,0.86,0}
\definecolor[named]{ACMLightBlue}{cmyk}{0.49,0.01,0,0}
\definecolor[named]{ACMGreen}{cmyk}{0.20,0,1,0.19}
\definecolor[named]{ACMPurple}{cmyk}{0.55,1,0,0.15}
\definecolor[named]{ACMDarkBlue}{cmyk}{1,0.58,0,0.21}

\usepackage[colorlinks,citecolor=blue,linkcolor=magenta,bookmarks=true]{hyperref}
 \usepackage[nameinlink]{cleveref}
\creflabelformat{ineq}{#2{\upshape(#1)}#3}
\crefname{sub}{Subsection}{Subsection}
\creflabelformat{Subsection}{#2{\upshape(#1)}#3}
\crefname{sdp}{SDP}{SDP}
\creflabelformat{sdp}{#2{\upshape(#1)}#3}
\crefname{lp}{LP}{LP}
\creflabelformat{lp}{#2{\upshape(#1)}#3}
\usepackage{aliascnt}
\usepackage{cleveref}
\crefname{ineq}{Inequality}{Inequality}
\creflabelformat{ineq}{#2{\upshape(#1)}#3}
\crefname{sub}{Subsection}{Subsection}
\creflabelformat{Subsection}{#2{\upshape(#1)}#3}
\crefname{sdp}{SDP}{SDP}
\creflabelformat{sdp}{#2{\upshape(#1)}#3}
\crefname{lp}{LP}{LP}
\creflabelformat{lp}{#2{\upshape(#1)}#3}



\newtheorem{theorem}{Theorem}[section]

\newtheorem{lemma}{Lemma}[section]
\newtheorem{informal theorem}[theorem]{Theorem (informal statement)}

\newtheorem{proposition}[theorem]{Proposition}

\newtheorem{fact}[theorem]{Fact}

\newtheorem{definition}[theorem]{Definition}

\newcommand{\lp}{\left}
\newcommand{\rp}{\right}
\newcommand\norm[1]{\left\| #1 \right\|}

\DeclareMathOperator*{\pr}{\mathbf{Pr}}
\DeclareMathOperator*{\E}{\mathbf{E}}
\newcommand{\proj}{\mathrm{proj}}

\newcommand{\err}{\mathrm{err}}

\newcommand{\R}{\mathbb{R}}
\newcommand{\s}{\mathbb{S}}

\newcommand{\eps}{\epsilon}

\newcommand{\poly}{\mathrm{poly}}

\newcommand{\polylog}{\mathrm{polylog}}

\newcommand{\sgn}{\mathrm{sign}}
\newcommand{\sign}{\mathrm{sign}}

\newcommand{\OPT}{\mathrm{opt}}

\newcommand{\Ind}{\mathds{1}}

\newcommand{\littleint}{\mathop{\textstyle \int}}

\newcommand{\opt}{\mathrm{opt}}

\newcommand{\iid}{{i.i.d.}\ }
\newcommand{\abs}[1]{\lp| #1 \rp|}
\newcommand{\A}{\mathcal{A}}

\newcommand{\relu}{\mathrm{Relu}}

\DeclareMathOperator{\Var}{Var}
\renewcommand\Pr{\pr}

\begin{document}
\title{Robust Regression of General ReLUs with Queries}
\footnotetext[1]{Authors are listed in alphabetical order.}
\author{
Ilias Diakonikolas\thanks{Supported by NSF Medium Award CCF-2107079, 
ONR Award number N00014-25-1-2268, and an H.I. Romnes Faculty Fellowship.}\\
University of Wisconsin-Madison\\
\texttt{ilias@cs.wisc.edu}
\and
Daniel M. Kane\thanks{Supported by NSF Medium Award CCF-2107547.}\\
University of California, San Diego\\
\texttt{dakane@cs.ucsd.edu}
\and
Mingchen Ma\thanks{Supported by NSF Award  CCF-2144298 (CAREER).}\\
University of Wisconsin-Madison\\
\texttt{mingchen@cs.wisc.edu}
}

\maketitle

\begin{abstract}
We study the task of 
agnostically learning general 
(as opposed to homogeneous) ReLUs 
under the Gaussian distribution with respect 
to the squared loss. In the passive learning setting, 
recent work gave a computationally efficient algorithm 
that uses $\poly(d,1/\eps)$ labeled examples 
and outputs a hypothesis with error $O(\opt)+\eps$, 
where $\opt$ is the squared loss of the best fit ReLU.
Here we focus on 
the interactive setting, where the learner 
has some form of query access to the labels of unlabeled 
examples. 
Our main result is the first computationally 
efficient learner  
that uses 
$d \polylog(1/\eps)+\tilde{O}(\min\{1/p, 1/\eps\})$ 
black-box label queries, 
where $p$ is the bias of the target function, and achieves error $O(\opt)+\eps$.
We complement our algorithmic result by showing 
that its query complexity 
bound is qualitatively near-optimal, 
even ignoring computational constraints. 
Finally, we establish that query access 
is essentially necessary 
to improve on the label complexity of passive learning. Specifically, for pool-based active learning, 
any active learner 
requires $\tilde{\Omega}(d/\eps)$ labels, 
unless it draws a super-polynomial number of unlabeled examples. 
\end{abstract}



\setcounter{page}{0}

\thispagestyle{empty}

\newpage

\section{Introduction}

The ReLU activation plays a central role in the design of modern neural 
networks. A ReLU function $\sigma(W\cdot x - t) = \max\{W\cdot x - t,0\}$ 
is specified by a pair of parameters $(W,t)$, 
where $W\in \R^d$ and $t \in \R$. 
ReLU regression is the following basic problem: 
given some form of access to a distribution $D$ 
over $\R^{d} \times \R$, output a ReLU 
with loss that can compete with $\opt$---the loss 
of the best-fit ReLU with respect to $D$. 
This fundamental task has been extensively studied 
in the past decades; see, e.g.,~\cite{kakade2011efficient,frei2020agnostic,
diakonikolas2020approximation,vardi2021learning, 
diakonikolas2022learning,wang2023robustly,
awasthi2022agnostic,guo2024agnostic} and references therein. 
Prior algorithmic work focused on learning from random examples 
(passive learning) and has obtained  
efficient learners with error $O(\opt)+\eps$, 
under the assumption that the marginal distribution $D_x$ is 
well-behaved; in most cases, a standard Gaussian or a structured  
distribution with similar properties. 
Such a benchmark is motivated by known computational hardness 
results for this problem. On the one hand, 
without any assumption on $D_x$, it is computationally intractable 
to achieve error $C\cdot\opt$ for any constant $C>1$ \cite{diakonikolas2022hardness}. On the other hand, 
even under the standard Gaussian, it is computationally hard 
to achieve loss $\opt+\eps$~\cite{diakonikolas2020near, DKPZ21, DKR23}.

Despite the aforementioned long line of work on this problem, 
the number of {\em labeled} examples needed 
to achieve the desired error guarantee remains poorly understood. 
For the special case that the target ReLU has negative threshold (i.e., $t\le 0$)
\cite{diakonikolas2022learningr,wang2023robustly} 
design efficient PAC learning algorithms achieving 
error $O(\opt)+\eps$ 
with a nearly optimal sample complexity of $d\polylog(1/\eps)$. 
The problem becomes much more challenging for {\em general} ReLUs 
(with arbitrary bias $t$) and in particular for $t>0$. 
Recently, \cite{guo2024agnostic,zarifis2025robustly} gave the first 
efficient PAC learning algorithms that achieve $O(\opt)+\eps$ error 
for general $t$, using $\poly(d,1/\eps)$ labeled examples.
In fact, to learn an arbitrary ReLU to error $\eps$, 
even with clean labels, one needs $\tilde{\Theta}(d/\eps)$ 
labeled examples in the passive PAC setting. 
Hence, to obtain improved label 
complexity, one needs to consider stronger data access models that allow some form of adaptive (query) access to the labels.

Learning with queries and (pool-based) active learning 
are powerful models that can be used to reduce the number 
of labeled examples needed for various learning tasks. 
Such models capture the ability 
to perform experiments, 
or the availability of expert advice, 
and are well-motivated by real-worlds 
applications where unlabeled examples are cheap. 
A long line of work has shown that in such interactive 
settings one can significantly reduce the number of labeled 
examples needed for learning in a variety of settings \cite{balcan2007margin,dasgupta2005analysis,balcan2013active,awasthi2017power,yan2017revisiting,shen2021power,diakonikolas2024fast,diakonikolas2024active,kontonis2024active,gajjar2024agnostic,li2025nearoptimal}.

The focus of this paper is on agnostic learning 
with black-box query access to the labels.
Specifically, for $x \in \R^d$, 
we can make a black-box query for its label $y(x)$,
where $y(x)$ is generated adversarially. 

In the context of ReLU regression, in the noiseless 
setting, the examples we are actually using are those with 
non-zero labels. This implies that if we have one example 
with non-zero label, then we only need to make 
$\tilde{O}(d)$ queries in its neighborhood 
to collect $d$ examples with non-zero labels.  
This leads to an algorithm with label complexity of $O(1/p+d)$---beating 
the label complexity of passive learning. 
If $p = \Omega(1)$, existing agnostic learning algorithms robustify the above ideas and 
nearly match such a label complexity, even in the passive setting. 
However, if $p$ is small (corresponding to large $t>0$)
our understanding of the label complexity is limited. 
Specifically, there is a large gap in the label complexity 
of known algorithms for $t>0$ versus $t<0$. 
This discussion motivates the following question:
\begin{quote}
{\em What is the label complexity of the 
general agnostic ReLU regression problem?}
\end{quote}

\paragraph{Our Results} 
Our main result is the first efficient learning algorithm 
that solves the general agnostic ReLU regression problem 
with near-optimal label complexity (see \Cref{th query learning} for the formal statement).

\begin{theorem}[Main Algorithmic Result, Informal] \label{thm:main-alg-inf}
Consider the problem of general agnostic ReLU regression with Gaussian 
marginals. There is an algorithm that makes 
$M=d\polylog(R^2/\eps)+\tilde{O}(\min\{1/p, R^2/\eps\})$ queries, 
runs in $\poly(M)$ time, and outputs a ReLU function 
$h=\sigma(W\cdot x - t)$ such that with high probability, 
$h$ has error $O(\opt)+\eps$. Here, $p$ is the fraction 
of examples that are labeled non-zero by the optimal ReLU 
and $R$ is the upper bound for $\norm{W^*}$.
\end{theorem}

Notably, our algorithmic approach 
solves not only the ReLU regression problem,  
but also develops new techniques that can be used 
to analyze a variety of generalized linear models. 

Unlike pool-based active learning, 
where an algorithm first selects a pool of unlabeled examples 
$S$ from the marginal distribution $D_x$ 
and is allowed to query $y(x)$ for $x \in S$, 
our algorithm makes use of a stronger oracle 
that is allowed to query {\em any} desired point in $\R^d$.

To complement our upper bound, we show that 
the query complexity of our algorithm 
is information-theoretically nearly optimal 
(see \Cref{th lb information} for the formal statement).

\begin{theorem}[Query Lower Bound, Informal] \label{thm:main-lb-inf}
Consider the problem of general agnostic ReLU regression with Gaussian 
marginals. Suppose that the optimal ReLU has bias $p$ and $\norm{W^*} = 1$. 
Any learning algorithm that learns a hypothesis $\hat{h}$ 
with error $\tilde{O}(p)$ and succeeds with probability $1/3$ 
must make $\Tilde{\Omega}(1/p^{1-o_d(1)}+d)$ queries, 
even if $\opt\ll p$.
\end{theorem}

Our final lower bound shows that unless the unlabeled dataset is extremely large, 
no pool-based active learning algorithm is able 
to achieve the label complexity of our algorithm, 
even in the realizable setting. 
This establishes a sharp separation between 
pool-based active regression and query learning, 
and resolves the label complexity of ReLU regression 
(see \Cref{th lb compute} for the formal statement).

\begin{theorem} [Active Learning Lower Bound, Informal]
\label{thm:lb-active-inf}
Consider the problem of general (realizable) ReLU regression with Gaussian marginal. 
Suppose the optimal ReLU has bias $p$. 
Any pool-based active learning algorithm that learns 
a hypothesis $\hat{h}$ with error $\tilde{O}(p)$ from a pool of
$m$ unlabled examples drawn from $N(0,I)$
and succeeds with probability $1/3$ must 
make $\Tilde{\Omega}(d/(p\log(m)))$ queries.
\end{theorem}

\paragraph{Preliminaries and Notation} 
Here we record the problem definition and basic notation.

\begin{definition}[Agnostic ReLU Regression with Queries]
 Let $\sigma(z)= \max\{z,0\}$ be the ReLU function.
 A labeling function $y(x): \R^d \to \R$ is a random function that maps each $x \in \R^d$ to an unknown real-valued random variable. 
    For each $h: \R^d \to \R$, denote by $\err(h) = \E_{x \sim N(0,I)} \left(h(x)-y(x)\right)^2/2$, 
    $\opt:= \min_{W: \norm{W} \le R, t\ge 0} \err(\sigma(W\cdot x-t))$ and $\sigma^*(x) = \sigma(W^*\cdot x-t^*)$ be any ReLU with error $\opt$. 
    A query takes $x \in \R^d$ as an input and returns a label $y \sim y(x)$. We say that a learning algorithm $\A$ is a constant-factor approximate learner if for every labeling function $y(x)$, and for every $\epsilon,\delta \in (0,1)$, it outputs some hypothesis $\hat{h}: \R^d \to \R$ by adaptively making queries, such that with probability at least $1-\delta$, $\err(\hat{h}) \le  O(\opt) + \epsilon$. The query complexity of $\A$ is the total number of queries it uses during the learning process.
\end{definition}
Furthermore, we will without loss of generality assume $\opt \le \eps$, since the final error guarantee is $O(\OPT+\eps)$.
We remark that in some parts of the paper, we will also consider the case where $\sigma$ is not a ReLU but a general function and $W^* \in \s^{d-1}$ (where $\s^{d-1}$ is the unit sphere). 
In that case, we call the problem agnostic learning (spherical) generalized linear model (GLM) and for a hypothesis $\sigma(w\cdot x)$, we use $\err(w)$ to denote its error accordingly.
For a vector $W$ in $\R^d$, we use $\norm{W}$ to denote it $\ell_2$ norm and use the lower case $w$ to denote its direction $W/\norm{W}$. For a one-dimensional standard normal $z\sim N(0,1)$ and $t\ge 0$, we denote by $\Phi(t):= \Pr_{z}(z>t)$ and $\psi(t)$ 
the value of the density function of $z$ at $t$. For a ReLU activation $\sigma(z-t)$, we denote by $V(t):=\E_{z\sim N(0,1)}\sigma^2(z-t)$ its second moment and call $p=\Phi(t)$ its bias.
For a real-valued function $f:\R \to \R$, we denote by $\norm{f}_2^2:=\E_{z \sim N(0,1)} f^2(z)$ its squared $L_2$ norm in Gaussian space and for $a \in [0,1]$, denote by $T_a f(z):= \E_{s\sim N(0,1)} f(az+\sqrt{1-a^2}s)$.

\paragraph{Organization} 
In \Cref{sec sphere}, we consider a special case 
of the problem where $t^*$ is given and 
the optimal vector $W^*$ is restricted over $\s^{d-1}$. 
We propose a framework showing that a simple 
projected gradient-descent method 
converges to a desired solution for learning spherical
GLMs, provided a warm-start. 
As an application,  
we give an efficient algorithm with optimal query complexity 
for ReLU regression when $\norm{W^*},t^*$ are given. 
In \Cref{sec regression}, we tackle the general ReLU 
regression problem, by discussing the technical difficulties 
and presenting key components of the problem, 
combining with the framework of \Cref{sec sphere}. 
In \Cref{sec lower}, we establish our query complexity lower bounds. 

\section{Warm-up: Robustly Learning Spherical GLM}\label{sec sphere}

As a warm-up, we consider the special setting 
where $\norm{W^*}=1$ and $t^*>0$ is known. 
Prior works in the passive setting, 
such as \cite{guo2024agnostic, zarifis2025robustly}, 
reduce the problem to this case 
by guessing $(\norm{W^*},t^*)$ in a brute-force way. 
There are two motivations for studying such a setting. 
First, as we will discuss in \Cref{sec regression}, 
although such a special case does not capture 
the main difficulty of obtaining query-optimal algorithms, 
it provides important technical components to achieve this goal. 
Second, when $(\norm{W^*},t^*)$ are known, 
the problem becomes a special case of agnostic learning 
of spherical GLMs, where $\sigma$ is a general activation 
function and is known \emph{precisely} to the learner. 

In this section, we start by analyzing the task of 
agnostic learning spherical GLMs in the passive learning setting. 
We present several technical tools that we develop to 
solve the ReLU regression problem in this special case 
with an optimal query complexity.
Unlike prior works, such as \cite{guo2024agnostic}, 
which designed a complicated update rule for $w$ 
(the current direction), we instead focus on the 
following simple projected gradient descent method: 
\begin{align}\label{eq update}
    w \gets \proj_{\s^{d-1}} \left( w - \mu \proj_{w^\perp} \nabla_w \err(w) \right).
\end{align}
We summarize our main technical contribution in this section informally as follows.
\begin{theorem}\label{th glm}
    Let $\sigma: \R \to \R$ be any activation function such that $\norm{\sigma'}_2^2 = L$, $L>0$. Let $D$ be any distribution over $\R^d \times \R$ such that $D_x = N(0,I)$ and $\err(w^*) \le \eps$, where $w^* \in \s^{d-1}$ is a direction that achieves the optimal loss. For $\alpha>1$, suppose we are given any unit vector $w^{0} \in \s^{d-1}$ such that $\norm{T_{\sqrt{\cos \theta_0}}\sigma'}_2^2 \ge \norm{\sigma'}_2^2 /\alpha$, where $\theta_0=\theta(w^{0},w^*)$. Starting from $w^{0}$, the update rule \eqref{eq update} gives a direction $\hat{w}$ with error $O(\alpha^2\eps)$. 
\end{theorem}
Roughly speaking, we show that the update rule \eqref{eq update} has an \emph{initialization-dependent} error guarantee, which holds for even very complicated activation functions that are non-monotone. The quality of the initialization is measured by the ratio $\alpha:=\norm{\sigma'}_2^2/\norm{T_{\sqrt{\cos \theta_0}}\sigma'}_2^2$. As we will show later, for general ReLU activations, we are able to efficiently get a $w^{0}$ with $\alpha=O(1)$, which implies a solution with error $O(\eps)$ if we can implement $\eqref{eq update}$ with high accuracy. 

We defer the 
technical details and proofs of this section to \Cref{app sphere}.
To analyze the error guarantee, we start with the following simple observation: 
\begin{align*}
    \err(w) \le \E_{x} \left(\sigma(w\cdot x)-\sigma(w^*\cdot x) \right)^2 + \E_{x} \left(\sigma(w^*\cdot x)-y \right)^2 \le 2\opt + \E_{x} \left(\sigma(w\cdot x)-\sigma(w^*\cdot x) \right)^2.
\end{align*}
So,
the central part of the analysis relies on characterizing 
the noiseless error of a $\sigma(w\cdot x)$, $\ell(w):=\E_{x\sim N(0,I)} \left(\sigma(w\cdot x)-\sigma(w^*\cdot x)\right)^2/2$. 
Prior works usually analyzed $\ell(w)$ via the angle $\theta(w,w^*)$. For example, for the problem of learning homogeneous halfspaces, $\ell(w) = \theta /\pi$. However, in general, $\ell(w)$ does not have such a simple closed form. 
In the following lemma, we give an integral expression for $\ell(w)$ over the unit sphere by drawing a 
connection with the Ornstein–Uhlenbeck semi-group. 
\begin{lemma}[Noiseless Error Estimation over the Sphere]\label{lm noiseless}
 Let $\sigma: \R \to \R$ be any activation function such that $\sigma' \in L_2(N(0,I))$ and let $w \in \s^{d-1}$ be any unit vector such that $\theta:=\theta(w,w^*) < \pi/2$. Then 
 $
     \ell(w) = \littleint_0^\theta \sin s \norm{T_{\sqrt{\cos s}} \sigma'}_2^2 ds \le (\pi/2) \sin^2 \theta \norm{\sigma'}_2^2.
 $
\end{lemma}
To relate \Cref{lm noiseless} with the update rule \eqref{eq update}, we use the following two structural 
lemmas that characterize the progress 
as well as the noise level during the update in each round. 

\begin{lemma}\label{lm signal spherical}
    Let $\sigma: \R \to \R$ be any activation function such that $\sigma' \in L_2(N(0,I))$ and let $w \in \s^{d-1}$ be any unit vector such that $\theta=\theta(w,w^*)<\pi/2$, where $\err(\sigma(w^*)) = \opt$. Write $w^* = aw+bu$, where $u \in \s^{d-1}, u \perp w, a,b \ge 0, a^2+b^2 = 1$,
    then, $\proj_{w^\perp} \nabla_w \ell(w) = -b \norm{T_{\sqrt{a}}\sigma'}_2^2 u$.
\end{lemma}

\begin{lemma}\label{lm noise control}
    Let $\sigma: \R \to \R$ be any activation function such that $\sigma' \in L_2(N(0,I))$ and let $w \in \s^{d-1}$ be any unit vector such that $\theta=\theta(w,w^*)<\pi/2$, where $\err(\sigma(w^*)) = \opt \le \eps$. Then for any $v \in \s^{d-1}$ and $v \perp w$,
    $ \abs{\proj_{w^\perp} \left( \nabla_w\err(w) - \nabla_w \ell(w) \right) \cdot v} \le \sqrt{\eps}\norm{\sigma'}_2$.
\end{lemma}

The intuition here is that in each update round, 
if the length of the gradient $b \norm{T_{\sqrt{a}}\sigma'}_2^2$ 
is larger than the noise level $\sqrt{\eps}\norm{\sigma'}_2$, 
then the projected gradient descent approach 
is able to improve the angle between $w$ and $w^*$. 
We summarize this as \Cref{lm gradient descent}.
\begin{lemma}[Angle Contraction]\label{lm gradient descent}
    Let $w^*,w^{(i)} \in \s^{d-1}$ such that $w^* = a w^{(i)}+b u$, where $u \in \s^{d-1}, u \perp w^{(i)}, a,b \ge 0, a^2+b^2 = 1$. Let $\theta_i=\theta(w^{(i)},w^*)$. 
    Let $G \in \R^d$ be a random vector such that with probability $1$, $G \perp w^{(i)}$. Let $g$ be the mean of $G$ and $\hat{g} \in \R^d$.
    Suppose there is some $c>0$ such that $g \cdot u \ge cb/10, \norm{g} \le cb, \norm{g-\hat{g}} \le bc/40$, then by setting $\mu =c/20 $, the update rule $w^{(i+1)} = \proj_{\s^{d-1}} (w^{(i)} + \mu\hat{g})$ satisfies
    $\sin (\theta_{i+1}/2)  \le \sqrt{1-\left(\frac{c}{20}\right)^2}\sin (\theta_i/2)$.
\end{lemma}

However, the length of the signal 
$\sin \theta \norm{T_{\sqrt{\cos \theta}} \sigma'}_2^2$ 
in general is not an increasing function of $\theta$. 
If $\theta_0$ is not small enough, the noise level 
could be too high to make the gradient point in the correct direction, 
making it impossible to reach a desirable error guarantee. 
On the other hand, by \Cref{lm noiseless}, 
if $\ell(w) \ge \Omega(\alpha^2\eps)$, then the noise level $\abs{\proj_{w^\perp} \left( \nabla_w\err(w) - \nabla_w \ell(w) \right) \cdot v}$ 
is at most $\sin\theta\alpha\norm{\sigma'}_2^2$. 
Since $\norm{T_{\sqrt{\cos\theta}}\sigma'}$ is decreasing in $\theta$, this implies that, as long as we have a good initialization, 
we can make progress and reach a good solution. 
We summarize this property in the following lemma.

\begin{lemma}\label{lm stopping condition}
     Let $\sigma: \R \to \R$ be any activation function such that $\sigma' \in L_2(N(0,I))$. Let $\alpha>1$ and $0<\theta_0<\pi/2$ such that $\norm{T_{\sqrt{\cos \theta_0}}\sigma'}_2^2 \ge \norm{\sigma'}_2^2 /\alpha$. Let 
     $w \in \s^{d-1}$ be any unit vector such that $\theta=\theta(w,w^*)<\theta_0$, where $\err(\sigma(w^*)) = \opt \le \eps$. 
    If $\sin^2 \theta \norm{\sigma'}_2^2 \ge 20 \alpha^2 \eps/ \pi$, then for any $v \in \s^{d-1}$ and $v \perp w$,
    $ \norm{\proj_{w^\perp} \left( \nabla_w\err(w) -\nabla_w \ell(w) \right)} \le \norm{\proj_{w^\perp} \nabla_w \ell(w)}/20 $. Furthermore, if $ \norm{\proj_{w^\perp} \left( \nabla_w\err(w) -\nabla_w \ell(w) \right)} > \norm{\proj_{w^\perp} \nabla_w \ell(w)}/20 $, then $\err(w) \le O(\alpha^2\eps).$
\end{lemma}

\subsection{Application: Query Complexity of Agnostic Learning Spherical ReLU with Known Bias}
As a direct application of the above framework, we show when $w^* \in \s^{d-1}$ and $t^*$ is given, how to solve the general ReLU regression problem with an optimal query complexity. As discussed above, to ensure we get a solution with error $O(\eps)$, we need some $w^{0}$ with $\norm{\sigma'}_2^2/\norm{T_{\sqrt{\cos \theta_0}}\sigma'}_2^2 = O(1)$. To characterize such a $\theta_0$, we present the following structural 
lemma.
\begin{lemma}\label{lm initialization}
    Let $\sigma$ be an activation function of the form $\sigma(z)=\relu(z -t^*)$, where $t^*>0$. If $\sin \theta/2 \le 1/t^*$, then $\norm{T_{\sqrt{\cos \theta}} \sigma'}_2^2 \ge \norm{\sigma'}_2^2 / 50$.
\end{lemma}
This implies that $\theta_0 = O(1/t^*)$ is sufficient for us to reach some $w$ with $O(\eps)$ error. But to obtain such a $w^{(0)}$ is also a challenging problem. Our initialization 
is motivated by the following lemma.
\begin{lemma}\label{lm reduction}
Let $\sigma$ be an activation f the form $\sigma(z)=\relu(z -t^*)$, where $t^*>0$. Let $y(x)$ be any labeling function such that $\opt \le \eps$.
Let $c>0$ be a suitably small constant.
If $V(t^*)>C \eps$, for some large constant $C$, then $\Pr_{x \sim N(0,I)} \left(\bar{y}(x) \neq \sign(w^*\cdot x - t^*) \right) \le \Phi(t^*)/C'$ for some large constant $C'>0$, where  $\Bar{y}(x): = \Ind\{y(x)>c/(t^*)\}$.   
\end{lemma}
That is to say, as long as $t^*$ is not too large 
to make $h \equiv 0$ have error $O(\eps)$, 
the truncated label $\Bar{y}(x): = \Ind\{y(x)>c/t^*\}$ 
can be seen as a labeling generated by the 
halfspace $h(x) = \sign(w^* \cdot x -t^*)$ 
corrupted with $\eta$-level adversarial label noise, 
such that $\eta/\Phi(t^*) < 1/C$ 
for some sufficiently large constant $C$. 
Such an observation is useful, as it allows us to make 
use of the recent technique developed in \cite{diakonikolas2024active} to obtain the warm-start using only $\Tilde{O}(1/p+d)$ queries.
\begin{lemma}\label{lm halfspace}[Halfspace Initialization via Queries (Theorem 3.8, Theorem F.1 \cite{diakonikolas2024active})]
    Let $h^*(x) = \sgn(w^*\cdot x-t^*)$, where $t^* \le O(\sqrt{\log(1/\eps)})$ and $y(x): \R^d \to {\pm 1}$ be any labeling function such that $\Pr_{x \sim N(0,I)} \left(h^*(x) \neq y(x) \right) \le \Phi(t^*)/C'$ for some large enough constant $C'$. There is an algorithm such that given some $t \in \R$ with $\abs{t-t^*} \le 1/\log(1/\eps)$, it makes $M=\Tilde{O}(1/\Phi(t)+d\log(1/\eps))$ queries, runs in $\poly(d,M)$ time, and with probability at least $1/\log(1/\Phi(t))$,
outputs some $w^{(0)}$ such that $\sin (\theta(w^{(0)},w^*)/2) \le \min\{1/t,1/2\}$.
\end{lemma}
However, obtaining such a $w^{(0)}$ is not enough for us to solve the problem with a small label complexity for the following reason: 
\Cref{lm stopping condition} does not bound 
the number of iterations and the number of labeled examples 
needed to reach a good solution. Specifically, 
by \Cref{lm gradient descent}, the progress made in each round is characterized by the length of the signal $g\cdot u$. 
If $g \cdot u \approx c \sin \theta$, then the angle decreases 
by a factor of $(1-c)$ in each round. If $t^* = 0$, 
prior works have shown that $c = \Omega(1)$ and $\Tilde{O}(d)$ 
labeled examples are enough to get a good estimation of $g$. 
The landscape changes dramatically when $t^*$ becomes large. 
As indicated by \Cref{lm signal spherical}, for large $t^*$, 
the length of the gradient is a most $\norm{\sigma'}_2^2$, 
which can be as small as $\poly(\eps)$. 
This implies that, unless we estimate the gradient 
to very high accuracy and rescale the gradient to a constant length, 
the progress made in each round is too small. 
Unfortunately, as the variance of the gradient 
is much larger than the accuracy we need, 
this blows up the total query complexity. 

To overcome this difficulty, we show that 
for the problem of ReLU regression, 
the initialization guarantee $\theta_0\le O(1/t^*)$ 
not only lets us make progress in each round, 
but it also allows us to use queries to boost 
the length of the gradient while maintain a small variance. 
Intuitively, every time we estimate the gradient, 
examples with $w^*\cdot x>t^*$ contribute most of the gradient. 
When our current $w^{(i)}$ is close to $w^*$, 
the regions $\{x\mid w^{(i)}\cdot x>t^* \}$ 
and $\{x\mid w^*\cdot x>t^* \}$ have significant intersection. 
Thus, we can boost the length of the gradient 
by querying examples in $\{x\mid w^{(i)}\cdot x>t^* \}$. 
Furthermore, due to the Lipschitz continuity of the ReLU, 
we can maintain small variance 
for the gradient; and thus $\Tilde{O}(d)$ examples 
suffice for us to accurately estimate the gradient. 
\begin{lemma}\label{lm initialization variance}
Let $\sigma(z)=\relu(z -t^*)$, with $t^*>0$. 
Let $y(x)$ be any labeling function 
such that $\opt \le \eps$. Let $w \in \s^{d-1}$ be any vector 
such that $\sin (\theta/2) \le 1/t^*$. 
Denote by $G^* \in \R^d$ the random vector 
$\left(\sigma(w \cdot x-t^*) - \sigma(w^* \cdot x-t^*)\right) \proj_{w^\perp} x$ and $G$ 
the random vector $\left(\sigma(w \cdot x-t^*) - y(x)\right)\proj_{w^\perp} x$, 
where $x \sim N(0,I) \mid_{\{x \mid w \cdot x> t^*\}}$. 
Then the following holds: 
$(1)$ $\E G^* = b \norm{T_{\sqrt{a}}\sigma'}_2^2 u/\Phi(t^*)$; 
$(2)$ $\abs{\left( \E G^* - \E G \right) \cdot v} \le \sqrt{\eps}\norm{\sigma'}_2/\Phi(t^*)$' 
$(3)$ If $\sin^2\theta \Phi(t^*)>\eps$, 
then $\E (G \cdot v)^2 \le \tilde{O}(b^2), \forall v \in \s^{d-1}$.
\end{lemma}

The structural results we obtained so far are almost all 
we need to get a query-optimal algorithm for the spherical case. 
We refer the reader to \Cref{app sphere} for the detailed proof. 
A remaining caveat is that the initialization algorithm 
for halfspace learning used in \cite{diakonikolas2024active} only 
succeeds with $1/\log(1/\eps)$ probability in the worst case. This 
implies that to succeed with a good probability, we need to run the same 
algorithm $\poly\log(1/\eps)$ times, which will give us a list of 
$\poly\log(1/\eps)$ candidate hypotheses. 
The usual way to find a desired hypothesis 
from the list is to check their empirical errors via labeled examples. 
However, this approach will take $\Omega(1/\eps^2)$ labeled examples, 
blowing up the query complexity we have achieved so far. 

To avoid this, 
we design a new active testing procedure that only takes 
$\poly\log(1/\eps)$ queries and selects a hypothesis with error 
$O(\eps)$. Our procedure not only works for selecting ReLU 
activations, but also applies for much general settings.
\begin{lemma}[Hypothesis Selection with Queries]\label{lm test}
Let $D$ be a distribution over $\R^d \times \R$ and let $D_x$ be the marginal distribution of $D_x$.
    There is an algorithm that, on input a list of hypotheses $h_1,\dots,h_k$ such that for $i \in [k]$, $h_i: \R^d \to \R, \E_{x\sim D} h^2_i(x)$ exists, it makes $\poly(k)$ queries and returns a hypothesis $\hat{h}$ such that 
    $\err(\hat{h}) \le O(\min_{i \in [k]} \E_{(x,y)\sim D}(y-h_i(x))^2)$.
\end{lemma}
The intuition of our hypothesis selection algorithm is as follows. 
For each pair $(i,j)$, we would like to check whether 
$\norm{y-h_i}_2\le \norm{y-h_j}_2$, 
which is equivalent to checking the sign of 
the correlation $\E[y(h_i(x)-h_j(x))]$, if $\|h_1\|=\|h_2\|$. Naively, 
this can be done by randomly querying $x\sim D_x$; 
but since the variance of $y(h_i(x)-h_j(x))$ can be very large, 
this can blow up the query complexity of the algorithm. 
Instead, we reweight $D_x$ according to $(h_i(x)-h_j(x))^2$ 
and use $\tilde{y}(x) \propto y/(h_i(x)-h_j(x))^2$ as our query value. 
Such a modification keeps the mean 
we are interested in, but reduces its variance, 
allowing us to use $\tilde{O}(1)$ queries to solve the task. We defer details of the hypothesis selection algorithm to \Cref{app test}.

\section{Agnostically Learning an Arbitrary ReLU with Near-Optimal Query Complexity}\label{sec regression}

To handle the general case, where $\norm{W^*}$ and $t^*$ are unknown,
we need to overcome several conceptual and technical difficulties. 
Prior works in the passive setting solve the general version of the  
problem by reducing it to the spherical case, 
by guessing $\norm{W^*}$ and $t^*$ up to $\poly(\eps/R)$ additive error, 
running the same algorithm $\poly(R/\eps)$ times 
and doing a hypothesis testing. Unfortunately, 
this simple approach 
is prohibitive in our context because it 
completely ruins the query complexity (even though 
we are able to solve the spherical setting with optimal query 
complexity). To achieve an optimal query complexity for the general problem, there are two main obstacles to overcome. 

The first one is to find the correct way to update parameters.
In the case where $t^* \le 0$, prior works 
avoid guessing $\norm{W^*}$ 
by replacing the update rule \eqref{eq update} 
by a standard gradient descent $ W \gets   W - \mu \nabla_W \err(W),$
and showing that when $\norm{W-W^*}$ is large enough to make the noiseless error larger than $\Omega(\eps)$, 
$\nabla_W \err(W)\cdot (W-W^*) \ge \Omega (\norm{W-W^*}^2)$; 
this leads to a constant factor of decay in $\norm{W-W^*}$. 
However, when there is a large threshold, 
$\norm{W-W^*}$ is not the correct quantity 
that characterizes the noiseless error. 
Consider the optimal hypothesis 
$h^* = \sigma(w^*\cdot x - t^*)$ and two other hypotheses
$h_1 = \sigma(v \cdot x - t^*)$, 
$h_2 = (t^* \xi +1)\sigma(w^* \cdot x - t^*)$, 
where $\xi>0$ and $\sin \theta (v,w^*) \approx \xi$.
When $t^*$ is large, the parameter distance of $h_2$ 
is much larger than that of $h_1$; 
however, it can be checked that the two functions 
have the same noiseless error, 
and thus adding the same level of noise 
can make the two functions indistinguishable. 
The implication of this phenomenon is that 
even if the parameter distance $\norm{W-W^*}$ is large, 
we are not able to guarantee the noise rate 
is small enough to make the fast decay happen. 

The second obstacle is how to implement 
the correct update with a small query complexity.
In the spherical setting, we make use of the fact that 
$\sigma(w^{(i)} \cdot x-t^*) - y(x)$ is small 
for most queries $x$ to make the variance 
of the gradient as small as $\sin^2 \theta_i$ 
(which matches the length of the gradient). 
The small variance is the key that makes $\Tilde{O}(d)$ queries 
sufficient to improve $w^{(i)}$. 
However, since $r^*,t^*$ are not part of the input, 
an inaccurate learned parameter $(r^{(i)},t^{(i)})$ 
could make $h(x)-y(x)$ very large, 
making it impossible to estimate the gradient 
accurately with few queries. So, when we do the parameter update, 
the statistics we rely on must have small variance 
throughout the learning process. 

We will require the following notation. 
For $r>0,w \in \s^{d-1}, t>0$, 
we define hypothesis $h(r,w,t) =  \sigma(rw\cdot x - t).$
In particular, we write the optimal hypothesis 
as $h^* = \sigma(r^*w^*\cdot x - t^*)$. 
We denote by $\bar{t}:=t/r$ the normalized threshold 
of a ReLU and define the noiseless error 
of $h(r,w,t)$ as 
$\ell(r,\bar{w},t) =    \frac{1}{2} \E_{x\sim N(0,I)} \left( \sigma(rw\cdot x-t)- \sigma(r^*w^*\cdot x-t^*)\right)^2.$ 
Our main algorithmic result is an efficient learning algorithm with an optimal query complexity. 


\begin{theorem}\label{th query learning}
     Consider the problem of agnostic general ReLU regression with queries under the Gaussian distribution.  
    There is an algorithm such that for every labeling function $y(x)$ and for every $\eps,\delta \in (0,1)$, it makes $M=\Tilde{O}_\delta(\min\{1/p, R^2/\eps\} + d\cdot\polylog(R^2/\eps))$ queries, runs in $\poly(d,M)$ time,
    where $p = \Phi(\bar{t}^*)$ is the bias of the optimal activation function, 
    and outputs an $\hat{h}$ such that with high probability 
    at least $1-\delta$, $\err(\hat{h}) \le O(\opt)+\eps$.
\end{theorem}

We remark that the dependence on $R^2$ is due to the natural of rescaling of the error parameter.
Similar to the spherical setting, we still need a warm-start in order to 
converge to a good solution. To maintain a truncated label 
to implement \Cref{lm halfspace}, 
some information about $(r^*,t^*)$ is needed. 
Here, we will grid $r \in [0,R]$ and 
$\bar{t} \in [0,O(\sqrt{\log(R^2/\eps)})]$, the normalized threshold, 
to get the initial information. 
However, instead of using a grid of size $\poly(R/\eps)$, 
we maintain a grid of size $\poly\log(R/\eps)$, 
exponentially smaller than the grid-size used in all prior works. 
In particular, for parameter $\bar{t}$, we will set 
$t^{(i)} = (i-1)/\polylog(R/\eps), i = 1,\dots,O(\polylog(R/\eps))$, 
while for parameter $r$, we build a two-level non-uniform grid 
as follows. The first level of the grid is defined 
as $r_i = 2^{i-1}\eps, i = 1,\dots, \log(R^2/\eps)$. 
For each interval $[r_i,r_{i+1}]$, we grid it uniformly 
into $r_{ij} = r_i+ (j-1)r_i/\polylog(R/\eps), j=1,\dots,\polylog(R/\eps)$. 
Such a grid can ensure that one of the grid points $(r,t)$ 
satisfies $r \le r^* \le 2r$ and 
$\abs{\bar{t}-\bar{t}^*} \le \polylog(R/\eps)$. 
We show in \Cref{lm Ini} that such a pair $(r,r\bar{t})$ 
suffices for us to get an initial direction $w^{(0)}$ as a warm start.
\begin{lemma}[Initialization with Raw Knowledge]\label{lm Ini}
     Let $h^* = r^*\sigma(\bar{w^*}\cdot x - \bar{t}^*)$ be the optimal hypothesis. Assuming that $(r^*)^2V(\bar{t}^*) \ge \Omega(\eps)$, there is an algorithm such that
    given parameters $r,t>0$ with 
    $r \le r^* \le 2r$ and $\abs{t-\bar{t}^*} \le 1/\log(R^2/\eps)$, 
    it makes $M=\Tilde{O}(1/p+d\log(R^2/\eps))$ queries, 
    runs in $\poly(d,M)$ time, 
    and with probability at least $1/\log(1/p)$,
outputs some $w^{(0)} \in \s^{d-1}$ such that 
$\sin (\theta( w^{(0)}, w^*)/2) \le \min\{1/\bar{t}^*,1/2\}$.
\end{lemma}
Although the first-level grid $r_i$ is enough for us to get 
a warm start $w^{(0)}$, for technical reasons, 
to obtain an algorithm with an optimal query complexity, 
it turns out that we need better knowledge about $r^*$. 
This is the reason why we use a two-level partition. 
In the rest of the proof, we assume 
that we have $(r_0,w^{(0)},t_0)$ such that 
$\abs{r_0 - r^*} \le r^*/\polylog(R^2/\eps), \sin(\theta_0/2) \le 1/\bar{t}^*$ and $\abs{\bar{t}_0-\bar{t}^*} \le \polylog(R^2/\eps)$. 

We next provide an overview of our algorithm based 
on this warm start $(r_0,w^{(0)},t_0)$. 
As we mentioned earlier, since $\norm{W-W^*}$ is not 
the correct measure for the noiseless error, 
the noise rate could be large even if $\norm{W-W^*}$ is large. 
To overcome this issue, in each round of the algorithm, 
we modify the standard gradient descent 
by decomposing the update into two directions; 
along direction $w^{(i)}$ and orthogonal to $w^{(i)}$. 
In other words, we use different statistics 
to update $(r_i,t_i)$ and $w^{(i)}$ separately in a careful way. 
The motivation for using such a strategy is due 
to the following noiseless error decomposition, \Cref{lm realizable error}, which implies that the noiseless error 
can be decomposed into two terms that are independent of each other. Importantly, as long as one of the two terms is suboptimal, 
we are able to improve it despite the noise from the other direction. 

\begin{lemma}\label{lm realizable error}
Let $r>0,w \in \s^{d-1}, t>0$. Then 
$\ell(r,w,t) \le (r^*)^2\int_0^\theta \sin s \norm{T_{\sqrt{\cos s}} \sigma'(z-\bar{t}^*)}_2^2 ds + \E_{z\sim N(0,1)}\left(\sigma(rz - t) - \sigma(r^*z - t^*)\right)^2.$
\end{lemma}

Furthermore, the errors from $(r_i,t_i)$, $w^{(i)}$ 
can be entangled with each other, which makes the analysis subtle. 
In particular, the error from one term can make the statistics 
we use for updating the other terms have a large variance, 
which could blow-up the query complexity we need. 
The second motivation of such a strategy is that 
it provides a clean way to analyze the error 
entanglement between the two terms. 
We list the algorithm below and describe the update methods we use. 

We defer the details of the analysis to \Cref{app overview}.

\begin{algorithm}
		\caption{\textsc{QueryLearning}(Learn optimal ReLU with a warm start)}\label{alg query learning 1}
		\begin{algorithmic} [1]
\State\textbf{Input:} $w^{(0)} \in \s^{d-1}:$ unit vector such that $\theta_0 \le 1/\polylog(R^2/\eps)$. $r_0>0:$ such that $\abs{r^*-r_0}\le r^*/\polylog(R^2/\eps)$, $t_0:\abs{t_0-t^*}\le \polylog(R^2/\eps)$.
\State\textbf{Output:} $\hat{h}: \R^d \to \R$, such that $\err(\hat{h}) \le O(\eps)$ with non-trivial probability.
\State $B_0: = r_0^2/\polylog(R^2/\eps)$
\For{$i=0,\dots,T-1$}
\State Generate $\polylog(R^2/\eps)$ samples $x^{(j)} \sim N(0,I) \mid \{ \bar{w^{(i)}} \cdot x > t_i\}$. Query $y(x^{(j)})$ and use them to get an estimate $\hat{g}_i$ for $(\E U_i,\E F_i)$.
\If{$\norm{\hat{g}_i} \ge B_i\polylog(R^2/\eps)$}
\State Set $(r_{i0},t_{i0}) = (r_i,t_i)$ and update $(r_{ij},t_{ij}) = (r_{i(j-1)},t_{i(j-1)})- (\hat{g_{ij}})/\polylog(R^2/\eps)$ until $\norm{\hat{g}_{ij}} \le B_i\polylog(R^2/\eps)$
\EndIf
\State $(r_{i+1},t_{i+1}) \gets (r_{ij},t_{ij})$

\For{$j=1,\dots,\Tilde{O}(d)$}
\State Generate $x^{(j)} \sim N(0,I) \mid \{ w^{(i+1)} \cdot x > t_{i+1}\}$ and query $y(x^{(j)})$
\EndFor
\State Estimate $\E G_i$ via median of mean and get $\hat{G}_{i}$
\State $w^{(i+1)} = \proj_{\s^{d-1}} \left( w^{(i)} - \mu \hat{G}_{i} \right)$
\State $B_{i+1} = (1-\rho) B_i$
\EndFor
\State $\hat{w} = w^{(T)}$
\State Build a unit grid of size $1/\polylog(R^2/\eps)$ over the ball centered at $(r_T,t_T)$ and randomly select a pair $(\hat{r},\hat{t})$ from the grid.
\State\Return $h(\hat{r},\hat{w},\hat{t})$
\bigskip 
\end{algorithmic}
\end{algorithm}

\paragraph{Angle Update} Similar to the spherical setting, we 
construct the gradient by boosting 
$\proj_{w^\perp} \nabla_w \err(w)$ via a rejection sampling approach. 
Given a ReLU activation, $h(r_i,w^{(i)},t_i)$, 
define random vectors 
$G^*_i:= \left( h(r_i,w^{(i)},t_i)(x) - h^*(x) \right)\proj_{(w^{(i)})^\perp}(x)$ 
and its noisy version 
$$G_i:= \left( h(r_i,w^{(i)},t_i)(x) - y(x) \right)\proj_{(w^{(i)})^\perp}(x) \;,$$  
where $x \sim N(0,I) \mid_{w^{(i)} \cdot x > \bar{t}_i}$.

To begin with, we present the following lemma that quantifies $G_i$.

\begin{lemma}\label{lm accurate update}
Let $h(r_i,\bar{w}^{(i)},t_i)$.
Write $w^* = a w^{(i)} + bu$, where $a,b>0, a^2+b^2 = 1$, 
$u \in \s^{d-1}, u \perp \bar{w}^{(i)}$. 
Then the following holds: 
\begin{enumerate}
\item[(1)] If $\abs{\bar{t}_i-\bar{t}^*} \le 1/\log(R^2/\eps)$ 
and $b \le 1/\bar{t_i}$, 
then $\E G^*_i = -\alpha br^* \norm{T_a \sigma'(z-\bar{t}_i)}^2 u / \Phi(\bar{t}_i),$
where $1/2< \alpha < 2$. 
\item[(2)] $\abs{\left( \E G^* - \E G \right) \cdot v} \le \sqrt{\eps}\norm{\sigma'(z-t_i)}_2/\Phi(\bar{t}_i)$, 
$\forall v \in \s^{d-1}, v \perp w^{(i)}$.
\end{enumerate}
\end{lemma}
\Cref{lm accurate update} implies that the error 
from $(r_i,t_i)$ does not affect the direction of the update $\E G$. 
Combining \Cref{lm noiseless} and \Cref{lm initialization}, 
it follows that as long as $\theta_i$ is large 
and contributes $\Omega(\eps)$ to the noiseless error, 
the gradient $G_i$ we construct can improve the angle. 
However, this does not imply that $\E G_i$ can be estimated 
with only a few queries. To that end, we use 
the following lemma that quantifies the variance of $G_i$. 

\begin{lemma}\label{lm angle variance}
Let $h(r_i,\bar{w}^{(i)},t_i)$ be a ReLU activation.
Write $w^* = a w^{(i)} + bu$, where $a,b>0, a^2+b^2 = 1$, $u \in \s^{d-1}, u \perp \bar{w}^{(i)}$.
     If $\abs{\bar{t}_i-\bar{t}^*} \le 1/\log(R^2/\eps)$, $C\eps/((r^*)^2\Phi(\bar{t}^*)) \le b^2 \le 1/t^2_i$,
     then $\E  (G_i \cdot v)^2 \le \Tilde{O}\left( \E_{z\sim N(0,I)}\left( \sigma(r_iz-t_i)- \sigma(r^*z-t^*)\right)^2/\Phi(\bar{t}_i) +(r^*)^2b^2\right).$
\end{lemma}

Unlike the spherical setting, 
the variance of $G_i$ depends on the accuracy of $(r_i,t_i)$. 
In particular, by \Cref{lm realizable error}, 
the contribution to the variance from $(r_i,t_i)$ 
is proportional to its contribution to the noiseless error. 
When $(r_i,t_i)$ contributes more noiseless error than $w^{(i)}$, 
$w^{(i)}$ might be updated incorrectly due to the estimation error. 

To overcome this difficulty, we make use of a potential analysis by maintaining a \emph{suitably small} upper bound for $(r^*)^2\sin^2\theta_i$ and reduce the upper bound in each round of our algorithm by considering the update $w^{(i+1)} = \proj_{\s^{d-1}}(w^{(i+1)}-\rho'\hat{G}_i)$ for some suitably small $\rho'$, where $\hat{G}_i$ is an estimation for $\E G_i$ with $\tilde{O}(d)$ queries. We will show that as long as the error from $(r_i,t_i)$ is within 
a $\polylog(R^2/\eps)$ factor of $B_i^2$, $d\polylog(R^2/\eps)$ queries are enough to estimate $\E G_i$ and decrease the angle $\theta_i$. So, the key to obtain the correct query complexity is to update $(r_i,t_i)$ correctly, so that it introduces small error throughout the implementation of \Cref{alg query learning 1}.

\paragraph{$(r,t)$ Update} 
We will next describe 
the way we update $(r,t)$. 
To begin with, we briefly explain the technical difficulty 
that needs to be handled. To simplify the intuition, 
we consider a simple one-dimensional setting 
where the optimal hypothesis is $r^*\sigma(z-\bar{t^*})$, 
and we want to learn $(r^*,t^*)$. An immediate observation is 
that if we know $\bar{t}^*$ to high accuracy, 
then applying a binary search over $r$ via querying examples with $z>\bar{t}^*$, we can learn $r^*$ efficiently. 
However, learning $t^*$ with few queries 
with an inaccurate learned parameter $r$ is challenging, 
because even estimating the bias of the target ReLU needs many samples. 
In halfspace learning, \cite{diakonikolas2024active} 
overcomes the difficulty by querying examples 
for which $\abs{z-\bar{t}}<B\ll 1$. Such a method can 
zoom the error from $(\bar{t}-\bar{t}^*)$ and $(\bar{t}-\bar{t}^*)/B$, 
which allows us to make a binary search to find $\bar{t}^*$ 
with very few queries. Unfortunately, this method is not robust 
to label noise in $L_2$ loss. Intuitively, queries 
are more sensitive when the labels are real-valued instead of binary. 
For the $L_2$ loss, the noise rate could be very high 
within any small region, making this approach fail.

To bypass this difficulty, for a given ReLU activation 
$h(r_i,w^{(i)},t_i)$, we consider the following two quantities: 
$U_i^*:= (h(r_i,w^{(i)},t_i) - h^*)(w^{(i)}\cdot x), F_i^*:= -(h(r_i,w^{(i)},t_i) - h^*)$,
and their noisy version $U,F$. 
Here we consider $x\sim N(0,I) \mid \{w^{(i)} \cdot x > \bar{t}_i\}$. 
Intuitively, if $(r_i,t_i)$ are close 
to $(r^*,t^*)$, the task for optimizing them can be 
approximately seen as optimizing the following quadratic function 
$Z(r,t):= \E_{x \sim N(0,I)}\left((rz-t)-(r^*z-t^*)\right)^2\Ind(z>\bar{t}^*)$. 
Such a quantity nearly characterizes the contribution of $(r_i,t_i)$ 
to the noiseless error as well as the contribution of $(r_i,t_i)$ 
to the variance of the gradient we use for updating $w^{(i)}$. Furthermore, 
if $\theta_i$ has already been updated in a reasonable range, 
then $(\E U_i, \E F_i)$ is very close to the gradient of $Z(r_i,t_i)$. This gives the intuition to estimate $(\E U_i, \E F_i)$ and run a standard gradient descent update to improve $(r_i,t_i)$. 

To formalize this intuition, we need to overcome two technical challenges. First, 
unlike the angle update, where the error from $(r_i,t_i)$ 
does not affect the update direction $\E G_i$, 
the mismatch of $w^{(i)}, w^*$ does affect $\E U^*_i, \E F^*_i$, 
even if we ignore the noise and estimation error. 
This is because the mismatch between $w^{(i)}$ and $w^*$ 
also contributes to $U_i,F_i$, 
which could vanish or even reverse the gradient for updating $(r_i,t_i)$.
We defer the quantitative evaluation for the noise level 
and variance for $U_i,F_i$ to \Cref{app regression}. 
Fortunately, such forms of error are only $O(r^*\sin\theta_i)$, 
which means that as long as $w^{(i)}$ is updated 
to a reasonable range, we are safe to update $(r_i,t_i)$. So, we will simultaneously
maintain $w^{(i)}$ within a reasonable region and only update $(r_i,t_i)$, when
the angle is within a reasonable accuracy. The second technical  issue arises because when the true threshold is large, $Z(r,t)$ is an ill-conditioned function. Due to the presence of noise, the ill-condition of $Z(r,t)$ implies that even if $w^{(T)}$ has already been updated to a desirable accuracy, the gradient descent update can only guarantee that 
$(r_T,t_T)$ is $O(\eps\polylog(R^2/\eps)/\Phi(\bar{t^*}))$ 
close to $(r^*,t^*)$ in terms of squared norm. Fortunately, 
there are only two parameters we need to worry about 
and they are already very close to $(r^*,t^*)$; 
randomly selecting a pair of parameters from their neighborhoods 
gives us a good hypothesis with enough probability.

\section{Label Complexity Lower Bounds And the Necessity of Queries}\label{sec lower}

Here we present our label complexity lower bounds.  
We defer the detailed proofs in this section to \Cref{app lb}.
We start with an information-theoretic lower bound in the query setting. 
\begin{theorem}[Query Complexity Lower Bound]\label{th lb information}
Consider the problem of agnostic ReLU regression with queries with a restriction that the optimal ReLU satisfies $\norm{W^*}\le 1$ and has bias at least $p$. Any learning algorithm that outputs a hypothesis with error less than $O(p/\log^2(p))$ with probability $1/3$, must make at least $\Tilde{\Omega}(1/p^{1-o_d(1)} +d)$ queries. Furthermore, this holds even if $\opt \le 2^{-\Omega(d^{1/4})}p$.
\end{theorem}

The proof of \Cref{th lb information} can be broken down into two parts. 
We first consider the lower bound of $\Omega(d)$. 
Such a lower bound even holds for an easier problem, 
which is the standard linear regression problem in the realizable setting.
Suppose that we have made $r$ queries $x^{(1)},\dots,x^{(r)}$ so far, 
and denote by $L$ the subspace spanned by them. 
Consider a Bayesian setting, where $w^*\sim \s^{d-1}$. 
Suppose we know $w^*_L$, then for a new example $x$, 
by symmetry, no hypothesis will have better error 
than the hypothesis $w^*_L\cdot x$. 
In particular, if $x\sim N(0,I)$, 
this implies that no hypothesis has error better than $1-\norm{w^*_L}_2^2$.
For any possible subspace $L$ with dimension $r$, 
in expectation $\norm{\proj_L w^*}_2^2 = r/d$, 
which implies that unless $\Omega(d)$ queries are made, 
an algorithm must incur $\Omega(1)$ error.
On the other hand, consider a hypothesis $\sigma(w^*\cdot x-t^*)$ 
with bias $p$. We want to tell whether the hypothesis is 0 or not. 
Notice that in the realizable setting, if we make a query $x$ 
and find $y(x)=0$, then by making another query in the opposite direction 
(but very far) we can easily verify whether $h\equiv 0$. 
However, with only a tiny fraction 
of adversarial label noise, we can corrupt all examples far from $0$ 
to have $0$ label. Now if $w^*\sim \s^{d-1}$, 
then only examples in a small cap with volume $p^{1-o(1)}$ 
have non-zero label; thus, unless $1/p^{1-o(1)}$ queries are made, 
we are not able to solve the distinguishing problem. 

Our next lower bound shows that no pool-based active learner 
can achieve the query complexity of our query learner, 
even in the realizable setting.
\begin{theorem}[Label Complexity Lower Bound, Active Learning]\label{th lb compute}
Consider the problem of realizable pool-based active ReLU regression with a restriction that the optimal ReLU has bias $p$. Any active learning algorithm $\A$ that makes less than $\Tilde{O}(d/(p\log(m)))$ label queries over $S$, 
a set of $m$ \iid points drawn from $N(0,I)$, 
will with probability at least $2/3$ output a hypothesis $\hat{h}$ with error $\Tilde{\Omega}(p)$.
\end{theorem}

While a number of prior works~\cite{dasgupta2004analysis, hanneke2015minimax,VMC24,diakonikolas2024active} 
established query complexity lower bounds for 
active classification problems, 
few techniques could be directly applied to the regression setting. 
An immediate obstacle is that the behavior of an adaptive algorithm 
can be much more complicated when the label is changed 
from binary to continuous. An observation, 
inspired by prior works for proving label complexity lower bounds 
for learning halfspaces~\cite{dasgupta2004analysis,diakonikolas2024active}, 
is that given a pool of $m$ unlabeled examples, 
learning a hypothesis with queries is not harder 
than using queries to find $d$ examples with non-zero labels. 
So we will focus on the hardness of this easier problem.

By Yao's minimax principle, we consider a deterministic algorithm 
that solves the problem, while $w^*\sim \s^{d-1}$ for large enough $d$. 
Suppose that a deterministic algorithm wants to use $r$ queries 
to find $k$ examples with non-zero labels 
from a pool of $m$ unlabeled examples drawn from $N(0,I)$. 
Given a set of $m$ unlabeled examples, and any fixed $w^*$, 
the behavior of $\A$ can be uniquely described 
as a path $P=((x^{1},y^1),\dots,(x^{r},y^{r}))$ according to its responses. 
In particular, if the algorithm successfully finds $k$ examples 
with non-zero responses, then there exist $k$ indices $i_1,\dots,i_k$ 
at which the responses are $y^{i_1},\dots,y^{i_k}>0$. 
Since there are $\binom{r}{k}$ such tuples, 
to argue that the algorithm has a large probability of failure, 
it suffices to argue that for a fixed tuple 
the probability of realization is very small. 

Proving such a statement turns out to be challenging, 
due to the rich behavior of the algorithm. In the binary classification 
setting, for each fixed tuple, the corresponding unlabeled examples 
$(x^{i_1},\dots,x^{i_k})$ are unique, since $y\in\{0,1\}$. 
That is, to prove a lower bound for a classification problem, 
we only need to argue that the probability 
these $k$ examples are all positive is small. 
Such a strategy does not work for regression, 
as the next example not only depends on whether the previous example 
is positive, but also depends on the full value of $y$. 
This results in many possible realizations of $((x^{i_1},y^{i_1}),\dots,(x^{i_k},y^{i_k}))$. So we need to bound the integral 
of their density function over all possible outcomes.

Consider a realization of the event 
$((x^{i_1},y^{i_1}),\dots,(x^{i_k},y^{i_k}))$, with $y^{i_j}>0$.  
Such a realization completely characterizes 
a vector $w_L \in L$, where $L$ is the subspace spanned by these $k$ queries. One observation is that if we change the basis of $L$ 
by defining $b_1 = x^{(i_1)}/\norm{x^{(i_1)}}$ 
and $b_j = \proj_{L_{j-1}} x^{(i_j)}/\norm{\proj_{L_{i-1}} x^{(i_j)}}$, 
where $L_{i-1} = \textbf{span}\{x^{(i_1)},\dots,x^{(i_{j-1})}\}$, 
then $w_L$ can be described as a $k$-dimensional vector 
$v(w_L): = (w_L\cdot b_1,\dots,w_L\cdot b_k)$. 
Importantly, the value of the density function of the event 
is exactly the density of the event that the first $k$ coordinates 
of a random unit random vector equal to $v(w_L)$. 
On the other hand, due to the tree-structure of the algorithm, 
given any possible $v(w_L)$, we can decode it to reconstruct 
the corresponding $((x^{i_1},y^{i_1}),\dots,(x^{i_k},y^{i_k}))$. 
Denote by $S_w$ the set of all possible $v(w_L)$. 
This means the probability we are interested in 
is exactly equal to the probability that $\proj_{L_0}(w^*) \in S_w$, 
where $L_0$ is the span of the first $k$ standard basis vectors. 
To derive an upper bound for this probability, 
it is sufficient to find a superset of $S_w$. 
Our observation is that if $x^{i_1},\dots,x^{i_k}$ are orthogonal, 
then $\norm{w_L}^2\ge k(t^*)^2$, which also implies that 
$\norm{v(w_L)}^2\ge k(t^*)^2$. Furthermore, as long as $m$ 
is not as large as $2^{\Omega(d)}$, 
for $k$ chosen to be slightly smaller than $d$, 
every $k$-tuple of examples from the pool are nearly orthogonal. 
This implies that the norm of $v(w_L)$ can be lower bounded uniformly, 
which suffices to bound above the target probability by  $O(p\log(1/p))^k$. 
Since there are at most $\binom{r}{k}$ tuples we care about 
and $k$ is slightly smaller than $d$, a carefully chosen $r$ concludes the proof.

\bibliography{mydb}

\newcommand{\etalchar}[1]{$^{#1}$}
\begin{thebibliography}{DKTZ22b}

\bibitem[ABL17]{awasthi2017power}
Pranjal Awasthi, Maria~Florina Balcan, and Philip~M Long.
\newblock The power of localization for efficiently learning linear separators with noise.
\newblock {\em Journal of the ACM (JACM)}, 63(6):1--27, 2017.

\bibitem[ATV23]{awasthi2022agnostic}
Pranjal Awasthi, Alex Tang, and Aravindan Vijayaraghavan.
\newblock Agnostic learning of general relu activation using gradient descent.
\newblock In {\em The Eleventh International Conference on Learning Representations, {ICLR} 2023}. OpenReview.net, 2023.

\bibitem[BBZ07]{balcan2007margin}
Maria-Florina Balcan, Andrei Broder, and Tong Zhang.
\newblock Margin based active learning.
\newblock In {\em International Conference on Computational Learning Theory}, pages 35--50. Springer, 2007.

\bibitem[BL13]{balcan2013active}
Maria-Florina Balcan and Phil Long.
\newblock Active and passive learning of linear separators under log-concave distributions.
\newblock In {\em Conference on Learning Theory}, pages 288--316. PMLR, 2013.

\bibitem[Das04]{dasgupta2004analysis}
Sanjoy Dasgupta.
\newblock Analysis of a greedy active learning strategy.
\newblock {\em Advances in neural information processing systems}, 17, 2004.

\bibitem[DG03]{dasgupta2003elementary}
Sanjoy Dasgupta and Anupam Gupta.
\newblock An elementary proof of a theorem of johnson and lindenstrauss.
\newblock {\em Random Structures \& Algorithms}, 22(1):60--65, 2003.

\bibitem[DGK{\etalchar{+}}20]{diakonikolas2020approximation}
Ilias Diakonikolas, Surbhi Goel, Sushrut Karmalkar, Adam~R Klivans, and Mahdi Soltanolkotabi.
\newblock Approximation schemes for relu regression.
\newblock In {\em Conference on Learning Theory}, pages 1452--1485. PMLR, 2020.

\bibitem[DKM05]{dasgupta2005analysis}
Sanjoy Dasgupta, Adam~Tauman Kalai, and Claire Monteleoni.
\newblock Analysis of perceptron-based active learning.
\newblock In {\em Learning Theory: 18th Annual Conference on Learning Theory, COLT 2005, Bertinoro, Italy, June 27-30, 2005. Proceedings 18}, pages 249--263. Springer, 2005.

\bibitem[DKM24]{diakonikolas2024active}
Ilias Diakonikolas, Daniel Kane, and Mingchen Ma.
\newblock Active learning of general halfspaces: Label queries vs membership queries.
\newblock {\em Advances in Neural Information Processing Systems}, 37:49180--49213, 2024.

\bibitem[DKMR22]{diakonikolas2022hardness}
Ilias Diakonikolas, Daniel Kane, Pasin Manurangsi, and Lisheng Ren.
\newblock Hardness of learning a single neuron with adversarial label noise.
\newblock In {\em International Conference on Artificial Intelligence and Statistics}, pages 8199--8213. PMLR, 2022.

\bibitem[DKPZ21]{DKPZ21}
I.~Diakonikolas, D.~M. Kane, T.~Pittas, and N.~Zarifis.
\newblock The optimality of polynomial regression for agnostic learning under gaussian marginals in the {SQ} model.
\newblock In {\em Proceedings of The 34\textsuperscript{th} Conference on Learning Theory, {COLT}}, 2021.

\bibitem[DKR23]{DKR23}
Ilias Diakonikolas, Daniel Kane, and Lisheng Ren.
\newblock Near-optimal cryptographic hardness of agnostically learning halfspaces and relu regression under gaussian marginals.
\newblock In {\em International Conference on Machine Learning, {ICML} 2023}, volume 202 of {\em Proceedings of Machine Learning Research}, pages 7922--7938. {PMLR}, 2023.

\bibitem[DKTZ22a]{diakonikolas2022learningr}
Ilias Diakonikolas, Vasilis Kontonis, Christos Tzamos, and Nikos Zarifis.
\newblock Learning a single neuron with adversarial label noise via gradient descent.
\newblock In {\em Conference on learning theory}, pages 4313--4361. PMLR, 2022.

\bibitem[DKTZ22b]{diakonikolas2022learning}
Ilias Diakonikolas, Vasilis Kontonis, Christos Tzamos, and Nikos Zarifis.
\newblock Learning general halfspaces with adversarial label noise via online gradient descent.
\newblock In {\em International Conference on Machine Learning}, pages 5118--5141. PMLR, 2022.

\bibitem[DKZ20]{diakonikolas2020near}
Ilias Diakonikolas, Daniel Kane, and Nikos Zarifis.
\newblock Near-optimal sq lower bounds for agnostically learning halfspaces and relus under gaussian marginals.
\newblock {\em Advances in Neural Information Processing Systems}, 33:13586--13596, 2020.

\bibitem[DMRT24]{diakonikolas2024fast}
Ilias Diakonikolas, Mingchen Ma, Lisheng Ren, and Christos Tzamos.
\newblock Fast co-training under weak dependence via stream-based active learning.
\newblock In {\em Forty-first International Conference on Machine Learning}, 2024.

\bibitem[FCG20]{frei2020agnostic}
Spencer Frei, Yuan Cao, and Quanquan Gu.
\newblock Agnostic learning of a single neuron with gradient descent.
\newblock {\em Advances in neural information processing systems}, 33:5417--5428, 2020.

\bibitem[GTX{\etalchar{+}}24]{gajjar2024agnostic}
Aarshvi Gajjar, Wai~Ming Tai, Xu~Xingyu, Chinmay Hegde, Christopher Musco, and Yi~Li.
\newblock Agnostic active learning of single index models with linear sample complexity.
\newblock In {\em The Thirty Seventh Annual Conference on Learning Theory}, pages 1715--1754. PMLR, 2024.

\bibitem[GV25]{guo2024agnostic}
Anxin Guo and Aravindan Vijayaraghavan.
\newblock Agnostic learning of arbitrary relu activation under gaussian marginals.
\newblock In {\em The Thirty Eighth Annual Conference on Learning Theory}, volume 291 of {\em Proceedings of Machine Learning Research}, pages 2592--2631. {PMLR}, 2025.

\bibitem[HY15]{hanneke2015minimax}
Steve Hanneke and Liu Yang.
\newblock Minimax analysis of active learning.
\newblock {\em J. Mach. Learn. Res.}, 16(1):3487--3602, 2015.

\bibitem[KKSK11]{kakade2011efficient}
Sham Kakade, Varun Kanade, Ohad Shamir, and Adam Kalai.
\newblock Efficient learning of generalized linear and single index models with isotonic regression.
\newblock {\em Advances in Neural Information Processing Systems}, 24, 2011.

\bibitem[KMT24a]{kontonis2024active}
Vasilis Kontonis, Mingchen Ma, and Christos Tzamos.
\newblock Active classification with few queries under misspecification.
\newblock In {\em The Thirty-eighth Annual Conference on Neural Information Processing Systems}, 2024.

\bibitem[KMT24b]{VMC24}
Vasilis Kontonis, Mingchen Ma, and Christos Tzamos.
\newblock Active learning with simple questions.
\newblock In {\em The Thirty Seventh Annual Conference on Learning Theory, June 30 - July 3, 2023, Edmonton, Canada}, volume 247 of {\em Proceedings of Machine Learning Research}, pages 3064--3098. {PMLR}, 2024.

\bibitem[KMT24c]{kontonis2024gain}
Vasilis Kontonis, Mingchen Ma, and Christos Tzamos.
\newblock The gain from ordering in online learning.
\newblock {\em Advances in Neural Information Processing Systems}, 36, 2024.

\bibitem[LT25]{li2025nearoptimal}
Yi~Li and Wai~Ming Tai.
\newblock Near-optimal active regression of single-index models.
\newblock In {\em The Thirteenth International Conference on Learning Representations}, 2025.

\bibitem[O'D14]{o2014analysis}
Ryan O'Donnell.
\newblock {\em Analysis of boolean functions}.
\newblock Cambridge University Press, 2014.

\bibitem[She21]{shen2021power}
Jie Shen.
\newblock On the power of localized perceptron for label-optimal learning of halfspaces with adversarial noise.
\newblock In {\em International Conference on Machine Learning}, pages 9503--9514. PMLR, 2021.

\bibitem[VYS21]{vardi2021learning}
Gal Vardi, Gilad Yehudai, and Ohad Shamir.
\newblock Learning a single neuron with bias using gradient descent.
\newblock {\em Advances in Neural Information Processing Systems}, 34:28690--28700, 2021.

\bibitem[WZDD23]{wang2023robustly}
Puqian Wang, Nikos Zarifis, Ilias Diakonikolas, and Jelena Diakonikolas.
\newblock Robustly learning a single neuron via sharpness.
\newblock In {\em International conference on machine learning}, pages 36541--36577. PMLR, 2023.

\bibitem[YZ17]{yan2017revisiting}
Songbai Yan and Chicheng Zhang.
\newblock Revisiting perceptron: Efficient and label-optimal learning of halfspaces.
\newblock {\em Advances in Neural Information Processing Systems}, 30, 2017.

\bibitem[ZWDD25]{zarifis2025robustly}
Nikos Zarifis, Puqian Wang, Ilias Diakonikolas, and Jelena Diakonikolas.
\newblock Robustly learning monotone generalized linear models via data augmentation.
\newblock In {\em The Thirty Eighth Annual Conference on Learning Theory}, volume 291 of {\em Proceedings of Machine Learning Research}, pages 5921--5990. {PMLR}, 2025.

\end{thebibliography}
\bibliographystyle{alpha}


\newpage

\appendix

\section*{Appendix}

\paragraph{Structure of Appendix} 
The Appendix is organized as follows. 
In \Cref{app background}, we record the 
notation and mathematical background required for our technical 
sections. In \Cref{app sphere}, we present omitted details 
from \Cref{sec sphere}. 
In \Cref{app regression}, we provide a full 
version of \Cref{sec regression}, including 
a detailed discussion on the update rule for $w$ 
and $(r,t)$, missing technical lemmas 
and proofs, and  
the proof of \Cref{th query learning} and the corresponding 
learning algorithm. 
Finally, in \Cref{app lb}, we give the detailed 
proofs of our lower bound results from \Cref{sec lower}.

\section{Preliminaries and Related Background}\label{app background}

\subsection{Problem Definitions and Notation}
We first define the task of Agnostic Learning with queries 
and the corresponding task in the active learning setting.

\begin{definition}[Agnostic ReLU Regression with Queries]
 Let $\sigma(z)= \max\{z,0\}$ be the ReLU function.
 A labeling function $y(x): \R^d \to \R$ is a random function that maps each $x \in \R^d$ to an unknown real-valued random variable. 
    For each $h: \R^d \to \R$, denote by $\err(h) = \E_{x \sim N(0,I)} \left(h(x)-y(x)\right)^2/2$, 
    $\opt:= \min_{W: \norm{W} \le R, t\ge 0} \err(\sigma(W\cdot x-t))$ and $\sigma^*(x) = \sigma(W^*\cdot x-t^*)$ be any ReLU with error $\opt$. 
    A query takes $x \in \R^d$ as an input and returns a label $y \sim y(x)$. We say that a learning algorithm $\A$ is a constant-factor approximate learner if for every labeling function $y(x)$, and for every $\epsilon,\delta \in (0,1)$, it outputs some hypothesis $\hat{h}: \R^d \to \R$ by adaptively making queries, such that with probability at least $1-\delta$, $\err(\hat{h}) \le  O(\opt) + \epsilon$. The query complexity of $\A$ is the total number of queries it uses during the learning process.
\end{definition}

\begin{definition}[Pool-Based Active Learning for Agnostic ReLU Regression] 
    Let $\sigma: \R \to \R$ be a known activation function and $w^* \in \s^{d-1}$ be a unit vector. Let $D$ be a distribution over $\R^d \times \{\pm 1\}$ such that $D_x$, the marginal distribution over $x$, is the standard Gaussian distribution $N(0,I)$. 
    For each $h : \R^d \to \R$, denote by $\err(h) = \E_{(x,y) \sim D} \left(h(x)-y\right)^2/2$, $\opt:= \min_{w: \norm{w} \le R} \err(\sigma(w\cdot x))$ and $\sigma^*(x) = \sigma(W^*\cdot x-t^*)$ be any ReLU with error $\opt$.
    Let $S$ be a set of $m$ \iid labeled examples drawn from $D$. An active learning algorithm (with label query access) is given $S$ but with hidden labels and is allowed to make a label query for each $x \in S$ and observe its label $y(x)$.  We say that a learning algorithm $\A$ is a constant-factor approximate learner if for every distribution $D$ and for every $\epsilon,\delta \in (0,1)$, it outputs some $\hat{h} \in H$ by adaptively making label queries over a set of $m$ examples drawn \iid from $D$, such that with probability at least $1-\delta$, $\err(\hat{h}) \le  O(\opt) + \eps$. The label complexity of $\A$ is the total number label queries made over $S$ during the learning process. 
\end{definition}

We will without loss of generality assume $\opt \le \eps$, since the final error guarantee is $O(\OPT+\eps)$.
We remark that in some parts of the paper, we will also consider the case where $\sigma$ is not a ReLU but a general function and $W^* \in \s^{d-1}$. 
In that case, we call the problem agnostic learning (spherical) generalized linear model (GLM) and for a hypothesis $\sigma(w\cdot x)$, we use $\err(w)$ to denote its error accordingly.
For a vector $W$ in $\R^d$, we use $\norm{W}$ to denote it $\ell_2$ norm and use the lower case $w$ to denote its direction $W/\norm{W}$. For a one-dimensional standard random variable $z\sim N(0,1)$ and $t\ge 0$, we denote by $\Phi(t):= \Pr_{z}(z>t)$ and $\psi(t)$ the value of the density function of $z$ at $t$. For a ReLU activation $\sigma(z-t)$, we denote by $V(t):=\E_{z\sim N(0,1)}\sigma^2(z-t)$ its second moment and call $p$ its bias.
For a real valued function $f:\R \to \R$, we denote by $\norm{f}_2^2:=\E_{z \sim N(0,1)} f^2(z)$ its squared $L_2$ norm in Gaussian space and for $a \in [0,1]$, denote by $T_a f(z):= \E_{s\sim N(0,1)} f(az+\sqrt{1-a^2}s)$.

\subsection{Background on Ornstein–Uhlenbeck Semigroup}
Here we provide basic background on the 
Ornstein–Uhlenbeck Semigroup. We refer the readers to \cite{o2014analysis} for additional background.

\begin{definition}[Ornstein–Uhlenbeck Semigroup]
    Let $\rho \in (0,1)$. The Ornstein–Uhlenbeck semigroup $T_\rho : L_2 \to L_2$ is a linear mapping that maps a function $f$ to a function $T_\rho f$ defined as 
    \begin{align*}
        T_\rho f(z) = \E_{s \sim N(0,1)} f(\rho z + \sqrt{1-\rho^2} s).
    \end{align*}
\end{definition}

\begin{definition}[Ornstein–Uhlenbeck Operator]
    The Ornstein–Uhlenbeck operator $L : L_2 \to L_2$ is a linear mapping that maps a function $f$ to a function $L f$ defined as 
    \begin{align*}
        L f(z) = \frac{d T_\rho f}{d \rho} \mid_{\rho = 1}(z)
    \end{align*}
\end{definition}

We list several facts about the Ornstein–Uhlenbeck semigroup and Ornstein–Uhlenbeck operator.

\begin{fact}[\cite{o2014analysis}]\label{fact ou}
    Let $f,g \in L_2(N)$. The follows statements hold.
    \begin{enumerate}[leftmargin=*]
    \item For $\rho \in [0,1]$, $\norm{T_\rho f}_2^2$ is a non-decreasing function with respect to $\rho$.
    \item For $a,b \in [0,1]$, $\E_{z \in N(0,1)} T_a f(z) T_b f(z) = \norm{T_{\sqrt{ab}} f(z)}_2^2$
    \item \label{item ou1} If $f$ is a differentiable function, then for $\rho \in (0,1)$ $(T_\rho f(z))' = \rho T_\rho f'(z)$. 
    \item \label{item ou2} $\frac{d T_\rho f}{d \rho} = \frac{1}{\rho} LT_\rho f$
    \item \label{item ou3} $\E_{z \sim N(0,1)} \left(f(z)LT_\rho g(z) \right) = \E_{z \sim N(0,1)} \left(f'(z)(T_\rho g(z))' \right)$
    \end{enumerate}
\end{fact}

\subsection{Background on Gaussian Integral}
In this section, we provide background on Gaussian Integral. Let $z\sim N(0,1)$ be the standard normal random variable. For $t\ge 0$, we denote by $\Phi(t):= \Pr_{z}(z>t)$ and $\psi(t)$ the value of the density function of $z$ at $t$. For a ReLU activation $\sigma(z-t)$, we denote by $V(t):=\E_{z\sim N(0,1)}\sigma^2(z-t)$ its second moment. We provide detailed characterization for $\Phi(t),\psi(t),V(t)$. First, we provide the following fact that characterizes $\Phi(t)$.

\begin{fact}[Komatsu's Inequality]\label{fact bias}
    For any $t \ge 0$, $\Phi(t)$ can be bounded as  
    \begin{align*}
        \sqrt{\frac{2}{\pi}} \frac{\exp(-t^2/2)}{t+\sqrt{t^2+4}} \le \Phi(t) \le \sqrt{\frac{2}{\pi}} \frac{\exp(-t^2/2)}{t+\sqrt{t^2+2}}. 
    \end{align*}
\end{fact}

Next, we provide the following fact calculated by \cite{guo2024agnostic} that relates $V(t), \Phi(t), \psi(t)$.
\begin{fact}[see, e.g.,~Appendix A, \cite{guo2024agnostic}]
    For any $t \ge 0$, the following fact holds.
    \begin{align*}
    &\E_{z \sim N(0,1)}z \Ind(z>t) = \psi(t) \\
    &\E_{z \sim N(0,1)}z(z-t) \Ind(z>t) = \Phi(t) \\
        &\E_{z \sim N(0,1)}z^2 \Ind(z>t) = \Phi(t) + t\psi(t) = (1+o_t(1))t^2\Phi(t) \\
        &V(t) = (t^2+1)\Phi(t) - t\psi(t) = (2+o_t(1))\Phi(t)/t^2 \\
    \end{align*}
\end{fact}
For large $t>0$, it is useful to mention the following asymptotic relation between $\Phi(t)$ and $\psi(t)$,
\begin{align*}
    \frac{\Phi(t)}{\psi(t)} \sim \frac{1}{t} - \frac{1}{t^3} + \frac{3}{t^5} - \dots.
\end{align*}

The following useful Stein's lemma will also be frequently used in 
our proofs.

\begin{fact}[Stein's Lemma]\label{fact stein}
    Let $x \sim N(\mu,\sigma^2I)$ be a Gaussian vector in $\R^d$. Let $g: \R^d \to \R $ such that $\E_{x\sim N(\mu,\sigma^2I)} g(x)x$ and $\E_{x\sim N(\mu,\sigma^2I)} \nabla g(x)$ exist. Then the following fact holds
    \begin{align*}
        \E_{x\sim N(\mu,\sigma^2I)} g(x)(x-\mu) = \sigma^2 \E_{x\sim N(\mu,\sigma^2I)} \nabla g(x).
    \end{align*}
\end{fact}

\section{Omitted Details from \Cref{sec sphere}}\label{app sphere}

Here we provide the full details omitted from 
\Cref{sec sphere}.

\subsection{Proof of \Cref{lm noiseless}}\label{app noiseless}
We provide the proof of \Cref{lm noiseless} and its restatement as \Cref{lm noiseless re}.

\begin{lemma}[Restatement of \Cref{lm noiseless}]\label{lm noiseless re}
 Let $\sigma: \R \to \R$ be any activation function such that $\sigma' \in L_2(N)$ and let $w \in \s^{d-1}$ be any unit vector such that $\theta:=\theta(w,w^*) < \pi/2$, then 
 \begin{align*}
     \ell(w) = \int_0^\theta \sin s \norm{T_{\sqrt{\cos s}} \sigma'}_2^2 ds \le \frac{\pi}{2} \sin^2 \theta \norm{\sigma'}_2^2.
 \end{align*}
\end{lemma}

\begin{proof}[Proof of \Cref{lm noiseless}]
We write $w^* = a w + b u$, where $a,b>0, a^2 + b^2 = 1$ and $u \in \s^{d-1}, u \perp w$. Notice that if $x \sim N(0,I)$, then $z=w \cdot x$ and $s = u \cdot x$ are independent standard one dimensional normal random variables.
\begin{align}
 \ell(w) & = \frac{1}{2}\E_{x \sim N(0,I)} \left(\sigma(w\cdot x) - \sigma(w^* \cdot x)\right)^2 = \E_{z \sim N(0,1)} \sigma^2(z) -  \E_{x \sim N(0,I)}\sigma(w\cdot x)\sigma(w^* \cdot x) \notag \\
 & = \E_{z \sim N(0,1)} \sigma^2(z) -  \E_{x \sim N(0,I)}\sigma(w\cdot x)\sigma(a w \cdot x+ b u \cdot x) \notag \\
 & = \E_{z \sim N(0,1)} \sigma^2(z) -  \E_{z,s \sim N(0,1)}\sigma(z)\sigma(a z+ b s) 
  = \E_{z \sim N(0,1)} \sigma(z) \left( \sigma(z) -  T_a\sigma(z) \right) \notag \\
  & = \E_{z \sim N(0,1)} \sigma(z) \int_a^1 \frac{d T_s \sigma(z)}{ds} ds = \E_{z \sim N(0,1)} \sigma(z) \int_a^1 \frac{1}{s} L T_s \sigma(z) ds 
   = \int_a^1 \E_{z \sim N(0,1)} \sigma(z)  \frac{1}{s} L T_s \sigma(z) ds \label{eq noiseless 1}\\
  &= \int_a^1 \E_{z \sim N(0,1)} \sigma'(z)  \frac{1}{s}  (T_s \sigma(z))' ds 
   = \int_a^1 \E_{z \sim N(0,1)} \sigma'(z) T_s \sigma'(z) ds \label{eq noiseless 2} \\ 
  &= \int_a^1 \norm{T_{\sqrt{s}} \sigma'}_2^2 ds = \int_0^\theta \sin s \norm{T_{\sqrt{\cos s}} \sigma'}_2^2 ds. \notag
\end{align}
Here, \eqref{eq noiseless 1} follows \Cref{item ou2} and \eqref{eq noiseless 2} holds because of \Cref{item ou1} and \Cref{item ou3}. The last equation follows a change of variable. By the monotone property of $\norm{T_\rho \sigma'}$, and $\theta \le \pi/2 \sin \theta$ for $\theta \in (0, \pi/2)$, we obtain that 
\begin{align*}
     \ell(w) = \int_0^\theta \sin s \norm{T_{\sqrt{\cos s}} \sigma'}_2^2 ds \le \frac{\pi}{2} \sin^2 \theta \norm{\sigma'}_2^2.
 \end{align*}
\end{proof}

\subsection{Proof of \Cref{lm signal spherical}}\label{app signal spherical}

Here we provide the proof of \Cref{lm signal spherical} and its restatement \Cref{lm signal spherical re}.

\begin{lemma}\label{lm signal spherical re}
    Let $\sigma: \R \to \R$ be any activation function such that $\sigma' \in L_2(N(0,I))$ and let $w \in \s^{d-1}$ be any unit vector such that $\theta=\theta(w,w^*)<\pi/2$, where $\err(\sigma(w^*)) = \opt$. Write $w^* = aw+bu$, where $u \in \s^{d-1}, u \perp w, a,b \ge 0, a^2+b^2 = 1$,
    then, $\proj_{w^\perp} \nabla_w \ell(w) = -b \norm{T_{\sqrt{a}}\sigma'}_2^2 u$.
\end{lemma}

\begin{proof}[Proof of \Cref{lm signal spherical}]
By \Cref{fact ou} and \Cref{fact stein}, we have the following calculation
    \begin{align*}
        \proj_{w^\perp} \nabla_w \ell(w) & = \proj_{w^\perp} \E_{x \sim N(0,I)} \left(\sigma(w\cdot x)-\sigma(w^*\cdot x)\right) \sigma' (w\cdot x) x \\
        & =\proj_{w^\perp} \E_{x \sim N(0,I)} \sigma(w\cdot x) \sigma' (w\cdot x) x - \proj_{w^\perp} \E_{x \sim N(0,I)} \sigma(w^*\cdot x) \sigma' (w\cdot x) x \\
        & =- \proj_{w^\perp} \E_{x \sim N(0,I)} \sigma(w^*\cdot x) \sigma' (w\cdot x) x 
        = - \proj_{w^\perp} \E_{x \sim N(0,I)} \nabla_x \left( \sigma(w^*\cdot x) \sigma' (w\cdot x) \right)  \\
        &= - \proj_{w^\perp} \E_{x \sim N(0,I)}  \left( \sigma'(w^*\cdot x) \sigma' (w\cdot x) w^* + \sigma(w^* \cdot x)\sigma''(w \cdot x)w \right) \\
        &= - \E_{x \sim N(0,I)}   \sigma'(w^*\cdot x) \sigma' (w\cdot x) bu = -  \E_{z,s \sim N(0,1)}   \sigma'(az+bs) \sigma' (z) bu \\
        & = -  \E_{z \sim N(0,1)}  T_a \sigma'(z) \sigma' (z) bu = -b \norm{T_{\sqrt{a}}\sigma'}_2^2 u \;.
    \end{align*}
\end{proof}

\subsection{Proof of \Cref{lm noise control}}\label{app noise control}

We next give the proof for \Cref{lm noise control} and its restatement as \Cref{lm noise control re}.

\begin{lemma}\label{lm noise control re}
    Let $\sigma: \R \to \R$ be any activation function such that $\sigma' \in L_2(N(0,I))$ and let $w \in \s^{d-1}$ be any unit vector such that $\theta=\theta(w,w^*)<\pi/2$, where $\err(\sigma(w^*)) = \opt \le \eps$. Then for any $v \in \s^{d-1}$ and $v \perp w$,
    $ \abs{\proj_{w^\perp} \left( \nabla_w\err(w) - \nabla_w \ell(w) \right) \cdot v} \le \sqrt{\eps}\norm{\sigma'}_2$.
\end{lemma}

\begin{proof}[Proof of \Cref{lm noise control}]
    We notice that 
    \begin{align*}
        \proj_{w^\perp} \left( \nabla_w\err(w) - \nabla_w \ell(w) \right) = \proj_{w^\perp} \E_{x \sim N(0,I)} \left(\sigma(w^*\cdot x)-y\right) \sigma' (w\cdot x) x.
    \end{align*}
    For every $v \in \s^{d-1}$ and $v \perp w$, we have 
    \begin{align*}
        \abs{\proj_{w^\perp} \left( \nabla_w\err(w) - \nabla_w \ell(w) \right) \cdot v} & = \abs{  \E_{x \sim N(0,I)} \left(\sigma(w^*\cdot x)-y\right) \sigma' (w\cdot x) (x \cdot v)} \\
        & \le \sqrt{\E_{x \sim N(0,I)} \left(\sigma(w^*\cdot x)-y\right)^2} \sqrt{\E_{x\sim N(0,I)}\left(\sigma' (w\cdot x) (x \cdot v)\right)^2} \\
        & = \sqrt{\opt}\sqrt{\E_{z,s\sim N(0,1)} \left(\sigma' (z) s\right)^2} \le \sqrt{\eps}\norm{\sigma'}_2.
    \end{align*}
Here, in the first inequality, we use Holder's inequality and in the last inequality, we use the fact that $z,s$ are independent.
\end{proof}

\subsection{Proof of \Cref{lm gradient descent}} \label{app gradient descent}

We next give the proof of the angle contraction lemma as follows.

\begin{lemma}[Angle Contraction]\label{lm gradient descent re}
    Let $w^*,w^{(i)} \in \s^{d-1}$ such that $w^* = a w^{(i)}+b u$, where $u \in \s^{d-1}, u \perp w^{(i)}, a,b \ge 0, a^2+b^2 = 1$. Let $\theta_i=\theta(w^{(i)},w^*)$. 
    Let $G \in \R^d$ be a random vector such that with probability $1$, $G \perp w^{(i)}$. Let $g$ be the mean of $G$ and $\hat{g} \in \R^d$.
    Suppose there is some $c>0$ such that $g \cdot u \ge cb/10, \norm{g} \le cb, \norm{g-\hat{g}} \le bc/40$, then by setting $\mu =c/20 $, the update rule $w^{(i+1)} = \proj_{\s^{d-1}} (w^{(i)} + \mu\hat{g})$ satisfies
    $\sin (\theta_{i+1}/2)  \le \sqrt{1-\left(\frac{c}{20}\right)^2}\sin (\theta_i/2)$.
\end{lemma}

\begin{proof}[Proof of \Cref{lm gradient descent}]
    \begin{align*}
    \norm{w^{(i+1)}-w^*}^2 = \norm{\proj_{\s^{d-1}}(w^{(i)} + \mu\hat{g}) - \proj_{\s^{d-1}} (w^*)}^2 \le \norm{w^{(i)} + \mu\hat{g}-w^*}^2. 
\end{align*}
It remains to upper bound $\norm{w^{(i)} + \mu\hat{g}-w^*}^2$. We have 
\begin{align*}
    \norm{w^{(i)} + \mu\hat{g}-w^*}^2 & = \norm{w^{(i)}-w^*}^2 + 2\mu\hat{g}\cdot(w^{(i)}-w^*)+ \mu^2\norm{\hat{g}}^2 \\ 
    & =  \norm{w^{(i)}-w^*}^2 - 2\mu\hat{g}\cdot w^*+ \mu^2\norm{\hat{g}}^2 \\
    & =  \norm{w^{(i)}-w^*}^2 - 2\mu g \cdot w^*+ 2\mu(g-\hat{g}) \cdot w^* +\mu^2\norm{\hat{g}}^2 \\
    & = \norm{w^{(i)}-w^*}^2 - 2\mu g \cdot w^*+ 2\mu(g-\hat{g}) \cdot b u +\mu^2\norm{\hat{g}}^2 \\
    & \le \norm{w^{(i)}-w^*}^2 - 2\mu g \cdot w^*+ 2\mu b\norm {g-\hat{g}} +\mu^2\norm{\hat{g}}^2. \\
    & = \norm{w^{(i)}-w^*}^2 - 2\mu b g \cdot u+ 2\mu b\norm {g-\hat{g}} +\mu^2\norm{\hat{g}}^2.
\end{align*}
Here, in the second equality, we use the fact that $\hat{g} \perp w^{(i)}$ and in the fourth equality, we use the fact that $(g-\hat{g})\cdot w^* = (g-\hat{g})\cdot a w^{(i)} + (g-\hat{g})\cdot b u = (g-\hat{g})\cdot bu$.

Notice that $\norm{w^{(i)}-w^*} =2 \sin\frac{\theta_i}{2}$ and $b = \sin \theta_i$. By choosing $\mu = 1/C \le c/20$ We have 
    \begin{align*}
        (2 \sin\frac{\theta_{i+1}}{2})^2 & \le (2 \sin\frac{\theta_{i}}{2})^2 -\mu c b^2/5 + \mu c b^2/20 + \mu^2c^2 b^2 \\
        &  \le (2 \sin\frac{\theta_{i}}{2})^2 -\mu c \sin^2\frac{\theta_i}{2}/5 = 4(1-\frac{\mu c}{20}) \sin^2 \frac{\theta_i}{2} = 4(1-\left(\frac{c}{20}\right)^2) \sin^2 \frac{\theta_i}{2} \;. \\
    \end{align*}
    This implies that 
$\sin (\theta_{i+1}/2)  \le \sqrt{1-\left(\frac{c}{20}\right)^2}\sin (\theta_i/2)$.
\end{proof}

\subsection{Proof of \Cref{lm stopping condition}}

For convenience, we restate the lemma below. 

\begin{lemma}\label{lm stopping condition re}
     Let $\sigma: \R \to \R$ be any activation function such that $\sigma' \in L_2(N(0,I))$. Let $\alpha>1$ and $0<\theta_0<\pi/2$ such that $\norm{T_{\sqrt{\cos \theta_0}}\sigma'}_2^2 \ge \norm{\sigma'}_2^2 /\alpha$. Let 
     $w \in \s^{d-1}$ be any unit vector such that $\theta=\theta(w,w^*)<\theta_0$, where $\err(\sigma(w^*)) = \opt \le \eps$. 
    If $\sin^2 \theta \norm{\sigma'}_2^2 \ge 20 \alpha^2 \eps/ \pi$, then for any $v \in \s^{d-1}$ and $v \perp w$,
    $ \norm{\proj_{w^\perp} \left( \nabla_w\err(w) -\nabla_w \ell(w) \right)} \le \norm{\proj_{w^\perp} \nabla_w \ell(w)}/20 $. Furthermore, if $ \norm{\proj_{w^\perp} \left( \nabla_w\err(w) -\nabla_w \ell(w) \right)} > \norm{\proj_{w^\perp} \nabla_w \ell(w)}/20 $, then $\err(w) \le O(\alpha^2\eps).$
\end{lemma}

\begin{proof}[Proof of \Cref{lm stopping condition}]
     Since $\sin^2 \theta \norm{\sigma'}_2^2 \ge 20 \alpha^2 \eps$, we know that $\sqrt{\eps} \le  \sin\theta \norm{\sigma'}_2 /\sqrt{20\alpha^2} .$ By \Cref{lm noise control}, we know that 
     \begin{align*}
     \norm{\proj_{w^\perp} \left( \nabla_w\err(w) - \nabla_w \ell(w) \right) \cdot v} & \le \sqrt{\eps}\norm{\sigma'}_2 \le  \sin \theta \norm{\sigma'}^2/\sqrt{20\alpha^2}\\
     &\le  \sin \theta \norm{T_{\sqrt{\cos \theta}}\sigma'}_2^2/\sqrt{20} =   \norm{\proj_{w^\perp} \nabla_w \ell(w)}/\sqrt{20}. 
     \end{align*}
     Here, the last inequality holds by the monotone property of $\norm{T_\rho \sigma'}$.
On the other hand, if      $ \norm{\proj_{w^\perp} \left( \nabla_w\err(w) -\nabla_w \ell(w) \right)} > \norm{\proj_{w^\perp} \nabla_w \ell(w)}/\sqrt{20} $, by \Cref{lm noiseless}, we know that 
    \begin{align*}
    \err(w) &\le \E_{x} \left(\sigma(w\cdot x)-\sigma(w^*\cdot x) \right)^2 + \E_{x} \left(\sigma(w^*\cdot x)-y \right)^2 \le 2\opt + \E_{x} \left(\sigma(w\cdot x)-\sigma(w^*\cdot x) \right)^2 \\
    & \le 2\opt + \pi \sin^2 \theta \norm{\sigma'}_2^2 \le O(\alpha^2 \eps).
\end{align*}
\end{proof}

Combining \Cref{lm noiseless}, \Cref{lm signal spherical}, \Cref{lm noise control}, \Cref{lm stopping condition}, we know that the update rule \eqref{eq update} satisfies the following property. Suppose we are given a warm start $w^{(0)}$ such that $\norm{T_{\sqrt{\cos \theta_0}}\sigma'}_2^2 \ge \norm{\sigma'}_2^2 /\alpha$. 
As long as the current $w$ has large angle, 
and thus has large noiseless error, 
the noise rate is much smaller than 
the length of the gradient used in the update 
and the angle can be improved. On the other hand, 
once the length of the gradient is small 
and the noise level is large, it must be the case 
that the angle is small enough so that the error 
of the current hypothesis is as small as $O(\alpha^2\eps)$. 
We next use this property to show that when $\norm{W^*}=1$ and $t^*>0$ is given, we are able to solve the ReLU regression 
problem with label complexity $\tilde{O}(1/p+d\polylog(1/\eps))$.

\subsection{Proof of \Cref{lm reduction}}\label{app reduction}

We start by proving that using a method of label truncation, 
we are able to reduce the initialization for ReLU regression 
to the initialization for agnostic learning of halfspaces. 
We present the proof of \Cref{lm reduction} and its restatement \Cref{lm reduction re}.

\begin{lemma}\label{lm reduction re}
Let $\sigma$ be an activation f the form $\sigma(z)=\relu(z -t^*)$, where $t^*>0$. Let $y(x)$ be any labeling function such that $\opt \le \eps$.
Let $c>0$ be a suitably small constant.
If $V(t^*)>C \eps$, for some large constant $C$, then $\Pr_{x \sim N(0,I)} \left(\bar{y}(x) \neq \sign(w^*\cdot x - t^*) \right) \le \Phi(t^*)/C'$ for some large constant $C'>0$, where  $\Bar{y}(x): = \Ind\{y(x)>c/(t^*)\}$.   
\end{lemma}
The idea of the proof is as follows. Since the labels are 
continuous, for those examples with ground truth labels very close 
to the threshold $c/(t^*)$, we are not able to control their 
behavior. But on the other hand, for those examples with ground 
truth labels that are far from the threshold, to change their 
pseudo-label $\bar{y}(x)$, the adversary must add high-level noise 
to them. We will show that as long as $h \equiv 0$ does not have 
error $O(\eps)$, we can reduce the problem to 
halfspace learning by carefully choosing $c$.

\begin{proof}[Proof of \Cref{lm reduction}]
    We partition $\R^d$ into three regions $I_1:=\{x \mid w^* \cdot x> t^* +2c/(t^*+1)\}, I_2:=\{x \mid w^* \cdot x< t^*\}$ and $J:=\{x \mid t^*< w^* \cdot x< t^*+2c/(t^*+1)\}$. Since $\opt=\E_{x \sim N(0,I)} \left( \sigma(w^*\cdot x -t^*) - y \right)^2/2 < \eps$, we have 
    \begin{align*}
        \eps & \ge  \E_{x\sim N(0,I)} \left( \sigma(w^*\cdot x -t^*) - y \right)^2/2 \\
        & \ge \E_{x\sim N(0,I)} \Ind\{x \in I_1 \cup I_2\} \left( \sigma(w^*\cdot x -t^*) - y\right)^2/2 \\
        & \ge \E_{x\sim N(0,I)} \left(\Ind\{x \in I_1, y(x)< c/t^* \} + \Ind\{x \in I_2, y(x)>c/t^*\}\right) \left( \sigma(w^*\cdot x -t^*)-y \right)^2/2 \\
        & \ge \frac{c^2}{2(t^*+1)^2} \Pr_{x \sim N(0,I)}\left( x \in I_1 \cup I_2, y(x) \neq \bar{y}(x) \right)
    \end{align*}
Since $\E_{x \sim N(0,I)} \sigma(w^*\cdot x - t^*)^2 \ge C\eps $, we obtain that
\begin{align*}
    \Pr_{x \sim N(0,I)}\left( x \in I_1 \cup I_2, y(x) \neq \bar{y}(x) \right) \le \frac{2(t^*+1)^2}{c^2} \eps \le \frac{2(t^*+1)^2 \E_{x \sim N(0,I)} \sigma(w^*\cdot x - t^*)^2}{Cc^2} \le \Phi(t^*)/C_1,
\end{align*}
for some large enough $C_1>0$.

On the other hand, 
\begin{align*}
    \Pr_{x \sim N(0,I)}\left(x \in J, y(x) \neq \bar{y}(x)\right) \le \Pr_{x \sim N(0,I)}\left(x \in J\right) \le \frac{2c}{(t^*+1)} \frac{1}{\sqrt{2\pi}}\exp\left(-\frac{(t^*)^2}{2}\right) \le 10c \Phi(t^*).
\end{align*}
Thus, 
\begin{align*}
    \Pr_{x \sim N(0,I)} \left(\bar{y}(x) \neq \sign(w^*\cdot x - t^*) \right) & = \Pr_{x \sim N(0,I)}\left(x \in J, y(x) \neq \bar{y}(x)\right) + \Pr_{x \sim N(0,I)}\left( x \in I_1 \cup I_2, y(x) \neq \bar{y}(x) \right)\\
    & \le \Phi(t^*)/C_1 + 10c \Phi(t^*) = \Phi(t^*)/C''
\end{align*}
for some large enough $C''>0$.
\end{proof}

\subsection{Proof of \Cref{lm initialization}}\label{app initialization}

Using the initialization technique recently developed 
in \cite{diakonikolas2024active}, with label complexity 
$\tilde{O}(1/p+d\polylog(1/\eps))$, 
we are able to get a $w^{(0)}$ with $\theta_0 \le O(1/t^*)$. 
We next show that such an angle satisfies \Cref{lm stopping condition} for $\alpha=O(1)$. 
We provide the proof of \Cref{lm initialization} and its restatement as \Cref{lm initialization re}.

\begin{lemma}\label{lm initialization re}
    Let $\sigma = \sigma^*_{t^*}$ be the optimal activation function of the form $\sigma(w^*\cdot x -t^*)$, where $\sigma$ is the ReLU function, $w^* \in \s^{d-1}$ and $t^*>0$ is known. If $\sin (\theta/2) \le 1/t^*$, then $\norm{T_{\sqrt{\cos \theta}} \sigma'}_2^2 \ge \norm{\sigma'}_2^2 / 50$.
\end{lemma}

\begin{proof}[Proof of \Cref{lm initialization}]
Let $w \in \s^{d-1}$ be any direction such that $\sin (\theta/2) \le 1/t^*$. Write
    $w^* = a w+b u$, where $u \in \s^{d-1}, u \perp w, a,b \ge 0, a^2+b^2 = 1$. By \Cref{lm gradient descent}, we know that 
    \begin{align*}
        \E_{x \sim N(0,I)} \sigma(w^* \cdot x - t^*)\sigma'(w^* \cdot x - t^*)(x \cdot u) = b \norm{T_{\sqrt{\cos \theta}} \sigma'}_2^2.
    \end{align*}
To show $\norm{T_{\sqrt{\cos \theta}} \sigma'}_2^2 \ge \norm{\sigma'}_2^2 / 50$, it is sufficient to show that $\E_{x \sim N(0,I)} \sigma(w^* \cdot x - t^*)\sigma'(w^* \cdot x - t^*)(x \cdot u) \ge b \norm{\sigma'}^2_2/50$.
Write $z = w \cdot x, s = u \cdot x$. Notice that $z,s$ are independent one-dimensional normal random variables. We have
\begin{align*}
    \E_{x \sim N(0,I)} \sigma(w^* \cdot x - t^*)\sigma'(w \cdot x - t^*)(x \cdot u) & = \E_{z,s} \sigma(az+bs-t^*)\sigma'(z-t^*) s \\
    & = \E_{z} \sigma'(z-t^*) \E_{s} \sigma(az+bs-t^*)  s \\
    & = b \E_{z} \sigma'(z-t^*) \E_{s} \sigma'(az+bs-t^*).
\end{align*}
Notice that $\sigma'(z-t^*) = \Ind(z>t^*)$. We have 
\begin{align*}
    b \E_{z} \sigma'(z-t^*) \E_{s} \sigma'(az+bs-t^*) & = b \int_{t^*}^\infty \left( \sigma'(z-t^*) \right)^2 \psi(z) \E_s \frac{\sigma'(az+bs-t^*)}{\sigma'(z-t^*)} dz \\
    & = b \int_{t^*}^\infty \left( \sigma'(z-t^*) \right)^2 \psi(z) \E_s \frac{\sigma'(s - \frac{t^*-az}{b})}{\sigma'(z-t^*)} dz \\ 
    & = b \int_{t^*}^\infty \left( \sigma'(z-t^*) \right)^2 \psi(z) \E_s \sigma'(s - \frac{t^*-az}{b}) dz \\
    & \ge b \int_{t^*}^\infty \left( \sigma'(z-t^*) \right)^2 \psi(z) \E_s \sigma'(s - \frac{t^*-at^*}{b}) dz \\
    & = b \int_{t^*}^\infty \left( \sigma'(z-t^*) \right)^2 \psi(z) dz \Pr_{s \sim N(0,1)}(s> \frac{b^2t^*}{b(1+a)}).
\end{align*}
Since $\sin(\theta/2)<1/t^*$, we know that $b<2/t^*$, which implies that $\Pr_{s \sim N(0,1)}(s> \frac{b^2t^*}{b(1+a)}) \ge \Pr_{s}(s>2)\ge 0.02$. Thus, we obtain that 
\begin{align*}
    \E_{x \sim N(0,I)} \sigma(w^* \cdot x - t^*)\sigma'(w^* \cdot x - t^*)(x \cdot u) \ge b\norm{\sigma'}^2_2/50.
\end{align*}
\end{proof}

\subsection{Proof of \Cref{lm initialization variance}}\label{app initialization variance}
Here we show how to use queries to boost the gradient 
used by \eqref{eq update}. The idea behind our proof 
is that if we consider the random vector 
$G=\left(\sigma(w \cdot x-t^*) - y(x)\right)\proj_{w^\perp} x$, 
where $x \sim N(0,I) \mid_{\{x \mid w \cdot x> t^*\}}$, 
then with the warm-start we have
that the expectation of such a random vector 
plays the same role as the gradient $\nabla_w \err(w)$, 
but has larger length and small variance. 
This allows us to estimate it to a desired accuracy 
with few queries. 
We provide the proof of \Cref{lm initialization variance} and its 
restatement \Cref{lm initialization variance re}.

\begin{lemma}\label{lm initialization variance re}
Let $\sigma(z)=\relu(z -t^*)$, with $t^*>0$. 
Let $y(x)$ be any labeling function 
such that $\opt \le \eps$. Let $w \in \s^{d-1}$ be any vector 
such that $\sin (\theta/2) \le 1/t^*$. 
Denote by $G^* \in \R^d$ the random vector 
$\left(\sigma(w \cdot x-t^*) - \sigma(w^* \cdot x-t^*)\right) \proj_{w^\perp} x$ and $G$ 
the random vector $\left(\sigma(w \cdot x-t^*) - y(x)\right)\proj_{w^\perp} x$, 
where $x \sim N(0,I) \mid_{\{x \mid w \cdot x> t^*\}}$. 
Then the following holds:
\begin{enumerate}[leftmargin=*]
    \item $\E G^* = b \norm{T_{\sqrt{a}}\sigma'}_2^2 u/\Phi(t^*)$.
    \item  $\abs{\left( \E G^* - \E G \right) \cdot v} \le \sqrt{\eps}\norm{\sigma'}_2/\Phi(t^*)$.
    \item If $\sin^2\theta \Phi(t^*)>\eps$, then $\E (G \cdot v)^2 \le \tilde{O}(b^2), \forall v \in \s^{d-1}$.
\end{enumerate}
\end{lemma}

\begin{proof}[Proof of \Cref{lm initialization variance}]
    We first prove the first item. 
    \begin{align*}
        \E G^* & = \E_{x \sim N(0,I) \mid_{\{x \mid w \cdot x> t^*\}}}\left(\sigma(w \cdot x-t^*) - \sigma(w^* \cdot x-t^*)\right) \proj_{w^\perp} x \\
        & = \E_{x \sim N(0,I)} \Ind\{w\cdot x > t^*\} \left(\sigma(w \cdot x-t^*) - \sigma(w^* \cdot x-t^*)\right) \proj_{w^\perp} x / \Phi(t^*) \\
        & = \E_{x \sim N(0,I)} \sigma'(w \cdot x-t^*) \left(\sigma(w \cdot x-t^*) - \sigma(w^* \cdot x-t^*)\right) \proj_{w^\perp} x / \Phi(t^*) \\
        & = -b \norm{T_{\sqrt{a}}\sigma'}_2^2 / \Phi(t^*) u
    \end{align*}
Here, the last equation follows from \Cref{lm signal spherical}.

We next prove the second item.
\begin{align*}
      \abs{\E (G^* -G) \cdot v}   & = \abs{ \E_{x \sim N(0,I) \mid_{\{x \mid w \cdot x> t^*\}}}\left(y(x) - \sigma(w^* \cdot x-t^*)\right) (x \cdot v))} \\
        & = \abs{ \E_{x \sim N(0,I)} \Ind\{w\cdot x > t^*\} \left(y(x) - \sigma(w^* \cdot x-t^*)\right) (x \cdot v))}/\Phi(t^*) \\
        & = \abs{ \E_{x \sim N(0,I)} \sigma'(w\cdot x -t^*) \left(y(x) - \sigma(w^* \cdot x-t^*)\right) (x \cdot v))}/\Phi(t^*) \\
        & \le \sqrt{\eps}\norm{\sigma'}_2/\Phi(t^*)
    \end{align*}
 Here the last inequality follows from \Cref{lm noise control}.
 
 Finally, we control the variance of $G$. We have 
 \begin{align*}
     \E  (G \cdot v)^2  \le  2\E  ((G-G^*) \cdot v)^2 +2\E  (G^* \cdot v)^2. 
 \end{align*}
We bound the two terms separately.
\begin{align*}
    \E  (G^* \cdot v)^2 = & \E_{x \sim N(0,I) \mid_{\{x \mid w \cdot x> t^*\}}}\left(\sigma(w \cdot x-t^*) - \sigma(w^* \cdot x-t^*)\right)^2 (x \cdot v)^2 \\
    & = \E_{x \sim N(0,I)} \Ind\{w\cdot x > t^*\}\left(\sigma(w \cdot x-t^*) - \sigma(w^* \cdot x-t^*)\right)^2 (x \cdot v)^2/\Phi(t^*) \\
    & \le \E_{x \sim N(0,I)} \Ind\{w \cdot x > t^*\}\left((w-w^*)\cdot x\right)^2 (x \cdot v)^2/\Phi(t^*)
\end{align*}
 Recall that $w^* = a w + b u$, where $u \in \s^{d-1}, u \perp w, a,b \ge 0, a^2+b^2 = 1$. We have 
 \begin{align*}
     \E  (G^* \cdot v)^2 & \le \E_{x \sim N(0,I)} \Ind\{w \cdot x > t^*\}\left( (1-a)^2(w\cdot x)^2 + b^2(u\cdot x)^2 \right) (x \cdot v)^2/\Phi(t^*) \\
     & = \E_{x \sim N(0,I)} \Ind\{w \cdot x > t^*\} (1-a)^2(w\cdot x)^2  (x \cdot v)^2/\Phi(t^*) + \Ind\{w \cdot x > t^*\}\left( b^2(u\cdot x)^2 \right) (x \cdot v)^2/\Phi(t^*) \\
     & \le O(b^4) \E_{z,s}\Ind\{z > t^*\} z^2  s^2/\Phi(t^*) + b^2 \E_{z,s,r}\Ind\{z > t^*\} s^2  r^2/\Phi(t^*)
      \le O(b^2)
 \end{align*}
 In the second last inequality above, we use the fact that $1-\cos \theta \le O(\sin^2 \theta)$ and $\E_{z} \Ind\{z>t^*\}z^2 \le O((t^*)^2\phi(t))$ and $b \le O( 1/t^*)$.

 We next bound the first term. We have 
 \begin{align*}
     \E \left((G^* -G) \cdot v\right)^2   & = \abs{ \E_{x \sim N(0,I) \mid_{\{x \mid w \cdot x> t^*\}}}\left(y(x) - \sigma(w^* \cdot x-t^*)\right)^2 (x \cdot v)^2)} \\
     & = \E_{x \sim N(0,I)}\left(y(x) - \sigma(w^* \cdot x-t^*)\right)^2 (x \cdot v)^2 \Ind\{w\cdot x>t^*\}/\Phi(t^*) \\
 \end{align*}
 Let $M>0$ be a threshold such that $\E_{z\sim N(0,1)} z^2 \Ind\{\abs{z}>M\} \le \eps$. Notice that if we set $y' = \sign(y) \min\{\abs{y},M\}$, then we will only introduce at most $\eps$ error. So, we can without loss of generality, assume $\abs{y} \le M$. In particular, for ReLU activation, $M \le O(\sqrt{\log(1/\eps)}).$ Based on this, we obtain that 
 \begin{align*}
     \E \left((G^* -G) \cdot v\right)^2 & = \E_{x \sim N(0,I)}\left(y(x) - \sigma(w^* \cdot x-t^*)\right)^2 (x \cdot v)^2 \Ind\{w\cdot x>t^*\}\Ind\{(x\cdot v)\le M\}/\Phi(t^*) \\ 
     &+ \E_{x \sim N(0,I)}\left(y(x) - \sigma(w^* \cdot x-t^*)\right)^2 (x \cdot v)^2 \Ind\{w\cdot x>t^*\}\Ind\{(x\cdot v)> M\}/\Phi(t^*) \\
     & \le \opt M^2/\Phi(t^*) + 4M^2 \E_{x \sim N(0,I)} (x \cdot v)^2 \Ind\{w\cdot x>t^*\}\Ind\{(x\cdot v)> M\}/\Phi(t^*) \\
     & \le \eps M^2/\Phi(t^*) + 4\eps M^2 \le \sin^2\theta M^2 = \tilde{O}(b^2) \;.
 \end{align*}
\end{proof}

\subsection{Proof of \Cref{lm test}}\label{app test}

In this section, we present the proof of \Cref{lm test} and the corresponding hypothesis selection algorithm. For convenience, we first state a formal version of  \Cref{lm test} as follows.
\begin{lemma}[Hypothesis Selection with Queries]\label{lm test re}
Let $D$ be a distribution over $\R^d \times \R$ and let $D_x$ be the marginal distribution of $D_x$.
    There is an algorithm that, on input a list of hypotheses $h_1,\dots,h_k$ such that for $i \in [k]$, $h_i: \R^d \to \R, \E_{x\sim D} h^2_i(x)$ exists, it makes $\poly(k)$ queries and returns a hypothesis $\hat{h}$ such that 
    $\err(\hat{h}) \le O(\min_{i \in [k]} \E_{(x,y)\sim D}(y-h_i(x))^2)$.
\end{lemma}
As a subroutine, we will use later, 
we present the following well-known median-of-mean estimator.
\begin{lemma}[Median of Means Estimation]\label{lm variance-sample}
    Let $G \in \R^d$ be a random vector such 
    that for every $i \in [d]$, $\E (G\cdot x_i)^2 \le B^2$. 
    Then there is an estimator that takes 
    $M=O(\log(1/\delta) dm)$, $\iid$ samples from $G$ 
    and computes a vector $g \in \R^d$ such that with probability at least $1-\delta$, $\norm{g - \E G} \le B/\sqrt{m}$.
\end{lemma}
In this hypothesis selection problem, we only need to use the 
version for $d=1$. Roughly speaking, if a random variable has a 
bounded variance, then very few samples suffice 
to estimate its mean. We present the following algorithm.

\begin{algorithm}
		\caption{\textsc{HypothesisSelection}(Select a good hypothesis from a list of hypothesis)}\label{alg test}
		\begin{algorithmic} [1]
\State\textbf{Input:} $h_1,\dots,h_k$ such that for $i \in [k]$, $h_i: \R^d \to \R,$, $D_x$ a marginal distribution over $\R^d$
\State\textbf{Output:} $\hat{h}: \R^d \to \R$, such that $\err(\hat{h}) \le O(\min_{i \in [k]} \E_{(x,y)\sim D}(y-h_i(x))^2)$ with non-trivial probability.
\State Let $d(x)$ be the density function of $D_x$ at $x$
\State Create an empty graph $G$ with a set of node $[k]$
\For{ each $(i,j), i\neq j \in [k]$}
\State For each pair of $(i,j)$ create a hypothesis $g_{ij}(x) = \frac{h_i(x)-h_j(x)}{\norm{h_i-h_j}}$
\State Denote by $D_{ij}$ the distribution over $x$ with density proportional to $g^2_{ij}(x)d(x)$
\State Use the median of mean method to estimate $\E_{x\sim D_{ij}}(y(x)-\frac{h_i+h_j}{2})g_{ij}(x)/g^2_{ij}(x)$ with $\poly(k)$ samples
\State If the estimated result is more than $\norm{h_i-h_j}$, draw an edge from $i$ to $j$
\EndFor
\State Return any $h_i$ such that $i$ is in a source strongly connected component of $G$
\end{algorithmic}
\end{algorithm}

\begin{proof}[Proof of \Cref{lm test}]
    We first observe that
    \begin{align*}
        \err(h_i) - \err(h_j) = 2\E_{x\sim D_x}y(h_j-h_i) + \norm{h_i}^2-\norm{h_j}^2 = 2\E_{x\sim D_x}(y-\frac{h_i+h_j}{2})(h_j-h_i),
    \end{align*}
    which means to compare the error of $h_i,h_j$, it is sufficient to check the sign of $\E_{x\sim D_x}(y-\frac{h_i+h_j}{2})(h_j-h_i)$.
    And this is equivalent to check the sign of $\E_{x\sim D_x}(y-\frac{h_i+h_j}{2})g_{ij}$. We remark that by our construction, $\norm{g_{ij}}_2=1$.
    Notice that 
    \begin{align*}
        \E_{x\sim D_{ij}}(y-\frac{h_i+h_j}{2})g_{ij}(x)/g^2_{ij}(x) = \E_{x\sim D_x}g_{ij}(x)\frac{(y-\frac{h_i+h_j}{2})}{g^2_{ij}(x)}g^2_{ij}(x) = \E_{x\sim D_x}(y-\frac{h_i+h_j}{2})g_{ij}(x).
    \end{align*}
Furthermore, consider the variance of the above quantity, we have 
\begin{align*}
    \E_{x\sim D_{ij}}(y-\frac{h_i+h_j}{2})^2g^2_{ij}(x)/g^4_{ij}(x) = \E_{x\sim D_x}(y-\frac{h_i+h_j}{2})^2\frac{g^2_{ij}(x)}{g^4_{ij}(x)}g^2_{ij}(x) = \norm{(y-\frac{h_i+h_j}{2})}_2^2.
\end{align*}
By \Cref{lm variance-sample}, we know that with $\poly(k)$ samples we are able to estimate $\E_{x\sim D_x}(y-\frac{h_i+h_j}{2})g_{ij}(x)$, with error $\norm{(y-\frac{h_i+h_j}{2})}/\poly(k)$.  
We consider a fixed pair of $(h_i,h_j)$. Assume $\norm{(y-\frac{h_i+h_j}{2})} \ge \norm{h_i-h_j}$, then we have 
\begin{align*}
    \norm{h_i-y} \le \norm{(y-\frac{h_i+h_j}{2})} + \frac{\norm{(h_i-h_j)}}{2} \le \frac{3}{4}(\norm{h_i-y}+\norm{h_j-y}),
\end{align*}
which implies $\norm{h_i-y} \le 3\norm{h_j-y}$. By symmetry, we have $\norm{h_j-y} \le 3\norm{h_i-y}$. This implies if $\norm{h_i-y}^2 \ge 10\norm{h_j-y}^2$, then we must have $\norm{(y-\frac{h_i+h_j}{2})} \le \norm{h_i-h_j}$. We next prove that if $\norm{h_i-y}^2 \ge 10\norm{h_j-y}^2$, then there must be an edge from $j$ to $i$ but no edge from $i$ to $j$. Recall that 
\begin{align*}
    &\E_{x\sim D_x}(y-\frac{h_i+h_j}{2})(h_j-h_i)  = \E_{x\sim D_x}y(h_j-h_i) - \frac{\norm{h_j}^2-\norm{h_i}^2}{2}\\
    =& \frac{1}{2}(\norm{h_i-y}^2-\norm{h_j-y}^2+\norm{h_j}^2-\norm{h_i}^2)- \frac{\norm{h_j}^2-\norm{h_i}^2}{2}\ge \frac{9}{2}\norm{h_j-y}^2
\end{align*}
Since 
\begin{align*}
    \norm{h_i-h_j} \ge \norm{h_i-y}-\norm{h_j-y} \ge (\sqrt{10}-1) \norm{h_j-y},
\end{align*}
we have $\norm{h_j-y} \le \norm{h_i-h_j}/2$.
This gives 
\begin{align*}
    \E_{x\sim D_x}(y-\frac{h_i+h_j}{2})(h_j-h_i) \ge \frac{9}{8}\norm{h_i-h_j}^2.
\end{align*}
Thus, when $\norm{h_i-y}^2 \ge 10 \norm{h_j-y}^2$, we have $\E_{x\sim D_x}(y-\frac{h_i+h_j}{2})g_{ji} \ge 9\norm{h_i-h_j}/8$ and the estimation of this is larger than $\norm{h_i-h_j}$ with high probability. Thus, there must be an edge from $j$ to $i$ and no edge from $i$ to $j$.

On the other hand, if $\norm{(y-\frac{h_i+h_j}{2})} \ge k^2 \norm{h_i-h_j}$, then we have 
\begin{align*}
    \norm{h_i-y} \le \norm{(y-\frac{h_i+h_j}{2})} + \frac{\norm{(h_i-h_j)}}{2} \le \left(\frac{1}{2}+ \frac{1}{4k^2}\right)
    (\norm{h_i-y}+\norm{h_j-y}),
\end{align*}
which implies $\norm{h_i-y} \le (1+1/2k^2)\norm{h_j-y}$. By symmetry, we have $\norm{h_j-y} \ge (1-1/2k^2)\norm{h_i-y}$. If $\norm{h_i-y}^2 \le \norm{h_j-y}^2$, then we have $\E yg_{ji} \le 0$. Since we have estimated $\E yg_{ji}$ up to error $\norm{y-\frac{h_i+h_{i+1}}{2}}/k^2$, this implies unless $\norm{h_i-y}$ and $\norm{h_j-y}$ are within a factor of $1\pm 1/2k^2$ of each other, there will be no edge from $j$ to $i$. Thus, if there is an edge from $j$ to $i$, then it must be the case that $\norm{h_i-y} \ge \norm{h_j-y}(1-1/2k^2)$.

Now, we consider any source strongly connected $S$ in $G$. For every pair of $(i,j)$ in $S$, there is a cycle $C$ that contains $(i,j)$. Without loss of generality, we write 
$C=(1 \to 2,\dots,m \to 1)$, where $m \le k$. We will show that every $h_i, i \in C$ has similar error. We prove this via contradiction. Assuming there is a pair of $(i,j)$ in the cycle such that $\norm{y-h_i}^2 \ge 25 \norm{y-h_i}^2$, then there must be an edge from $i$ to $j$. Without loss of generality, we assume $i=m$ and $j=1$, otherwise we can prove the same statement over a smaller cycle. Since for every $i \le m-1$, there is an edge from $i$ to $i+1$, this implies  
\begin{align*}
    \norm{h_{m-1}-y} \ge \left(1-\frac{1}{2k^2}\right)^k \norm{h_{1}-y} \ge \norm{h_{1}-y}/2 \ge \frac{5}{2} \norm{h_m-y}, 
\end{align*}
which implies that there should not be an edge from $m-1$ to $m$. Thus, for each pair of $(i,j)$ in a cycle $C$, $\norm{h_i-y}^2 \le 25 \norm{h_j-y}^2$. 
This implies that any source strongly connected component $S$ cannot contain a vertex $j$ such that $\norm{y-h_j}^2 \ge 25 \min_{i \in [k]} \E_{(x,y)\sim D}(y-h_i(x))^2$. To see why this is true, we consider two cases. Let $i^*$ be the vertex such that $\norm{h_{i^*}-y} = \min_{i \in [k]} \E_{(x,y)\sim D}(y-h_i(x))^2$. If $S$ contains $i^*$, since $S$ is a strongly connected component, each pair of $(i,i^*)$ is contained within some cycle and their error must be within a factor of $25.$ On the other hand, if $i^* \not\in S$ and $S$ contains some $j$ such that $\norm{y-h_j}^2 \ge 25 \min_{i \in [k]} \E_{(x,y)\sim D}(y-h_i(x))^2$, which implies that there must be an edge from $i^*$ to $j$ and $S$ cannot be a source strongly connected component. Notice that every time we make a comparison between $h_i,h_j$, we only use $\poly(k)$ samples, thus the total number of queries we make is $\poly(k)$.

Since for every pair of nodes $(i,j)$, there is a cycle such that $(i,j)$ are both in the cycle. In this case, $\norm{y-h_i}$ and $\norm{h_j-y}$ are within a factor. Let $i^*$ be the index such that $\norm{h_{i^*}-y}$ is minimized. We claim that $C$ does not contain some $j$ such that $\norm{h_{j}-y}^2\ge 10\norm{h_{i^*}-y}^2$. If so, then there must be an edge from $i^*$ to $j$. Since $C$ is a source strongly connected component, then $i^* \in C$, which implies that $(i,i^*)$ must be in the same cycle, which gives a contradiction.
\end{proof}

\subsection{Proof of \Cref{th spherical}}\label{app spherical}
Finally, we give a query-optimal learning algorithm that solves the general ReLU regression problem for the special case where $\norm{W^*}=1$ and $t^*>0$ is given.

\begin{theorem}[Query Learning for Spherical ReLU] \label{th spherical}
   Consider the problem of agnostic PAC learning GLM with membership queries under the Gaussian distribution. Suppose the optimal activation function $\sigma = \sigma_{t^*}^*$ is of the form $\sigma(w^*\cdot x -t^*)$, where $\sigma$ is the ReLU function, $w^* \in \s^{d-1}$ and $t^* \in \R$ is known,    
    there is an algorithm such that for every labeling function $y(x)$ and for every $\eps,\delta \in (0,1)$, it makes $M=\Tilde{O}_\delta(\min\{1/p, 1/\eps\} + d\cdot\polylog(1/\eps))$ memberships queries, runs in $\poly(d,M)$ time,
    where $p = \Phi(t^*)$ is the bias of the optimal activation function $\sigma^*$, 
    and outputs an $\hat{h} = \sigma(\hat{w}\cdot x - t^*),$ $\hat{w} \in \s^{d-1}$, such that with probability 
    at least $1-\delta$, $\err(\hat{h}) \le O(\opt)+\eps$.
\end{theorem}

\begin{algorithm}[h]
		\caption{\textsc{SphericalLearning}(Learn $w^*$ over the unit sphere)}\label{alg spherical refine}
		\begin{algorithmic} [1]
\State\textbf{Input:} error parameter $\epsilon \in (0,1)$, confidence parameter $\delta \in (0,1)$ 
\State\textbf{Output:} $\hat{h}: \R^d \to \R$, such that $\err(\hat{h}) \le O(\eps)$ with non-trivial probability.
\State Call \textsc{Initialization$(t^*)$} to get $w^{(0)}$ and return $\hat{h}(x) = 0$ if no $w^{(0)}$ is returned.
\For{$i=0,\dots,T-1$}
\For{$j=1,\dots,\Tilde{O}(d)$}
\State Generate $z^{(j)}>0$ with probability $\phi(t^*+z^{(i)})/\Phi(t^*)$ and $u^{(j)} \perp w^{(i)} \sim N(0,I)$
\State Query $y(x^{(j)})$, where $z^{(j)} w^{(i)} + u^{(j)}$
\EndFor
\State Estimate $\E\frac{1}{m} \sum_{j=1}^m \left(\sigma(w^{(i)} \cdot x^{(j)}-t^*) - y(x^{(j)})\right) u^{(j)}$ via median of mean and get $g^{(i)}$
\State $w^{(i+1)} = \proj_{\s^{d-1}} \left( w^{(i)} + \mu g^{(i)} \right)$
\EndFor
\State\Return $w^{(T)}$

\Procedure{Initialization}{Find a warm start $w^{(0)}$ given $t^*$}
    \State \textbf{Input: $t^*>0$ }
    \State \textbf{Output:} $w^{(0)} \in \s^{d-1}$ such that $\theta(w^{(0)},w^*) \le 1/(2t^*)$ or assert $\err(0) \le O(\eps)$.
    \If{$\Phi(t^*)/(t^*)^2 \le O(\epsilon)$}
        \State Assert $\err(0) \le O(\eps)$.
    \Else \State Run the initialization algorithm for query learning halfspaces (Algorithm 2/Algorithm 5 in \cite{diakonikolas2024active}) by simulating binary membership query with $\Bar{y}: = \Ind\{y>c/(t^*)\}$ for small constant $c>0$ and denote the return of the algorithm by $w^{(0)}$.
    \EndIf
    
\EndProcedure

\end{algorithmic}
\end{algorithm}

\begin{proof}[Proof of \Cref{th spherical}]
In each round of \Cref{alg spherical refine}, we use $\theta_i$ to denote $\theta(w^{(i)},w^*)$ and we denote by $G_i^* \in \R^d$ the random vector be the random vector $\left(\sigma(w^{(i)} \cdot x-t^*) - \sigma(w^* \cdot x-t^*)\right) \proj_{(w^{(i)})^\perp} x$ and $G_i$ be the random vector $\left(\sigma(w^{(i)} \cdot x-t^*) - y(x)\right)\proj_{(w^{(i)})^\perp} x$ , where $x \sim N(0,I) \mid_{\{x \mid w^{(i)} \cdot x> t^*\}}$.

By \Cref{lm reduction} and \Cref{lm halfspace}, we know that the subroutine \textsc{Initialization} takes $\Tilde{O}(1/p+d\log(1/\eps))$ queries and output a unit vector $w^{(0)}$ such that with probability at least $\log(1/\Phi(t^*))$, $\sin(\theta_0/2) \le 1/t^*$. 
In the rest of the proof, we assume \textsc{Initialization} succeeds.
Let $\alpha = 50$.
By \Cref{lm initialization}, let $\alpha = 50$, we know that if $\theta_i\le \theta_0$, then $\norm{T_{\sqrt{\cos \theta_i}} \sigma'}_2^2 \ge \norm{\sigma'}_2^2 / \alpha$. Let $\phi^* \in (0,\pi/2)$ such that $\sin^2 \phi^* = 20\alpha^2 \eps /(\pi \norm{\sigma'}_2^2)$. \Cref{lm stopping condition} implies that if $\theta_i \le \phi^*$, then $\err(w) \le O(\alpha^2 \eps) = O(\eps)$, since $\alpha = O(1)$.

Recall that for activation function $\sigma(z) = ReLU(z - t^*)$, we have $\norm{\sigma'}_2^2 = \Phi(t^*)$. We will show that for if $\theta_i \le \theta_0$ and $\theta_i \ge \phi^*$, then with high probability, $\theta_{i+1} \le (1-1/C) \theta_i$ for some constant $C>1$. Write $w^* = a w^{(i)} + b u$, where $u \in \s^{d-1}, a,b \ge 0, a^2+b^2=1$.
By \Cref{lm initialization variance}, we known that 
\begin{align*}
   & \E G_i^* = \proj_{(w^{(i)})^\perp} \nabla_w \ell(w^{(i)})/\Phi(t^*) =  b \norm{T_{\sqrt{a}}\sigma'}_2^2 u/\Phi(t^*) \\
   & \norm{\left( \E G^* - \E G \right)} = \norm{\proj_{w^\perp} \left( \nabla_w\err(w) - \nabla_w \ell(w) \right)}/\Phi(t^*)  \le \sqrt{\eps}\norm{\sigma'}_2/\Phi(t^*)
\end{align*}
By \Cref{lm stopping condition}, we know that if $\phi^*<\theta_i\le \theta_0$, then $\norm{\left( \E G^* - \E G \right)} \le \norm{\E G_i}/20$. This implies that 
\begin{align*}
     \E G_i \cdot u  = \E G_i \cdot u = \E G^*_i \cdot u + \E (G_i - G^*_i) \cdot u \ge \frac{19}{20} \norm{\E G_i} \ge 19 b \alpha^{-1} /20. 
\end{align*}
Furthermore, since $\norm{\E G} \le \frac{21}{20} \norm{\E G_i},$ \Cref{lm gradient descent}, we know that by estimating $G$ upto error $\norm{\E G}/40$, then $\theta_{i+1} \le (1-1/C) \theta_i$ for some large constant $C>0$.
By \Cref{lm initialization variance}, we know that $\E (G \cdot v)^2 \le O(b^2 \log(1/\eps)), \forall v \in \s^{d-1}$. By \Cref{lm variance-sample}, we know that in each round of update, we only need $\Tilde{O}(d)$ queries. Since we only need $\Tilde{O}(\log(1/\eps))$ rounds to make $\theta_i \le \theta^*$, the query complexity of \Cref{alg spherical refine} is $\Tilde{O}(1/\Phi(t^*)+d \poly\log(1/\eps))$.
\end{proof}

\section{Omitted Proofs from \Cref{sec regression}}\label{app regression}

In this section, we provide the missing details for solving the general ReLU regression problem.

\subsection{Proof of \Cref{lm Ini}}

We start with the initialization algorithm. We provide the proof 
of \Cref{lm Ini} and its restatement \Cref{lm Ini re} as follows. 
The key point of the lemma is that if we have a 
reasonable initial knowledge about $r^*,t^*$, 
we are still able to do the initialization.

\begin{lemma}[Initialization with Raw Knowledge]\label{lm Ini re}
    Let $\sigma$ be the ReLU activation function. Let $h^* = r^*\sigma(\bar{w^*}\cdot x - \bar{t}^*)$ be the optimal hypothesis. Suppose that $(r^*)^2V(\bar{t}^*) \ge \Omega(\eps)$, there is an algorithm such that
    given parameter $r,t>0$ such that $r \le r^* \le 2r$ and $\abs{t-\bar{t}^*} \le 1/\log(R^2/\eps)$, it makes $M=\Tilde{O}(1/p+d\log(R^2/\eps))$, runs in $\poly(d,M)$ time, and with probability at least $1/\log(1/p)$,
outputs some $w^{(0)} \in \s^{d-1}$ such that $\sin (\theta( w^{(0)}, w^*)/2) \le \min\{1/\bar{t}^*,1/2\}$.
\end{lemma}

\begin{proof}[Proof of \Cref{lm Ini}]
We consider the truncated label 
$\bar{y}(x) = \Ind \{y(x)>cr/t \}$, for a suitably small constant 
$c>0$. We will show that the truncated label $\bar{y}(x)$ can be 
seen as generated from the halfspace $h^*(x) = \sign(w^* \cdot x - \bar{t}^*)$ corrupted with adversarial label noise with level at 
most $\Phi(\bar{t}^*)/C$ for some large enough constant $C$. By 
the assumption that $(r^*)^2V(\bar{t}^*) \ge \Omega(\eps)$, we 
have $\bar{t}^* \le O(\sqrt{\log(R^2/\eps)})$, 
otherwise, $h(x) \equiv 0$ has error $O(\eps)$. On the other 
hand, we can assume $\bar{t}^* \ge 1$, otherwise, estimating 
$\E_{x \sim N(0,I)} y(x) x$ with constant error is enough to get 
a $w^{(0)}$ such that $\sin\theta(w^{(0)},w^*) \le 1/C$. With 
these assumptions, we use 
a similar argument as we did for the proof of \Cref{lm reduction} 
but with a more careful analysis.

We partition $\R^d$ into three regions $I_1:=\{x \mid w^* \cdot x> \bar{t}^* +2c/t\}, I_2:=\{x \mid w^* \cdot x< \bar{t}^*\}$ and $J:=\{x \mid \bar{t}^*< w^* \cdot x< \bar{t}^*+2c/t\}$. Since $\opt=\E_{x \sim N(0,I)} \left( r^*\sigma(w^*\cdot x -\bar{t}^*)- y\right)^2/2 < \eps$, we have 
    \begin{align*}
        \eps & \ge  \E_{x\sim N(0,I)} \left( r^*\sigma(w^*\cdot x -\bar{t}^*) \right)^2/2 \\
        & \ge \E_{x\sim N(0,I)} \Ind\{x \in I_1 \cup I_2\} \left( r^*\sigma(w^*\cdot x -\bar{t}^*) - y \right)^2/2 \\
        & \ge \E_{x\sim N(0,I)} \left(\Ind\{x \in I_1, y(x)< cr/t \} + \Ind\{x \in I_2, y(x)>cr/t\}\right) \left( r^*\sigma(w^*\cdot x -\bar{t}^*)-y \right)^2/2 \\
        & \ge \frac{c^2r^2}{2t^2} \Pr_{x \sim N(0,I)}\left( x \in I_1 \cup I_2, y(x) \neq \bar{y}(x) \right) \;.
    \end{align*}
Since $\E_{x \sim N(0,I)} (r^*)^2\sigma(w^*\cdot x - \bar{t}^*)^2 \ge C\eps $, we obtain that
\begin{align*}
    \Pr_{x \sim N(0,I)}\left( x \in I_1 \cup I_2, y(x) \neq \bar{y}(x) \right) \le \frac{2t^2}{r^2c^2} \eps \le \frac{2t^2 \E_{x \sim N(0,I)} (r^*)^2 \sigma(w^*\cdot x - \bar{t}^*)^2}{Cc^2} \le \Phi(\bar{t}^*)/C_1,
\end{align*}
for some large enough $C_1>0$. Here, we use the fact that $r \le r^* \le 2r$, $\bar{t}^* \le O(\sqrt{\log(R^2/\eps)})$ and $\abs{t-\bar{t}^*} \le 1/\log(R/\eps)$.

On the other hand, 
\begin{align*}
    \Pr_{x \sim N(0,I)}\left(x \in J, y(x) \neq \bar{y}(x)\right) \le \Pr_{x \sim N(0,I)}\left(x \in J\right) \le \frac{2c}{t} \frac{1}{\sqrt{2\pi}}\exp\left(-\frac{(\bar{t}^*)^2}{2}\right) \le 10c \Phi(\bar{t}^*).
\end{align*}
Thus, 
\begin{align*}
    \Pr_{x \sim N(0,I)} \left(\bar{y}(x) \neq \sign(w^*\cdot x - \bar{t}^*) \right) & = \Pr_{x \sim N(0,I)}\left(x \in J, y(x) \neq \bar{y}(x)\right) + \Pr_{x \sim N(0,I)}\left( x \in I_1 \cup I_2, y(x) \neq \bar{y}(x) \right)\\
    & \le \Phi(\bar{t}^*)/C_1 + 10c \Phi(\bar{t}^*) = \Phi(\bar{t}^*)/C''
\end{align*}
for some large enough $C''>0$.
By \Cref{lm halfspace}, we are able to efficiently obtain a direction $w^{0}$ that satisfies the statement of \Cref{lm Ini}.
\end{proof}

\subsection{Proof of \Cref{lm realizable error}}\label{app realizable}

We next present the proof of \Cref{lm realizable error} and its restatement \Cref{lm realizable error re} to decompose the noiseless error for the general ReLU regression problem.

\begin{lemma}\label{lm realizable error re}
Consider the problem of agnostic ReLU regression with queries. Let $r>0,w \in \s^{d-1}, t>0$ and $\theta(w,\bar{w^*}) = \theta$. If $t^* \le O(\sqrt{\log(R^2/\eps)})$ and $\abs{t-t^*} \le 1/\log(R^2/\eps)$, then 
    \begin{align*}
        \ell(r,w,t) \le (r^*)^2\int_0^\theta \sin s \norm{T_{\sqrt{\cos s}} \sigma'(z-\bar{t}^*)}_2^2 ds + \E_{z\sim N(0,I)}\left(\sigma(rz - t) - \sigma(r^*z - t^*)\right)^2.
    \end{align*}
\end{lemma}

\begin{proof}[Proof of \Cref{lm realizable error}]
    To simplify the notation, we denote by $h' = h(r^*,w,t^*)$. Notice that 
    \begin{align*}
        \ell(r,\bar{w},t) = \frac{1}{2} \E_{x \sim N(0,I)} \left(h - h ^*\right)^2 = \frac{1}{2} \E_{x \sim N(0,I)} \left(h - h' + h' -h ^*\right)^2 \le \E_{x \sim N(0,I)} \left(h - h' \right)^2 + \left( h' -h ^*\right)^2
    \end{align*}
On the one hand, we have 
\begin{align*}
    \E_{x \sim N(0,I)}\left(h - h ^*\right)^2 = \E_{x \sim N(0,I)}\left(\sigma(r w \cdot x - t) - \sigma(r^* w \cdot x - t^*)\right)^2 = \E_{z \sim N(0,1)}\left(\sigma(r z - t) - \sigma(r^* z - t^*)\right)^2
\end{align*}
On the other hand, we have
\begin{align*}
    \E_{x \sim N(0,I)}\left(h - h ^*\right)^2 & = \E_{x \sim N(0,I)}\left(\sigma(r^* w \cdot x - t^*) - \sigma(r^* w^* \cdot x - t^*)\right)^2\\
    &= (r^*)^2\E_{x \sim N(0,I)}\left(\sigma( w \cdot x - \bar{t}^*) - \sigma( w^* \cdot x - \bar{t}^*)\right)^2 \\
    & = (r^*)^2\int_0^\theta \sin s \norm{T_{\sqrt{\cos s}} \sigma'(z-\bar{t}^*)}_2^2 ds.
\end{align*}
Here, the last equation follows \Cref{lm noiseless}. This implies
\begin{align*}
        \ell(r,w,t) \le (r^*)^2\int_0^\theta \sin s \norm{T_{\sqrt{\cos s}} \sigma'(z-\bar{t}^*)}_2^2 ds + \E_{z\sim N(0,I)}\left(\sigma(rz - t) - \sigma(r^*z - t^*)\right)^2.
    \end{align*}
\end{proof}

\subsection{Overview of Query Learning Algorithm}\label{app overview}
In this section, we provide the pseudocode of the algorithm corresponding to \Cref{th query learning}. We present a detailed version of \Cref{th query learning} below.

\begin{theorem}\label{th query learning re}
     Consider the problem of agnostic general ReLU regression with queries under the Gaussian distribution.  
    There is an algorithm such that for every labeling function $y(x)$ and for every $\eps,\delta \in (0,1)$, it makes $M=\Tilde{O}_\delta(\min\{1/p, R^2/\eps\} + d\cdot\polylog(R^2/\eps))$ queries, runs in $\poly(d,M)$ time,
    where $p = \Phi(\bar{t}^*)$ is the bias of the optimal activation function, 
    and outputs an $\hat{h}$ such that with high probability 
    at least $1-\delta$, $\err(\hat{h}) \le O(\opt)+\eps$.
\end{theorem}

We remark that the dependence on $R^2/\eps$ is due to the natural scaling of the squared $\ell_2$ loss. If we want to learn $R\sigma(w^*\cdot x-t^*)$ up to error $\eps$, this is equivalent to learning $\sigma(w^* \cdot x -t^*)$ to error $\eps/R^2$.

Given a ReLU activation, $h(r_i,w^{(i)},t_i)$, 
define random vectors 
$$G^*_i:= \left( h(r_i,w^{(i)},t_i)(x) - h^*(x) \right)\proj_{(w^{(i)})^\perp}(x)$$ 
and its noisy version 
$$G_i:= \left( h(r_i,w^{(i)},t_i)(x) - y(x) \right)\proj_{(w^{(i)})^\perp}(x) \;,$$ 
where $x \sim N(0,I) \mid_{w^{(i)} \cdot x > \bar{t}_i}$.
Define random variables as follows
\begin{align*}
    U_i^*:= (h(r_i,w^{(i)},t_i) - h^*)(w^{(i)}\cdot x), F_i^*:= -(h(r_i,w^{(i)},t_i) - h^*),
\end{align*}
and denote by $U_i,F_i$, their noisy version, namely
\begin{align*}
    U_i:= (h(r_i,w^{(i)},t_i) - y)(w^{(i)}\cdot x), F_i:= -(h(r_i,w^{(i)},t_i) - y) \;.
\end{align*}

\begin{algorithm}
		\caption{\textsc{QueryLearning}(Learn optimal ReLU with a warm start)}\label{alg query learning}
		\begin{algorithmic} [1]
\State\textbf{Input:} $w^{(0)} \in \s^{d-1}:$ unit vector such that $\theta_0 \le 1/\polylog(R^2/\eps)$. $r_0>0:$ such that $\abs{r^*-r_0}\le r^*/\polylog(R^2/\eps)$, $t_0:\abs{t_0-t^*}\le \polylog(R^2/\eps)$.
\State\textbf{Output:} $\hat{h}: \R^d \to \R$, such that $\err(\hat{h}) \le O(\eps)$ with non-trivial probability.
\State $B_0: = r_0^2/\polylog(R^2/\eps)$
\For{$i=0,\dots,T-1$}
\State Generate $\polylog(R^2/\eps)$ samples $x^{(j)} \sim N(0,I) \mid \{ \bar{w^{(i)}} \cdot x > t_i\}$. Query $y(x^{(j)})$ and use them to get an estimate $\hat{g}_i$ for $(\E U_i,\E F_i)$.
\If{$\norm{\hat{g}_i} \ge B_i\polylog(R^2/\eps)$}
\State Set $(r_{i0},t_{i0}) = (r_i,t_i)$ and update $(r_{ij},t_{ij}) = (r_{i(j-1)},t_{i(j-1)})- (\hat{g_{ij}})/\polylog(R^2/\eps)$ until $\norm{\hat{g}_{ij}} \le B_i\polylog(R^2/\eps)$
\EndIf
\State $(r_{i+1},t_{i+1}) \gets (r_{ij},t_{ij})$

\For{$j=1,\dots,\Tilde{O}(d)$}
\State Generate $x^{(j)} \sim N(0,I) \mid \{ w^{(i+1)} \cdot x > t_{i+1}\}$ and query $y(x^{(j)})$
\EndFor
\State Estimate $\E G_i$ via median of mean and get $\hat{G}_{i}$
\State $w^{(i+1)} = \proj_{\s^{d-1}} \left( w^{(i)} - \mu \hat{G}_{i} \right)$
\State $B_{i+1} = (1-\rho) B_i$
\EndFor
\State $\hat{w} = w^{(T)}$
\State Build a unit grid of size $1/\polylog(R^2/\eps)$ over the ball centered at $(r_T,t_T)$ and randomly select a pair $(\hat{r},\hat{t})$ from the grid.
\State\Return $h(\hat{r},\hat{w},\hat{t})$

\end{algorithmic}
\end{algorithm}

We now proceed with an intuitive explanation of 
\Cref{alg query learning} and its analysis. The detailed analysis is carried out in the rest of the section. 
We are given a tuple of initial parameters 
$(r_0,w^{(0)},t_0)$ such that each of them are close to 
the true parameters up to a $1 \pm 1/\polylog(R^2/\eps)$ 
factor. We remark that by \Cref{lm Ini}, we are only able 
to get some $w^{(0)}$ such that 
$\theta_0 \le 1/\bar{t}^*$, 
but as we will show in \Cref{app angle update}, since the other parameters are close enough, 
by updating $w^{(0)}$ for $\log\log (R^2/\eps)$ rounds, 
we are able to achieve this guarantee. 
We will maintain an upper bound $B_i^2$ for 
$(r^*)^2\sin^2\theta$ and reduce this upper bound stably 
in each round of the algorithm. In our analysis, we will 
show that as long as the error from $(r_i,t_i)$ is within 
a $\polylog(R^2/\eps)$ factor of $B_i^2$, we are able to 
make roughly $d\polylog(R^2/\eps)$ queries to safely 
improve $w^{(i)}$ and decrease the angle $\theta_i$ by a 
small constant factor.

Regarding the parameters $(r_i,t_i)$, since they are close 
to $(r^*,t^*)$, the task for optimizing them can be 
approximately seen as optimizing the following quadratic function 
$$Z(r,t):= \E_{x \sim N(0,I)}\left((rz-t)-(r^*z-t^*)\right)^2\Ind(z>\bar{t}^*) \;.$$ 
Such a quantity nearly characterizes the contribution of $(r_i,t_i)$ 
to the noiseless error as well as the contribution of $(r_i,t_i)$ 
to the variance of the gradient we use for updating $w^{(i)}$.
As we will show in \Cref{app tr update}, 
if $B_i^2\Phi(\bar{t}^*) \le Z(r_i,t_i)/\polylog(R^2/\eps)$, 
i.e., $\theta_i$ has already been updated in a reasonable range, 
then $(\E U_i, \E F_i)$ is very close to the gradient of $Z(r_i,t_i)$. 
Thus, we are able to tell whether $Z(r_i,t_i)$ is desirable 
by checking the norm of $(\E U_i, \E F_i)$ 
due to the nature of quadratic minimization. 
When $Z(r_i,t_i)$ is large, we use a standard gradient descent to update these 
parameters such that $Z(r_i,t_i)$ is controlled by $B_{i+1}$ 
within a polylogorithmic factor. This allows us to safely update $w^{(i)}$ 
via gradient descent. When $Z(r_i,t_i)$ is already close to $B_{i+1}$, 
we do not update it.
We remark that due to the noise and the ill-condition of $Z$ for large $t^*$, 
the step size we choose is $1/\polylog(R^2/\eps)$ instead of a constant.
After at most $\polylog(R^2/\eps)$ rounds, 
$w^{(T)}$ is updated to a desirable accuracy. 
However, due to the presence of noise, we are only able to guarantee that 
$(r_T,t_T)$ is $O(\eps\polylog(R^2/\eps)/\Phi(\bar{t^*}))$ 
close to $(r^*,t^*)$ in terms of squared norm. Fortunately, 
there are only two parameters we need to worry about 
and they are already very close to $(r^*,t^*)$; 
randomly selecting a pair of parameters from their neighborhoods 
gives us a good hypothesis with enough probability.

\subsection{Omitted Details Regarding Angle Update}\label{app angle update}

In this section, we provide the full details on the subroutine for updating the 
direction $w^{(i)}$. Specifically, we will consider the following random vectors 
in $\R^d$. Given a ReLU activation, $h(r_i,w^{(i)},t_i)$, 
define random vectors 
$G^*_i:= \left( h(r_i,w^{(i)},t_i)(x) - h^*(x) \right)\proj_{(w^{(i)})^\perp}(x)$ 
and its noisy version 
$G_i:= \left( h(r_i,w^{(i)},t_i)(x) - y(x) \right)\proj_{(w^{(i)})^\perp}(x)$, 
where $x \sim N(0,I) \mid_{w^{(i)} \cdot x > \bar{t}_i}$.

\subsubsection{Proof of \Cref{lm accurate update}}

We first give an evaluation of the mean and the noise analysis of $G$. We provide the proof of \Cref{lm accurate update} and its restatment \Cref{lm accurate update re}.

\begin{lemma}\label{lm accurate update re}
Consider the problem of agnostic ReLU regression with queries. Let $h(r_i,\bar{w}^{(i)},t_i)$.
Write $w^* = a w^{(i)} + bu$, where $a,b>0, a^2+b^2 = 1$, $u \in \s^{d-1}, u \perp \bar{w}^{(i)}$. Then the following statements hold.
\begin{enumerate}
    \item If $\abs{\bar{t}_i-\bar{t}^*} \le 1/\log(R^2/\eps)$ and $b \le 1/\bar{t_i}$ , then 
    \[\E G^*_i = -\alpha br^* \norm{T_a \sigma'(z-\bar{t}_i)}^2 u / \Phi(\bar{t}_i)\]
where $1/2< \alpha < 2$.
    \item  $\abs{\left( \E G^* - \E G \right) \cdot v} \le \sqrt{\eps}\norm{\sigma'(z-t_i)}_2/\Phi(\bar{t}_i)$
    , $\forall v \in \s^{d-1}, v \perp w^{(i)}$.
    \end{enumerate}
\end{lemma}

\begin{proof}[Proof of \Cref{lm accurate update}]
    We first consider the mean of $G^*_i$. 
    By \Cref{fact ou} and \Cref{fact stein}, we have
        \begin{align*}
        \E G^*_i & = \E_{x \sim N(0,I) \mid_{\{w^{(i)} \cdot x > \bar{t_i}\}}}\left(h(r_i,w^{(i)},t_i)(x) - h^*(x)\right) \proj_{w^\perp} x \\
        & = \E_{x \sim N(0,I)} \Ind\{w^{(i)} \cdot x > \bar{t}_i\} \left(h(r_i,w^{(i)},t_i)(x) - h^*(x)\right) \proj_{(w^{(i)})^\perp} x / \Phi(\bar{t}_i) \\
        & = - \proj_{(w^{(i)})^\perp} \E_{x \sim N(0,I)} \Ind\{w^{(i)} \cdot x > \bar{t}_i\} h^*(x)  x / \Phi(\bar{t}_i) \\
        & = -  \E_{x \sim N(0,I)} \Ind\{w^{(i)} \cdot x > \bar{t}_i\} \Ind\{w^{*} \cdot x > \bar{t}^*\}br^* u / \Phi(\bar{t}_i) \\
        &= -  \E_{x \sim N(0,I)} \Ind\{w^{(i)} \cdot x > \bar{t}_i\} \Ind\{w^{*} \cdot x > \bar{t}_i\}br^* u / \Phi(t_i)\\
        &+  \E_{x \sim N(0,I)} \Ind\{w^{(i)} \cdot x > \bar{t}_i\}(\Ind\{w^{*} \cdot x > \bar{t}_i\}-\Ind\{w^{*} \cdot x > t^*\})br^* u / \Phi(\bar{t}_i) \\
        & = -br^* \norm{T_a \sigma'(z-\bar{t}_i)}_2^2 u / \Phi(\bar{t}_i) +  \E_{x \sim N(0,I)} \Ind\{w^{(i)} \cdot x > \bar{t}_i\}(\Ind\{w^{*} \cdot x > \bar{t}_i\}-\Ind\{w^{*} \cdot x > t^*\})br^* u / \Phi(\bar{t}_i)
    \end{align*}
By \Cref{lm initialization}, we know that $\norm{T_a \sigma'(z-\bar{t}_i)}^2 \ge \Omega(1)\Phi(\bar{t}_i)$. Furthermore, we have
\begin{align*}
\abs{ \E_{x\sim N(0,I)} \Ind\{\bar{w}^{(i)} \cdot x > \bar{t}_i\}(\Ind\{\bar{w}^{*} \cdot x > \bar{t}_i\}-\Ind\{\bar{w}^{*} \cdot x > \bar{t}^*\}) } \le \abs{\E_{x\sim N(0,I)} (\Ind\{\bar{w}^{*} \cdot x > \bar{t}_i\}-\Ind\{\bar{w}^{*} \cdot x > \bar{t}^*\})} =o ( \Phi(\bar{t}_i)).
\end{align*}
We conclude that $\E G^*_i = -\alpha br^* \norm{T_a \sigma'(z-\bar{t}_i)}^2 u / \Phi(\bar{t}_i)$
where $1/2< \alpha < 2$.

We next analyze the noise term. For every $v \in \s^{d-1}$ and $v \perp w^{(i)}$, by Holder's inequality, we have
\begin{align*}
      \abs{\E (G_i^* -G_i) \cdot v}   & = \abs{ \E_{x \sim N(0,I) \mid_{\{w^{(i)} \cdot x > \bar{t}_i\}}}\left(y(x) - h^*(x)\right) (x \cdot v))} \\
        & = \abs{ \E_{x \sim N(0,I)} \Ind\{w^{(i)} \cdot x > \bar{t}_i\} \left(y(x) - h^*(x)\right) (x \cdot v))}/\Phi(\bar{t}_i) \\
        & \le \sqrt{\E_{x \sim N(0,I)}(y(x) - h^*(x))^2 }\sqrt{\E_{x \sim N(0,I)} \Ind\{w^{(i)} \cdot x > \bar{t}_i\}(x \cdot v)^2}\\
        & \le \sqrt{\eps}\norm{\sigma'(z-\bar{t}_i)}_2/\Phi(\bar{t}_i).
    \end{align*}
\end{proof}

\subsubsection{Proof of \Cref{lm angle variance}}\label{app angle variance}
We next provide the evaluation for the variance of $G$. We provide the proof of \Cref{lm angle variance} and its restatement \Cref{lm angle variance re}.

\begin{lemma}\label{lm angle variance re}
     Let $h(r_i,\bar{w}^{(i)},t_i)$ be a ReLU activation.
Write $w^* = a w^{(i)} + bu$, where $a,b>0, a^2+b^2 = 1$, $u \in \s^{d-1}, u \perp \bar{w}^{(i)}$.
     If $\abs{\bar{t}_i-\bar{t}^*} \le 1/\log(R^2/\eps)$, $C\eps/((r^*)^2\Phi(t^*)) \le b^2 \le 1/\bar{t}^2_i$,
     then 
     \begin{align*}
         \E  (G_i \cdot v)^2 \le \Tilde{O}\left( \E_{z\sim N(0,I)}\left( \sigma(r_iz-t_i)- \sigma(r^*z-t^*)\right)^2/\Phi(\bar{t}_i) +(r^*)^2b^2\right),
     \end{align*}
     for every $v \in \s^{d-1}$, $v \perp w^{(i)}$.
\end{lemma}

\begin{proof}[Proof of \Cref{lm angle variance}]
For every $v \in \s^{d-1}$, $v \perp w^{(i)}$ we have
    \begin{align*}
     \E  (G_i \cdot v)^2  \le  2\E  ((G_i-G_i^*) \cdot v)^2 +2\E  (G_i^* \cdot v)^2. 
 \end{align*}
We bound the two terms separately.
\begin{align*}
    \E  (G^*_i \cdot v)^2 = & \E_{x \sim N(0,I) \mid_{\{w^{(i)} \cdot x > t_i\}}}\left(h(r_i,w^{(i)},t_i) - h^*\right)^2 (x \cdot v)^2 \\
    & = \E_{x \sim N(0,I)} \Ind\{w^{(i)} \cdot x > \bar{t}_i\}\left(h(r_i,w^{(i)},t_i) - h^*\right)^2 (x \cdot v)^2/\Phi(\bar{t}_i) \\
\end{align*}
We next expand 
$\left(h(r_i,w^{(i)},t_i) - h^*\right)^2$ as follows: 
\begin{align*}
    &\left(h(r_i,w^{(i)},t_i) - h^*\right)^2\\
    &= \left(h(r_i,w^{(i)},t_i) - h(r^*,w^{(i)},t^*) + h(r^*,w^{(i)},t^*) - h(r^*,w^*,t^*) \right)^2 \\
    & \le 2 \left(h(r_i,w^{(i)},t_i) - h(r^*,w^{(i)},t^*)\right)^2 + 2\left(h(r^*,w^{(i)},t^*) - h(r^*,w^{(i)},t^*)\right)^2.
\end{align*}
  For the first term, we have 
 \begin{align*}
   &  \E_{x \sim N(0,I)} \Ind\{w^{(i)} \cdot x > \bar{t}_i\}\left(h(r_i,w^{(i)},t_i) - h(r^*,w^{(i)},t^*)\right)^2 (x \cdot v)^2/\Phi(\bar{t}_i)  \\
   \le  &\E_{z\sim N(0,I)} \Ind\{w^{(i)} \cdot x > \bar{t}_i\} \left( \sigma(r_iz-t_i)- \sigma(r^*z-t^*)\right)^2/\Phi(\bar{t}_i)
 \end{align*}

For the second term,  notice that 
\begin{align*}
    \left(\sigma(w^{(i)}\cdot x - \bar{t}^*)- \sigma(w^*\cdot x - \bar{t}^*)\right)^2 \le \left(((1-a)w^{(i)} -bu) \cdot x\right)^2 \le 2(1-a)^2 (w^{(i)}\cdot x)^2 + 2b^2(u \cdot x)^2.
\end{align*}
Since $(1-a) = O(b^2)$ and $b = O(\bar{t}_i)$, 
this gives 
\begin{align*}
   & \E_{x \sim N(0,I)} \Ind\{w^{(i)} \cdot x > \bar{t}_i\}\left(\sigma(w^{(i)}\cdot x - \bar{t}^*)- \sigma(w^*\cdot x - \bar{t}^*)\right)^2(x \cdot v)^2/\Phi(\bar{t}_i)\\
\le & \E_{x \sim N(0,I)} \Ind\{w^{(i)} \cdot x > \bar{t}_i\}\left(2(1-a)^2 (w^{(i)}\cdot x)^2 + 2b^2(u \cdot x)^2\right)^2(x \cdot v)^2/\Phi(\bar{t}_i)\\
 \le & O(b^4 \bar{t}_i^2 + b^2) = O(b^2)\\
\end{align*}
This gives 
\begin{align*}
    \E  (G^*_i \cdot v)^2 \le O\left(\E_{z\sim N(0,I)}\left( \sigma(r_iz-t_i)- \sigma(r^*z-t^*)\right)^2/\Phi(\bar{t}_i) +(r^*)^2b^2 \right).
\end{align*}

 Next, we bound the variance of the noisy term.
 \begin{align*}
     \E \left((G^* -G) \cdot v\right)^2   & = \abs{ \E_{x \sim N(0,I) \mid_{\{w^{(i)} \cdot x > \bar{t}_i\}}}\left(y(x) - h^*(x)\right)^2 (x \cdot v)^2)} \\
     & = \E_{x \sim N(0,I)}\left(y(x) - h^*(x)\right)^2 (x \cdot v)^2 \Ind\{w^{(i)} \cdot x > \bar{t}_i\}/\Phi(\bar{t}_i) \;.
 \end{align*}
 
 Let $M>0$ be a threshold such that $\E_{z\sim N(0,1)} z^2 \Ind\{\abs{z}>M\} \le \eps/R^2$. Notice that if we set $y' = \sign(y) \min\{\abs{y},r_i M\}$, then we will only introduce at most $\eps$ error since $\abs{r_i - r^*} \le r^*/\log(R/\eps)$. So, we can without loss of generality, assume $\abs{y} \le M$. In particular, for ReLU activation, $M \le O(\sqrt{\log(R^2/\eps)}).$ Based on this, we obtain that 
 \begin{align*}
     \E \left((G_i^* -G_i) \cdot v\right)^2 & = \E_{x \sim N(0,I)}\left(y(x) - h^*(x)\right)^2 (x \cdot v)^2 \Ind\{w^{(i)} \cdot x > \bar{t}_i\}\Ind\{(x\cdot v)\le M\}/\Phi(\bar{t}_i) \\ 
     &+ \E_{x \sim N(0,I)}\left(y(x) - h^*(x)\right)^2 (x \cdot v)^2 \Ind\{w^{(i)} \cdot x > \bar{t}_i\}\Ind\{(x\cdot v)> M\}/\Phi(\bar{t}_i) \\
     & \le \opt M^2/\Phi(t^*) + 4M^2 \E_{x \sim N(0,I)} (x \cdot v)^2 \Ind\{w^{(i)} \cdot x > \bar{t}_i\}\Ind\{(x\cdot v)> M\}/\Phi(\bar{t}_i) \\
     & \le \eps M^2/\Phi(\bar{t}_i) + 4\eps M^2/\Phi(\bar{t}_i) \le \sin^2\theta M^2 = O((r^*)^2b^2 \log(R^2/\eps))
 \end{align*}
Thus, we conclude that 
\begin{align*}
     \E  (G_i \cdot v)^2 \le \Tilde{O}\left( \E_{z\sim N(0,I)}\left( \sigma(r_iz-t_i)- \sigma(r^*z-t^*)\right)^2/\Phi(\bar{t}_i) +(r^*)^2b^2\right).
\end{align*}
\end{proof}

\subsubsection{Progress on Angle Update}\label{app angle progress}
In this section, we analyze the progress made over the update of 
$w^{(i)}$. We present the following lemma for measuring the progress on 
$w^{(i)}$.

\begin{lemma}\label{lm angle progress}
    Consider the problem of agnostic ReLU regression with queries. Let $h(r_i,w^{(i)},t_i)$ be a ReLU activation function. Write $w^* = a w^{(i)} + bu$, where $a,b>0, a^2+b^2 = 1$, $u \in \s^{d-1}, u \perp w^{(i)}$. Suppose $\abs{\bar{t}_i-\bar{t}^*} \le 1/\log(R^2/\eps)$, $b \le 1/\bar{t_i}$ and $\abs{r_i-r^*} \le r^*/\polylog(R^2/\eps)$. Let $B_i>0$ be a parameter that satisfies the following property: 
    \begin{enumerate}
        \item $\E_{z\sim N(0,I)}\left( \sigma(r_iz-t_i)- \sigma(r^*z-t^*)\right)^2 \le B_i^2 \Phi(\bar{t}^*)$,
        \item $B_i^2 \Phi(\bar{t}^*) \ge \Omega(\eps)$.
    \end{enumerate}
     Define $B_{i+1} = (1-1/C)B_i$ for some universal constant $C>0$. If $(r^*)^2\sin^2\theta_i \le B_i^2 \beta^2 $, $0<\beta<1/10$, then the update $w^{(i+1)} = \proj_{\s^{d-1}}(w^{(i)} - \mu_i\hat{G}_i)$ satisfies $(r^*)^2\sin^2\theta_{i+1} \le B_{i+1}^2 \beta^2$, where $\mu_i$ is a small constant and $\hat{G_i}$ is an estimation of $\E G_i$ with $\tilde{O}(d/\beta^2)$ samples.
\end{lemma}

\begin{proof}[Proof of \Cref{lm angle progress}]
Write $\bar{w}^* = a \bar{w}^{(i)} + bu$, where $a,b>0, a^2+b^2 = 1$, $u \in \s^{d-1}, u \perp \bar{w}^{(i)}$. 
Since $\sin \theta_{i} \le 1/\bar{t_i}$, $\abs{\bar{t}_i-\bar{t}^*} \le 1/\log(R^2/\eps)$ and $\abs{r_i-r^*} \le r^*/\polylog(R^2/\eps)$,
by \Cref{lm angle variance} and \Cref{lm variance-sample}, we know that with $\tilde{O}(d)$ samples of $G_i$, we have $\norm{\hat{G}_i - \E G_i} \le \beta B_i/1000$. We consider two cases for 
$(r^*)^2\sin^2\theta_{i}$.
In the first case, we assume that $ \beta^2 B^2_{i}/2 \le (r^*)^2\sin^2\theta_{i} \le \beta^2B^2_{i}$. By \Cref{lm accurate update}, we know that $\norm{\hat{G}_i - \E G_i} \le \beta\norm{\E G_i}/C_1$, for some large enough constant $C_1>0$. By \Cref{lm gradient descent}, we know that there is a small constant $c_2>0$ such that $\sin \theta_{i+1} \le (1-c_2) \sin\theta_i$, which implies that $(r^*)^2\sin^2\theta_{i+1} \le \beta^2 B_{i+1}^2$. 

In the second case, we have $(r^*)^2\sin^2\theta_{i} \le \beta^2B^2_{i}/2$. Notice that 
    \begin{align*}
2(\sin(\frac{\theta_{i+1}}{2}) - \sin(\frac{\theta_i}{2})) = \norm{\bar{w}^{(i+1)}-\bar{w}^*} - \norm{\bar{w}^{(i)}-\bar{w}^*} \le \norm{\bar{w}^{(i+1)}-\bar{w}^*} \le \norm{\bar{w}^{(i)}+\mu_i\hat{G_i} - \bar{w}^{(i)}} = \norm{\mu_i\hat{G_i}}.
\end{align*}
Then we have 
\begin{align*}
    \beta B_{i+1} - r^* \sin(\frac{\theta_{i+1}}{2}) & = \beta B_{i+1} - r^* \sin(\frac{\theta_{i}}{2}) - r^*(\sin(\frac{\theta_{i+1}}{2})-\sin(\frac{\theta_{i}}{2})) \\
    & \ge B_{i+1}\beta - \beta B_i/2 - r^* \norm{\mu_i \hat{G_i}} >0 \;.
\end{align*}
Thus, in both cases, we have $(r^*)^2\sin^2\theta_{i+1} \le \beta^2 B_{i+1}^2$.
\end{proof}

We discuss the implication of \Cref{lm angle progress} 
before providing the details for updating the parameters $(r,t)$. 
In our application, we will choose $\beta=1/\polylog(R^2/\eps)$.
As we discussed in \Cref{sec regression}, we have initial parameters $r_0,t_0$ such that $\abs{r_0 - r^*} \le r^*/\polylog(R^2/\eps)$ and $\abs{\bar{t}_0-\bar{t}^*} \le 1/\polylog(R^2/\eps)$. This implies that 
\begin{align*}
&\E_{z\sim N(0,I)}\left( \sigma(r_0z-t_0)- \sigma(r^*z-t^*)\right)^2 \\
\le & 2 \E_{z\sim N(0,I)}\left( (r_0-r^*) \sigma(z-\bar{t}_0)\right)^2 + \E_{z\sim N(0,I)} (r^*)^2(\sigma(z-\bar{t}_i)-\sigma(z-\bar{t}^*))^2     \\
\le &O((r_0-r^*)^2V(\bar{t}^*)+(r^*)^2(\bar{t}_0-\bar{t}^*)^2\Phi(\bar{t}^*))  \le (r^*)^2\Phi(\bar{t^*})/\polylog(R^2/\eps).
\end{align*}
If $\bar{t}^*$ is large enough such that 
for $\theta\ge 1/\polylog(R^2/\eps)$, $(r^*)^2\sin\theta^2\Phi(\bar{t}^*) \le C\eps$, 
then by \Cref{lm angle progress}, starting from $w^{(0)}$ 
and updating $w^{(0)}$ at most $\log\log(R^2/\eps)$ steps, 
we obtain a hypothesis with error $O(\eps)$. 
Thus, in the rest of the section, we will assume that 
$w^{(0)}$ satisfies $\sin\theta_0 \le 1/\polylog(R^2/\eps)$.

\subsection{Omitted Details in $(r,t)$ Updates}\label{app tr update}
Here we provide the details for updating the parameters $(r,t)$.
For a given ReLU activation $h(r_i,w^{(i)},t_i)$, 
we define the following random variables: 
\begin{align*}
    U_i^*:= (h(r_i,w^{(i)},t_i) - h^*)(w^{(i)}\cdot x), F_i^*:= -(h(r_i,w^{(i)},t_i) - h^*),
\end{align*}
and denote by $U_i,F_i$, their noisy version, namely
\begin{align*}
    U_i:= (h(r_i,w^{(i)},t_i) - y)(w^{(i)}\cdot x), F_i:= -(h(r_i,w^{(i)},t_i) - y) \;.
\end{align*}
Here, $x \sim N(0,I) \mid_{w^{(i)} \cdot x > \bar{t}_i}$.

When the direction $w$ is fixed, we will also use $U(r,t),F(r,t)$ defined as follows for convenience:
\begin{align*}
    U(r,t):= (h(r,w,t) - y)(w\cdot x), F(r,t):= -(h(r,w,t) - y).
\end{align*}
We give quantitative characterizations for these random variables. We write $w^* = a w^{(i)} + bu$, where $a,b>0, a^2+b^2 = 1$, $u \in \s^{d-1}, u \perp w^{(i)}$. We furthermore assume that the parameters $h(r_i,w^{(i)},t_i)$ are close to $(r^*,w^*,t^*)$. That is to say $\theta_i:=\theta(w^*,w^{(i)}) \le O(1/(\bar{t}^*)^2)$, $\abs{r_i-r^*} \le r^*/\polylog(R^2/\eps)$, $\abs{\bar{t}_i-\bar{t}^*} \le 1/\polylog(R^2/\eps)$.

\paragraph{Evaluation of $\E U^*_i$ and $ \E F^*_i$}

    We first evaluate the expectation of $U^*$. To do this, we first expand $h(r_i,w^{(i)},t_i) - h^*$ as follows
    \begin{align*}
        h(r_i,w^{(i)},t_i) - h^* = h(r_i,w^{(i)},t_i) - h(r^*,w^{(i)},t^*) + h(r^*,w^{(i)},t^*)- h^*.
    \end{align*}
Since
    \begin{align*}
    \E U^*_i & = \E_{x \sim N(0,I)}(h(r_i,w^{(i)},t_i) - h^*)(w^{(i)}\cdot x)\Ind\{w^{(i)}\cdot x >\bar{t}_i\}/\Phi(\bar{t_i}) \\
    = &\E_{x \sim N(0,I)}(h(r_i,w^{(i)},t_i) - h(r^*,w^{(i)},t^*) + h(r^*,w^{(i)},t^*)- h^*)(w^{(i)}\cdot x)\Ind\{w^{(i)}\cdot x >\bar{t}_i\}/\Phi(\bar{t_i}),
    \end{align*}
we evaluate the two components separately.
For the first term, we have 
\begin{align*}
    &\E_{x \sim N(0,I)}(h(r_i,w^{(i)},t_i) - h(r^*,w^{(i)},t^*)(w^{(i)}\cdot x)\Ind\{w^{(i)}\cdot x >\bar{t}_i\}/\Phi(\bar{t_i})\\
    = &\E_{z\sim N(0,I)} \left(\sigma(r_i z - t_i)- \sigma(r^* z - t^*)\right)z\Ind\{z >\bar{t}_i\}/\Phi(\bar{t}_i)
\end{align*}

We next consider the second term. 
By \Cref{fact ou}, we have 
\begin{align*}
   & \E_{x \sim N(0,I)}(h(r^*,w^{(i)},t^*) - h^*)(w^{(i)}\cdot x)\Ind\{w^{(i)}\cdot x >\bar{t}_i\}/\Phi(\bar{t_i})  \\
    =&\left( r^*/\Phi(\bar{t}^*) \right) \E_{z \sim N(0,1)} \left(\sigma(z-\bar{t}^*) - T_a\sigma(z-\bar{t}^*) \right) z \Ind \{z>\bar{t}_i\} \\
    = &\left( r^*/\Phi(\bar{t}^*) \right) \E_{z \sim N(0,1)} \left(\int_a^1 \frac{d T_s \sigma(z-\bar{t}^*)}{ds} ds \right) \sigma'(z-\bar{t}_i) z = \left( r^*/\Phi(\bar{t}^*) \right) \int_a^1 \E_{z \sim N(0,1)}  \frac{d T_s \sigma(z-\bar{t}^*)}{ds}  \sigma'(z-\bar{t}_i) z ds \\
    = & \left( r^*/\Phi(\bar{t}^*) \right) \int_a^1 \E_{z \sim N(0,1)}  \frac{1}{s} L T_s \sigma(z-\bar{t}^*) (z\sigma'(z-\bar{t}_i))  ds = \left( r^*/\Phi(\bar{t}^*) \right) \int_a^1 \E_{z \sim N(0,1)}  \frac{1}{s} (L T_s \sigma(z-\bar{t}^*))' (z\sigma'(z-\bar{t}_i))'  ds \\
    = & \left( r^*/\Phi(\bar{t}^*) \right) \int_a^1 \E_{z \sim N(0,1)}   T_s \sigma'(z-\bar{t}^*) \left(\sigma'(z-\bar{t}_i)+z\delta(z-\bar{t}_i)\right)  ds \\
    = & \left( r^*/\Phi(\bar{t}^*) \right) \int_0^\theta \sin s \E_{z \sim N(0,1)}   T_{\cos s} \sigma'(z-\bar{t}^*) \left(\sigma'(z-\bar{t}_i)+z\delta(z-\bar{t}_i)\right)  ds.
\end{align*}
We notice that 
\begin{align*}
    & \E_{z \sim N(0,1)}   T_{\cos s} \sigma'(z-\bar{t}^*) \sigma'(z-\bar{t}_i) \\
    = &\E_{z \sim N(0,1)}   T_{\cos s} \sigma'(z-\bar{t}^*) \sigma'(z-\bar{t}^*) - \E_{z \sim N(0,1)}   T_{\cos s} \sigma'(z-\bar{t}^*) (\sigma'(z-\bar{t}^*)-\sigma'(z-\bar{t}_i)) \\
    = & \E_{z \sim N(0,1)}   \norm{T_{\sqrt{\cos s}}\sigma'(z-\bar{t}^*)}_2^2 - \E_{z \sim N(0,1)}   T_{\cos s} \sigma'(z-\bar{t}^*) (\sigma'(z-\bar{t}^*)-\sigma'(z-\bar{t}_i)) = \Theta(\norm{T_{\sqrt{\cos s}}\sigma'(z-\bar{t}^*)}_2^2),
\end{align*}
when $\sin s < 1/\bar{t}^*$ and $\abs{\bar{t}_i-\bar{t}^*}\le O(\log(R/\eps))$.
This implies 
\begin{align*}
    \left( r^*/\Phi(\bar{t}^*) \right) \int_0^\theta \sin s \E_{z \sim N(0,1)}   T_{\cos s} \sigma'(z-\bar{t}^*) \left(\sigma'(z-\bar{t}_i))\right)  ds = \Theta((r^*/\Phi(\bar{t}^*)) b^2 \norm{T_{\sqrt{\cos \theta}}\sigma'(z-\bar{t}^*)}_2^2).
\end{align*}
On the other hand,
\begin{align*}
    \E_{z \sim N(0,1)}   T_{\cos s} \sigma'(z-\bar{t}^*) z\delta(z-\bar{t}_i) = T_{\cos s} \sigma'(\bar{t}_i-\bar{t}^*) t \psi(t) = \Pr_{\beta \sim N(0,1)}(\cos s \bar{t}_i+\sin s\beta-\bar{t}^*>0)\bar{t}_i\psi(\bar{t}_i).
\end{align*}
This implies that 
\begin{align*}
     &\left( r^*/\Phi(\bar{t}^*) \right) \int_0^\theta \sin s \E_{z \sim N(0,1)}   T_{\cos s} \sigma'(z-\bar{t}^*)z\delta(z-\bar{t}_i)  ds \\
      = & r^*  \int_0^\theta \sin s \Pr_{\beta \sim N(0,1)}(\cos s \bar{t}_i+\sin s\beta-\bar{t}^*>0)\bar{t}_i\psi(\bar{t}_i) ds \\
      \le & O(r^*b^2\bar{t}_i\psi(\bar{t}_i)/\Phi(\bar{t}_i)) 
      = O(r^*b).
\end{align*}
Putting everything together, we get
\begin{align*}
    \E U_i^* = \E_{z\sim N(0,I)} \left(\sigma(r_i z - t_i)- \sigma(r^* z - t^*)\right)z\Ind\{z >\bar{t}_i\}/\Phi(\bar{t}_i) + cr^*b
\end{align*}
for some suitable constant $c>0$.

We next evaluate $\E F^*_i$ in a similar way.
\begin{align*}
    \E F^*_i & = \E_{x \sim N(0,I)}(h(r_i,w^{(i)},t_i) - h^*)\Ind\{w^{(i)}\cdot x >\bar{t}_i\}/\Phi(\bar{t_i}) \\
    = &\E_{x \sim N(0,I)}(h(r_i,w^{(i)},t_i) - h(r^*,w^{(i)},t^*) + h(r^*,w^{(i)},t^*)- h^*)\Ind\{w^{(i)}\cdot x >\bar{t}_i\}/\Phi(\bar{t_i}) \\
    =& \E_{z\sim N(0,I)} \left(\sigma(r_i z - t_i)- \sigma(r^* z - t^*)\right)/\Phi(\bar{t}_i) +\E_{x \sim N(0,I)}(h(r^*,w^{(i)},t^*) - h^*)\Ind\{w^{(i)}\cdot x >\bar{t}_i\}/\Phi(\bar{t_i})
\end{align*}
For the second component in $\E F_i^*$, we have
\begin{align*}
   & \E_{x \sim N(0,I)}(h(r^*,w^{(i)},t^*) - h^*)\Ind\{w^{(i)}\cdot x >\bar{t}_i\}/\Phi(\bar{t_i})  \\
    =&\left( r^*/\Phi(\bar{t}^*) \right) \E_{z \sim N(0,1)} \left(\sigma(z-\bar{t}^*) - T_a\sigma(z-\bar{t}^*) \right)  \Ind \{z>\bar{t}_i\} \\
    = &\left( r^*/\Phi(\bar{t}^*) \right) \E_{z \sim N(0,1)} \left(\int_a^1 \frac{d T_s \sigma(z-\bar{t}^*)}{ds} ds \right) \sigma'(z-\bar{t}_i)  = \left( r^*/\Phi(\bar{t}^*) \right) \int_a^1 \E_{z \sim N(0,1)}  \frac{d T_s \sigma(z-\bar{t}^*)}{ds}  \sigma'(z-\bar{t}_i)  ds \\
    = & \left( r^*/\Phi(\bar{t}^*) \right) \int_a^1 \E_{z \sim N(0,1)}  \frac{1}{s} L T_s \sigma(z-\bar{t}^*) (\sigma'(z-\bar{t}_i))  ds = \left( r^*/\Phi(\bar{t}^*) \right) \int_a^1 \E_{z \sim N(0,1)}  \frac{1}{s} (L T_s \sigma(z-\bar{t}^*))' (\sigma'(z-\bar{t}_i))'  ds \\
    = & \left( r^*/\Phi(\bar{t}^*) \right) \int_a^1 \E_{z \sim N(0,1)}   T_s \sigma'(z-\bar{t}^*) \left(\delta(z-\bar{t}_i)\right)  ds \\
    = & \left( r^*/\Phi(\bar{t}^*) \right) \int_0^\theta \sin s \E_{z \sim N(0,1)}   T_{\cos s} \sigma'(z-\bar{t}^*) \left(\delta(z-\bar{t}_i)\right)  ds \le O( \left( r^*/\Phi(\bar{t}^*) \right)b^2\psi(\bar{t}_i)) \le O (r^*b) \;.
\end{align*}
This gives $\E F_i^* = \E_{z\sim N(0,I)} \left(\sigma(r_i z - t_i)- \sigma(r^* z - t^*)\right)\Ind\{w^{(i)}\cdot x >\bar{t}_i\}/\Phi(\bar{t}_i) + cbr^*$ for some suitably small constant $c$.

As a summary, we have the following:  
\begin{proposition}\label{prop mean uf}
    There exist small constants $c_1,c_2>0$ such that,
    \begin{align*}
  &  \E U_i^* = \E_{z\sim N(0,I)} \left(\sigma(r_i z - t_i)- \sigma(r^* z - t^*)\right)z\Ind\{z >\bar{t}_i\}/\Phi(\bar{t}_i) + c_1 r^*b \\
   & \E F_i^* = \E_{z\sim N(0,I)} \left(\sigma(r_i z - t_i)- \sigma(r^* z - t^*)\right)\Ind\{w^{(i)}\cdot x >\bar{t}_i\}/\Phi(\bar{t}_i) + cbr^* .
\end{align*}
\end{proposition}

\paragraph{Evaluation for Noise Terms}
We next evaluate the noise terms. 
For the noise term in $U_i$, we have 
\begin{align*}
    \abs{\E (U_i-U_i^*)} & = \abs{ \E_{x \sim N(0,I)}\left(y - h^*\right)(w^{(i)}\cdot x) \Ind(w^{(i)}\cdot x>\bar{t}_i)/\Phi(\bar{t}_i)} \\
    & \le \sqrt{\E_{x \sim N(0,I)}(y(x) - h^*(x))^2 }\sqrt{\E_{x \sim N(0,I)} \Ind\{w^{(i)} \cdot x > \bar{t}_i\}(x \cdot w^{(i)})^2} /\Phi(\bar{t}_i)\\
    &\le \abs{\sqrt{\eps}\sqrt{\E_{z\sim N(0,1)} z^2 \Ind(z>\bar{t}_i)}}/\Phi(\bar{t}_i) = \sqrt{\eps}\sqrt{\bar{t}_i\psi(\bar{t}_i)+\Phi(\bar{t}_i)}/\Phi(\bar{t}_i).
    \end{align*}
Here, in the last equation, we use the fact that $\E_{z\sim N(0,1)} z^2 \Ind(z>\bar{t}_i) = \bar{t}_i\psi(\bar{t}_i)+\Phi(\bar{t}_i)$.

Similarly, for the noise term in $F$, we have 
\begin{align*}
    \abs{\E (F_i-F_i^*)} & = \abs{ \E_{x \sim N(0,I)}\left(y - h^*\right) \Ind(w^{(i)}\cdot x>\bar{t}_i)/\Phi(\bar{t}_i)} \\
    & \le \sqrt{\E_{x \sim N(0,I)}(y(x) - h^*(x))^2 }\sqrt{\E_{x \sim N(0,I)} \Ind\{w^{(i)} \cdot x > \bar{t}_i\}} /\Phi(\bar{t}_i)\\
    &\le \abs{\sqrt{\eps}\sqrt{\E_{z\sim N(0,1)} \Ind(z>\bar{t}_i)}}/\Phi(\bar{t}_i) = \sqrt{\eps}\sqrt{\Phi(\bar{t}_i)}/\Phi(\bar{t}_i).
\end{align*}

As a summary, we have established the following: 
\begin{proposition}\label{prop }
    \begin{align*}
        \abs{\E (U_i-U_i^*)} \le \sqrt{\eps}\sqrt{\bar{t}_i\psi(\bar{t}_i)+\Phi(\bar{t}_i)}/\Phi(\bar{t}_i) \\
        \abs{\E (F_i-F_i^*)} \le \sqrt{\eps}\sqrt{\Phi(\bar{t}_i)}/\Phi(\bar{t}_i).
    \end{align*}
\end{proposition}

We remark that when $\bar{t}_i$ is close to $\bar{t}^*$, 
the noise rate of $U_i$ is larger than that of $F_i$ 
by a factor of $\bar{t}^*$. This is one of the central reasons 
why we need a small step size to update $(r,t)$, 
and can only guarantee  that $(r_T,t_T)$ is $O(\eps\polylog(R^2/\eps)/\Phi(\bar{t^*}))$ 
close to $(r^*,t^*)$ with gradient descent only.

\paragraph{Evaluation for Variance}
Finally, we examine the variance of $U_i$ and $F_i$. 
We start with $U_i$.  
Notice that $\E U_i^2 \le 2\E(U_i^*)^2 + 2\E(U_i^*-U_i)^2$.

Let $M>0$ be a threshold such that 
$\E_{z\sim N(0,1)} z^2 \Ind\{\abs{z}>M\} \le \eps/R$. 
Notice that if we set $y' = \sign(y) \min\{\abs{y},r_i M\}$, 
then we will only introduce at most $\eps$ error, 
since $\abs{r_i - r^*} \le r^*/\log(R/\eps)$. 
So, we can, without loss of generality, 
assume that $\abs{y} \le M$. In particular, 
for the ReLU activation, $M \le O(\sqrt{\log(R/\eps)}).$ 
Based on this, we obtain that: 
 \begin{align*}
     \E(U_i^*-U_i)^2 & = \E_{x \sim N(0,I)}\left(y(x) - h^*(x)\right)^2 (x \cdot w^{(i)})^2 \Ind\{w^{(i)} \cdot x > \bar{t}_i\}\Ind\{(x\cdot w^{(i)})\le M\}/\Phi(\bar{t}_i) \\ 
     &+ \E_{x \sim N(0,I)}\left(y(x) - h^*(x)\right)^2 (x \cdot w^{(i)})^2 \Ind\{w^{(i)} \cdot x > \bar{t}_i\}\Ind\{(x\cdot w^{(i)})> M\}/\Phi(\bar{t}_i) \\
     & \le \eps M^2/\Phi(t) + 4M^2 \E_{x \sim N(0,I)} (x \cdot w)^2 \Ind\{w \cdot x > t\}\Ind\{(x\cdot w)> M\}/\Phi(t_i) \\
     & \le \eps M^2/\Phi(\bar{t}_i) + 4\eps M^2/\Phi(\bar{t}_i) = \Tilde{O}(\eps/\Phi(\bar{t}_i)).
 \end{align*}
We next consider the variance of $\E (U^*_i)^2$.
We have 
\begin{align*}
    \E (U^*_i)^2 &\le 2 \E_{x \sim N(0,I)}(h(r_i,w^{(i)},t_i) - h(r^*,w^{(i)},t^*))^2(w^{(i)}\cdot x)^2\Ind\{w^{(i)}\cdot x >\bar{t}_i\}/\Phi(\bar{t_i}) \\
    &+  2\E_{x \sim N(0,I)}(h(r^*,w^{(i)},t^*) - h^*)^2(w^{(i)}\cdot x)^2\Ind\{w^{(i)}\cdot x >\bar{t}_i\}/\Phi(\bar{t_i})
\end{align*}
We separately evaluate the two terms above. First, we have
\begin{align*}
    &\E_{x \sim N(0,I)}(h(r_i,w^{(i)},t_i) - h(r^*,w^{(i)},t^*))^2(w^{(i)}\cdot x)^2\Ind\{w^{(i)}\cdot x >\bar{t}_i\}/\Phi(\bar{t_i})\\
    = &\E_{z \sim N(0,1)}(r_i\sigma(z-\bar{t}_i) -r^*\sigma(z-t^*))^2 z^2\Ind\{z >\bar{t}_i\}/\Phi(\bar{t_i})
\end{align*}
For the second term, recall that we decompose $w^* = a w^{(i)} + bu$. Notice that
\begin{align*}
  \left(  \sigma(w^{(i)}\cdot x-\bar{t}^*)-\sigma(w^*\cdot x-\bar{t}^*)\right)^2& \le \left( (w^{(i)} -w^*) \cdot x\right)^2 = \left( (1-a)(w^{(i)}\cdot x) +  b(u\cdot x))\right)^2\\
  &\le 2(1-a)^2(w^{(i)}\cdot x)^2 +2b^2(u\cdot x)^2
\end{align*}
This implies 
\begin{align*}
    &\E_{x \sim N(0,I)}(h(r^*,w^{(i)},t^*) - h^*)^2(w^{(i)}\cdot x)^2\Ind\{w^{(i)}\cdot x >\bar{t}_i\}/\Phi(\bar{t_i})\\ 
    = &\E_{x \sim N(0,I)}\left( 2(1-a)^2(w^{(i)}\cdot x)^2 +2b^2(u\cdot x)^2\right)
    (w^{(i)}\cdot x)^2\Ind\{w^{(i)}\cdot x >\bar{t}_i\}(r^*)^2/\Phi(\bar{t_i}) \\
    = &\E_{z,s \sim N(0,1)}\left( 2(1-a)^2(z)^2 +2b^2(s)^2\right)
    (z)^2\Ind\{z >\bar{t}_i\}(r^*)^2/\Phi(\bar{t_i}) \\
    \le & O(\E_{z \sim N(0,1)}b^4z^4\Ind\{z >\bar{t}_i\}(r^*)^2/\Phi(\bar{t_i}) + \E_{z \sim N(0,1)}b^2z^2\Ind\{z >\bar{t}_i\}(r^*)^2/\Phi(\bar{t_i})) \;. \\
\end{align*}

We evaluate $F$ in a similar way. First, we have 
\begin{align*}
    \E(F_i^*-F_i)^2 \le \eps/\Phi(\bar{t}_i).
\end{align*}
Next, we have 
\begin{align*}
    \E (F^*_i)^2 &\le 2 \E_{x \sim N(0,I)}(h(r_i,w^{(i)},t_i) - h(r^*,w^{(i)},t^*))^2\Ind\{w^{(i)}\cdot x >\bar{t}_i\}/\Phi(\bar{t_i}) \\
    &+  2\E_{x \sim N(0,I)}(h(r^*,w^{(i)},t^*) - h^*)^2\Ind\{w^{(i)}\cdot x >\bar{t}_i\}/\Phi(\bar{t_i})
\end{align*}
The first term can be evaluated as 
\begin{align*}
    &\E_{x \sim N(0,I)}(h(r_i,w^{(i)},t_i) - h(r^*,w^{(i)},t^*))^2\Ind\{w^{(i)}\cdot x >\bar{t}_i\}/\Phi(\bar{t_i})\\
    = &\E_{z \sim N(0,1)}(r_i\sigma(z-\bar{t}_i) -r^*\sigma(z-t^*))^2 \Ind\{z >\bar{t}_i\}/\Phi(\bar{t_i})
\end{align*}
For the second term, we have
\begin{align*}
    &\E_{x \sim N(0,I)}(h(r^*,w^{(i)},t^*) - h^*)^2\Ind\{w^{(i)}\cdot x >\bar{t}_i\}/\Phi(\bar{t_i})\\ 
    = &\E_{x \sim N(0,I)}\left( 2(1-a)^2(w^{(i)}\cdot x)^2 +2b^2(u\cdot x)^2\right)
    \Ind\{w^{(i)}\cdot x >\bar{t}_i\}(r^*)^2/\Phi(\bar{t_i}) \\
    = &\E_{z,s \sim N(0,1)}\left( 2(1-a)^2(z)^2 +2b^2(s)^2\right)
    (z)^2\Ind\{z >\bar{t}_i\}(r^*)^2/\Phi(\bar{t_i}) \\
    \le & O(\E_{z \sim N(0,1)}b^4z^2\Ind\{z >\bar{t}_i\}(r^*)^2/\Phi(\bar{t_i}) + \E_{z \sim N(0,1)}b^2\Ind\{z >\bar{t}_i\}(r^*)^2/\Phi(\bar{t_i})) \;. 
\end{align*}
As a summary, we have established: 
\begin{proposition}\label{prop variance}
    \begin{align*}
        \E U_i^2 \le \tilde{O}\left(\E_{z \sim N(0,1)}(r_i\sigma(z-\bar{t}_i) -r^*\sigma(z-t^*))^2 z^2\Ind\{z >\bar{t}_i\}/\Phi(\bar{t_i}) + b^2(r^*)^2 + \eps/\Phi(\Bar{t_i})\right) \\
        \E F_i^2 \le \tilde{O}\left(\E_{z \sim N(0,1)}(r_i\sigma(z-\bar{t}_i) -r^*\sigma(z-t^*))^2 \Ind\{z >\bar{t}_i\}/\Phi(\bar{t_i}) + b^2(r^*)^2 + \eps/\Phi(\Bar{t_i})\right).
    \end{align*}
\end{proposition}

\subsubsection{Progress on $(r,t)$ Updates}\label{app tr progress}
In this section, we describe how to update the parameters $(r,t)$ 
so that we are able to make the term 
$\E_{z\sim N(0,I)}\left( \sigma(r_iz-t_i)- \sigma(r^*z-t^*)\right)^2$ stably drop every time we choose to update the parameters. 
For convenience, in this section, 
we define the following $2$-dimensional vector: 
\begin{align*}
    g^{(i)} & =\left(\E_{z\sim N(0,I)} \left((r_i z - t_i)- (r^* z - t^*)\right)z\Ind\{z >\bar{t}^*\},-\E_{z\sim N(0,I)} \left((r_i z - t_i)- (r^* z - t^*)\right)\Ind\{z >\bar{t}^*\}\right)^\top \\
    & = \left( (r_i-r^*)\delta(\bar{t}^*) - (t_i-t^*)\psi(\bar{t}^*), -(r_i-r^*)\psi(\bar{t}^*) + (t_i-t^*)\Phi(\bar{t}^*) \right)^\top
\end{align*}
and function
\begin{align*}
   & Z(r,t):= \E_{z \sim N(0,1)}\left((rz-t)-(r^*z-t^*)\right)^2\Ind(z>\bar{t}^*)\\
   = & (r-r^*)^2\delta(\bar{t}^*)+(t-t^*)\Phi(\bar{t}^*)-2(r-r^*)(t-t^*)\psi(\bar{t}^*) \\
    & W(r,t) := \E_{z \sim N(0,1)}\left(\sigma(rz-t)-\sigma(r^*z-t^*)\right)^2 \;.
\end{align*}
To simplify the notation, we define 
\begin{align*}
    Q:= 
    \begin{pmatrix}
        \delta(\bar{t}^*) & -\psi(\bar{t}^*) \\
        -\psi(\bar{t}^*) & \Phi(\bar{t}^*)
    \end{pmatrix} \;,
\end{align*}
and define $\Delta(r,t) = ((r-r^*),(t-t^*))^\top$. 
By definition, $Z(r,t) = \Delta^\top Q \Delta$. 
In particular, the largest eigenvalue of $Q$ is 
$O((\bar{t}^*)^2 \Phi(\bar{t}^*))$, 
while the smallest eigenvalue of $Q$ is $\Omega(\Phi(\bar{t}^*))/(\bar{t}^*)^4)$. 
Thus, the ratio of the largest and smallest eigenvalues of $Q$ 
is at most $\log^3(R^2/\eps)$. Furthermore, we want to mention that, 
when $\bar{t}^*$ becomes large, $Q$ becomes ill-conditioned; 
this is the reason why in \Cref{alg query learning}, 
we update $(r_i,t_i)$ with a much smaller rate 
of $1/\polylog(R^2/\eps)$ and the gradient descent stage 
can only guarantee that $(r_T,t_T)$ is 
$\eps\polylog(R^2/\eps)/\Phi(\bar{t^*})$ 
close to the optimal parameters $(r^*,t^*)$.

We remark that $g^{(i)}$ is exactly the gradient of $Z$ 
at point $(r_i,t_i)$. The motivation of using this notation 
is that when $(r,t)$ is close to $(r^*,t^*)$, 
$Z(r,t)$ and $W(r,t)$, the quantities we want to control, 
are different by a negligible factor. 
We give the following proposition.
\begin{proposition}\label{prop WF}
    Suppose $\bar{t}^* \le \sqrt{\log(R^2/\eps)}$ and $r^* \le R$, given a pair of parameter $(r,t)$ such that $\abs{r-r^*} \le r^*/\polylog(R^2/\eps)$ and $\abs{\bar{t}-\bar{t}^*} \le 1/\polylog(R^2/\eps)$, we have 
    \begin{align*}
        \abs{W(r,t) - Z(r,t)} = o(Z(r,t)).
    \end{align*}
\end{proposition}
\begin{proof}
Since $W(r,t)$ and $Z(r,t)$ are only different at region $\{z \mid \abs{\Ind(z>\bar{t}_i-\Ind(z>\bar{t}^*)}>0\} $, when $(r,t)$ are close to $(r^*,t^*)$,
we have 
\begin{align*}
  \abs{  W(r,t) - Z(r,t)} &\le \left((r-r^*)^2+(t-t^*)^2\right)\Pr_{z}(z \in \{z \mid \abs{\Ind(z>\bar{t}_i-\Ind(z>\bar{t}^*)}>0\}) \\
  &\le \left((r-r^*)^2+(t-t^*)^2\right)\Phi(\bar{t}^*)/\polylog(R^2/\eps) \le Z(r,t)/\polylog(R^2/\eps).
\end{align*}
Here, in the last inequality, we use the fact that the smallest eigenvalue of $Q$ is $\Omega(\Phi(\bar{t^*})/(\bar{t^*}^4))$ and thus $Z(r,t) \ge \Omega(\Delta(r,t)^2/\polylog(R^2/\eps))$
\end{proof}
Such a proposition can also be generalized to the gradient of $F$ and $W$. We next give the following proposition.
\begin{proposition}\label{prop gradient}
    Suppose $\bar{t}^* \le \sqrt{\log(R^2/\eps)}$ and $r^* \le R$, given a pair of parameter $(r,t)$ such that $\abs{r-r^*} \le r^*/\polylog(R^2/\eps)$ and $\abs{\bar{t}-\bar{t}^*} \le 1/\polylog(R^2/\eps)$, we have 
    \begin{align*}
        \norm{\nabla Z(r,t) -\nabla W(r,t) } \le o(\norm{\Delta(r,t)}\Phi(\bar{t})).
    \end{align*}
\end{proposition}
\begin{proof}[Proof of \Cref{prop gradient}]
    We notice that 
    \begin{align*}
        \nabla_r W(r,t) & = 2 \E_{z\sim N(0,1)} \left(\sigma(rz-t)-\sigma(r^*z-t^*)\right)z\Ind(z>\bar{t}) \\
        \nabla_t W(r,t) & = -2 \E_{z\sim N(0,1)} \left(\sigma(rz-t)-\sigma(r^*z-t^*)\right)\Ind(z>\bar{t})
    \end{align*}
Since $W(r,t)$ and $Z(r,t)$ are only different at region $\{z \mid \abs{\Ind(z>\bar{t}_i-\Ind(z>\bar{t}^*)}>0\} $, when $(r,t)$ are close to $(r^*,t^*)$, when 
\begin{align*}
    \norm{\nabla Z(r,t) - \nabla W(r,t)}^2 & \le ((r-r^*)^2+(t-t^*)^2)\bar{t}^2\Pr_{z}(z \in \{z \mid \abs{\Ind(z>\bar{t}_i-\Ind(z>\bar{t}^*)}>0\})^2 \\
    & \le \Delta^2\Phi(\bar{t})^2/\polylog(R^2/\eps)\;.
\end{align*}    
\end{proof}

Before formally presenting the analysis for updating the 
parameters $(r,t)$, we present the following standard 
gradient descent analysis for analyzing quadratic minimization.

\begin{proposition}\label{prop gdq}
    Let $Q \in \R^{2 \times 2}$ be a positive definite matrix and let $L,\mu$ be the largest eigenvalue and smallest eigenvalue of $Q$. For $x \in \R^2$, define $F(x)=x^\top Q x$. For $0<c<(\mu/L)/C$, for a large constant $C$ , given $g' \in \R^2$ such that $\norm{g'-\nabla F} \le c\norm{\nabla F}$, let $0<\eta \le O(1/L)$, then the update $x' = x - \mu g'$ satisfies $F(x') \le (1-O(\mu\eta))F(x)$ and $\norm{x'} \le (1-O(\eta\mu)) \norm{x}$. 
\end{proposition}
\begin{proof}[Proof of \Cref{prop gdq}]
The proof of \Cref{prop gdq} follows the standard gradient descent analysis for $L$ smooth functions that satisfy the 
Polyak-Lojasiewicz condition with parameter $\mu$. We have 
\begin{align*}
   & \norm{\nabla F}^2 \ge 2\mu F(x) \\
   & F(x') \le F(x) - \eta \nabla F(x) \cdot g' + L\eta^2\norm{g'}^2/2 \;.
\end{align*}
Since 
\begin{align*}
\nabla F(x)\cdot g' \ge \|\nabla F(x)\|^2 - \|\nabla F(x)\|\cdot \|g'-\nabla F(x)\| \geq (1 - c) \|\nabla F(x)\|^2,
\end{align*}
and 
\begin{align*}
    \norm{g'} \le (1+c) \|\nabla F(x)\|,
\end{align*}
we have 
\begin{align*}
    F(x') & \leq F(x) - \eta (1 - c) \|\nabla F(x)\|^2 + \frac{L \eta^2}{2} (1 + c)^2 \|\nabla F(x)\|^2 \\
    & \le \left[ 1 - 2\mu \left( \eta(1 - c) - \frac{L \eta^2}{2} (1 + c)^2 \right) \right] F(x).
\end{align*}
This implies that by choosing $\eta = O(1/L)$, we have $F(x') \le (1-O(\mu\eta)) F(x)$.

We next show that the update also contracts the parameter distance.
\begin{align*}
    \norm{x'} &= \norm{x-\eta g'} = \norm{x -\eta \nabla F(x) + \eta (\nabla F(x)-g')} \le \norm{x -\eta \nabla F(x)} + \eta \norm{\nabla F(x)-g'} \\
    & \le \norm{x -\eta F(x)} + \eta c \norm{F(x)} \le (1-\eta \mu) \norm{x} + c\eta L \norm{x} \le (1-O(\eta\mu)) \norm{x}.
\end{align*}
Here, the last inequality follows the fact that $c \le (\mu/L)/C$.
\end{proof}

In our setting, we have that the parameters $\mu,L$ for $Q$ 
are $\Omega(\Phi(\bar{t}^*)/(\bar{t}^*)^4)$ and $O((\bar{t}^*)^2\Phi(\bar{t}^*))$. 
Thus, after zooming in to the region $\Ind(z>\bar{t})$, 
the rescaled parameters become 
$\mu = \Omega(1/(\bar{t}^*)^4)$ and $L= O((\bar{t}^*)^2)$.
By choosing the step size $\mu=1/\polylog(R^2/\eps)$, 
we are able to drop $Z(r,t)$, and thus $Z(r,t)$, 
by a factor of $1-1/\polylog(R^2/\eps)$ in each round of update. Furthermore, after the update $(r,t)$, 
is still close to $(r^*,t^*)$.
So far, we have already stated all ingredients 
for analyzing the update for parameters $(r,t)$. 
We present the following lemma to conclude this section.

\begin{lemma}\label{lm tr progress}
Consider the problem of agnostic ReLU regression with queries. 
Let $h(r_i,w^{(i)},t_i)$ be a ReLU activation function. 
Write $w^* = a w^{(i)} + bu$, where $a,b>0, a^2+b^2 = 1$, 
$u \in \s^{d-1}, u \perp w^{(i)}$. Suppose that 
$\abs{\bar{t}_i-\bar{t}^*} \le 1/\polylog(R^2/\eps)$, 
$b \le 1/\bar{t_i}$ and $\abs{r_i-r^*} \le r^*/\polylog(R^2/\eps)$. 
Let $B_i>0$ and define $B_{i+1}=(1-1/C)B_i$ 
for some universal constant $C$. Suppose $B_i,B_{i+1}$ satisfy 
the following properties: 
\begin{enumerate}[leftmargin=*]
    \item $(r^*)^2\sin^2\theta_i\le B_{i+1}^2$,
    \item $B_i^2 \Phi(\bar{t}^*) \ge \Omega(\eps)$,
    \item $W(r_i,t_i)\le B_i^2\Phi(\bar{t}^*)\polylog(R^2/\eps)$.
\end{enumerate}
If $W(r_i,t_i) \ge  B_{i+1}^2\Phi(\bar{t}^*)\polylog(R^2/\eps)$, then 
there is an algorithm that makes $\polylog(R^2/\eps)$ queries, 
runs in $\polylog(R^2/\eps)$ time, 
and outputs a pair of $(r_{i+1},t_{i+1})$ 
such that 
$W(r_{i+1},t_{i+1}) \le  B_{i+1}^2\Phi(\bar{t}^*)\polylog(R^2/\eps)$. 
\end{lemma}

\begin{proof}[Proof of \Cref{lm tr progress}]
For a parameter $(r,t)$, 
if $W(r,t) \ge B_{i+1}^2\Phi(\bar{t}^*)\polylog(R^2/\eps)$, 
by \Cref{prop WF}, it must be the case that 
$Z(r,t) \ge B_{i+1}^2\Phi(\bar{t}^*)\polylog(R^2/\eps)$, 
which implies that $\norm{\nabla Z(r,t)}^2_2\ge B_{i+1}^2\Phi(\bar{t}^*)\polylog(R^2/\eps)$. 
By \Cref{prop gradient}, this implies that 
$\norm{\nabla W(r,t)}^2 \ge B_{i+1}^2\Phi(\bar{t}^*)\polylog(R^2/\eps)$. 
On the other hand, by \Cref{prop mean uf}, we know that when 
$(r^*)^2\sin^2\theta_i\le B_{i+1}^2$, 
$$\norm{(\E U(r,t),\E F(r,t)) - \nabla W(r,t)} \le o(\norm{\nabla W(r,t)})\;.$$ 
This implies that $\norm{(\E U(r,t),\E F(r,t))- \nabla Z(r,t)/\Phi(\bar{t}^*)} \le o(\norm{\nabla Z(r,t)/\Phi(\bar{t}^*)})$. 
By \Cref{prop variance}, with $\polylog(R^2/\eps)$ queries, 
we are able to estimate $(\E U(r,t),\E F(r,t))$ with error 
$ \norm{(\E U(r,t),\E F(r,t))}/\polylog(R^2/\eps)$;  
and thus this is different from $$\nabla Z(r,t)/\Phi(\bar{t}^*)$$ 
by 
$\norm{\nabla Z(r,t)/\Phi(\bar{t}^*)}/\polylog(R^2/\eps)$. 
By \Cref{prop gdq}, we know that by running gradient descent 
with step size $O(1/\polylog(R^2/\eps))$, 
we are able to drop $Z(r,t)$ be a factor 
of $(1-1/\polylog(R^2/\eps))$ in each round; 
and thus after $\polylog(R^2/\eps)$ rounds of iterations, 
we have 
$$Z(r_{i+1},t_{i+1}) \le B_{i+1}^2\Phi(\bar{t}^*)\polylog(R^2/\eps) \;.$$ 
Therefore, by \Cref{prop WF}, 
we have $W(r_{i+1},t_{i+1}) \le  B_{i+1}^2\Phi(\bar{t}^*)\polylog(R^2/\eps)$.
\end{proof}

\subsection{Proof of \Cref{th query learning}}\label{app query learning}
In this section, we give the proof of \Cref{th query learning}. 

\begin{proof}[Proof of \Cref{th query learning}]    
We first notice that $p< \eps/R^2$, 
otherwise the $0$ function has error $O(\eps)$. 
Thus, by \Cref{lm Ini}, with $\min\{1/p,R^2/\eps\}$ queries, 
we obtain a warm start $w^{(0)}$ such that 
$\theta_0 \le O(1/\bar{t}^*)$. Furthermore, 
by the construction of the grid $(r,t)$, 
if we randomly chose a parameter $(r_0,t_0)$ from the grid, 
with a non-trivial probability$(1/\polylog(R^2/\eps))$, 
we have $\abs{r_0-r^*} \le r^*/\polylog(R^2/\eps)$ 
and $ \abs{\bar{t}_0-\bar{t}^*} \le 1/\polylog(R^2/\eps)$. 

With the pair of $(r_0,w^{(0)},t_0)$, 
we run \Cref{alg query learning}. 
We prove by induction that in each round of the update 
$W(r_i,t_i) \le B_i^2\Phi(\bar{t}^*)\polylog(R^2/\eps)$ 
and $(r^*)^2\sin^2\theta_i \le B_{i+1}^2$. 
Notice that for $i=0$ this holds automatically 
by the property of $(r_0,w^{(0)}),t_0$ as a warm start. 

Now suppose that this holds in the $i$th round. 
We will show this holds for the $i+1$th round. 
By induction, we have $(r^*)^2\sin^2\theta_{i+1} \le B_{i+1}^2$. 
By \Cref{prop mean uf} and \Cref{prop gradient}, 
we must have 
$\norm{(\E U(r_i,t_i),\E F(r_i,t_i))- \nabla F(r_i,t_i)/\Phi(\bar{t}^*)} \le o(\norm{\nabla Z(r,t)/\Phi(\bar{t}^*)})$. 
By \Cref{prop variance}, with $\polylog(R^2/\eps)$ samples, 
we are able to estimate $(\E U(r_i,t_i),\E F(r_i,t_i))$ 
up to error $\norm{(\E U(r_i,t_i),\E F(r_i,t_i))}/\polylog(R^2/\eps)$. 
Since $Z(r,t)/\Phi(\bar{t}^*) \ge \norm{\nabla Z/\Phi(\bar{t}^*)}^2/(\bar{t}^*)^4$, 
this implies if $W(r_{i},t_i)>B_{i+1}^2\Phi(\bar{t}^*)\polylog(R^2/\eps)$, 
we are able to verify this by looking at the length of the gradient; 
and thus by \Cref{lm tr progress}, 
after making $\polylog(R^2/\eps)$ queries, 
we get $W(r_{i+1},t_{i+1}) \le B_{i+2}^2\Phi(\bar{t}^*)\polylog(R^2/\eps)$. 
After this by \Cref{lm angle progress}, 
by making $\tilde{O}(d)$ queries, 
we have $(r^*)^2\theta_{i+1}^2 \le B_{i+2}^2$. 

Now after $T=O(\log(R^2/\eps))$ rounds, we have 
$(r^*)^2\theta_{T}^2\Phi(\bar{t}^*) \le O(\eps)$, while $W(r_T,t_T) \le O(\eps \polylog(R^2/\eps)$. 
By \Cref{prop WF}, this implies 
$\norm{\Delta(r_T,t_T)}^2\Phi(\bar{t}^*) \le Z(r_T,t_T)\polylog(R^2/\eps) \le O(\eps \polylog(R^2/\eps))$. 
Notice that if a pair of $(r,t)$ satisfies 
$\norm{\Delta(r,t)}^2 \le \eps/ (\Phi(\Bar{t}_i)\polylog(R^2/\eps))$, 
then $W(r,t) \le O(\eps)$. This implies that if we consider the ball 
$B \subseteq \R^2$ centered at $(r_T,t_T)$ 
with radius $O(\sqrt{\eps\polylog(R^2/\eps)/\Phi(\bar{t}^*)})$, 
then $(r^*,t^*) \in B$. In particular, if we grid $B$ with a net $N$ 
of size $\polylog(R^2/\eps)$, there must be some $(r',t')$ 
that is  $\sqrt{\eps/ (\Phi(\Bar{t}_i)\polylog(R^2/\eps))}$ close 
to $(r',t')$. By uniformly sampling from the grid, 
with probability at least $1/\polylog(R^2/\eps)$, 
we get such $(r',t')$ with $W(r',t') \le O(\eps)$. 
By \Cref{lm realizable error}, we know that $h(r',w^{(T)},t')$ 
has noiseless error at most $O(\eps)$, 
and thus $\err(h(r',w^{(T)},t')) = O(\eps)$. 
Furthermore, the total number of queries we use is $\Tilde{O}_\delta(\min\{1/p, R^2/\eps\} + d\cdot\polylog(R^2/\eps))$. 
We remark that the current algorithm has a probability of success of $1/\polylog(R^2/\eps)$. Thus, 
by running it $\polylog(R^2/\eps)$ times, 
we get a list of $\polylog(R^2/\eps)$ hypothesis, 
one of which has error $O(\eps)$. 
By doing a hypothesis selection procedure using \Cref{lm test}, 
we are able to select a hypothesis with $O(\eps)$ error 
with high probability, which will only cost 
another $\polylog(R^2/\eps)$ queries.
\end{proof}

\section{Omitted Proofs from \Cref{sec lower}}\label{app lb}

\subsection{Proof of \Cref{th lb information}}

For convenience, we restate the theorem below. 

\begin{theorem}\label{th lb information re}
   Consider the problem of agnostic ReLU regression with queries with a restriction that the optimal ReLU satisfies $\norm{W^*}\le 1$ and has bias at least $p$. Any learning algorithm that outputs a hypothesis with error less than $O(p/\log^2(p))$ with probability $1/3$, must make at least $\Tilde{\Omega}(1/p^{1-o(1)} +d)$ queries. Furthermore, this holds even if $\opt \le 2^{-\Omega(d^{1/4})}p$.
\end{theorem}

\begin{proof}
We break down the proof into two parts. 
First, we show a lower bound of $\Tilde{\Omega}(1/p)$. 
Consider two hypotheses $h_1(x) = \sigma(w^* \cdot x -t^*)$, 
where $w^*$ is drawn uniformly from $\s^{d-1}$ and $h_2(x) = 0$. 
Notice that 
\begin{align*}
\E_{x\sim N(0,I)}(h_1(x)-h_2(x))^2 = V(t^*)= (t^*)^2 \Phi(t^*) - t^*\psi(t^*) \ge \Phi(t^*)/(t^*)^2 = \Tilde{\Omega}(p).
\end{align*}
Thus, any learning algorithm that can learn a hypothesis 
with error $\Tilde{O}(p)$ can distinguish whether 
the target hypothesis is $h_1$ or $h_2$. 
We construct adversarial label noise as follows, 
when $h_1$ is the underlying hypothesis. 
For every example such that $\norm{x}^2>d+\Delta$, 
where $\Delta=d^{\alpha}$, $0<\alpha<1$, 
the adversary changes its label by $y(x) = 0$. 
We first show that the noise level is small. We have 
\begin{align*}
\E_{x\sim N(0,I)} (h_1(x)-y)^2 & \le \E_{x\sim N(0,I)} h^2_1(x)\Ind\{w^*\cdot x>t^*, \norm{x}^2>d+\Delta\} \\
& \le \Tilde{O}(p\exp(-\Omega(d^{2\alpha-1})),
\end{align*}
where in the last inequality, we use the tail bound 
for $\chi^2$-distribution. 
Here, we choose $\alpha=5/8$ to make $2\alpha-1=1/4$.

This implies that by querying examples with norm larger than 
$\sqrt{d+\Delta}$, a learner will get no information. 
Now we consider a deterministic learner that makes $r$ queries 
over the ball with radius $\sqrt{d+\Delta}$. 
Since $w^*$ is drawn uniformly from the unit sphere, 
we know that for each realization of $w^*$, 
only examples in $R:=\{x \mid w^* \cdot x > t^*\}$ 
have non-zero labels. This implies that the number 
of queries that fall into this region is 
\begin{align*}
    r \Pr_{x\sim B^d(\sqrt{d+\Delta})}(x_1 \ge t^*) \le \tilde{O}(r \frac{1}{t^*}\exp(-\frac{(t^*)^2}{2(1+d^{\alpha-1})}) = \Tilde{O}(rp^{1-o(1)}),
\end{align*}
for $\alpha \in (0,1)$.
Thus, unless $r \ge \Omega(1/p^{1-o(1)})$, 
no query will have a non-zero response, 
and thus it is impossible to distinguish 
whether the ground truth is $h_1$ or $h_2$.

We next establish the lower bound on $d$. 
We show that this even holds for learning 
a homogeneous ReLU with error $\Omega(1)$ 
in the realizable setting. 
We consider an even simpler model in the realizable setting, 
where $w^* \in \s^{d-1}$ and $t^*=0$. Furthermore, 
for every query $x$ made by the learner, 
we additionally provide the information $w^* \cdot x$. 
Denote by $L$ the subspace spanned by $x^{(1)},\dots,x^{(r)}$, 
the queries made by the learner. We assume that 
$r \le d/\log(d)$.
Denote by $w^*_L = \proj_L (w^*)$. 
Since $w^*\sim \s^{d-1}$, given $w_L^*$, 
we know that the orthogonal component $w_{L^\perp}^*$ 
is a random vector drawn from the ball in $L^\perp$ 
with radius $\sqrt{1-\norm{w^*_L}^2}$. 
Let $D$ be the distribution of $w_{L^\perp}^*$. 
Now for any fixed hypothesis $h$, 
we consider the expected error of $h$ 
when $w_{L^\perp}^* \sim D$. We have 
\begin{align*}
    \E_{w_{L^\perp}^*\sim D} \E_{x \sim N(0,I)} (h(x) - \sigma(w^*_L \cdot x + w_{L^\perp}^*))^2 &=  \E_{x \sim N(0,I)}\E_{w_{L^\perp}^*\sim D} (h(x) - \sigma(w^*_L \cdot x + w_{L^\perp}^*))^2 \\
    & \ge \E_{x \sim N(0,I)} \Var_{w_{L^\perp}^*\sim D}(\sigma(w^*_L \cdot x + w_{L^\perp}^*\cdot x)) \\
    & = \E_{x \sim N(0,I)} \Var_{w_{L^\perp}^*\sim D}(\sigma(w^*_L \cdot x_L + w_{L^\perp}^*\cdot x_{L^\perp})) \;.
\end{align*}
To simplify the notation, we denote by $T=w^*_L\cdot x_L$ 
the threshold of the ReLU, 
$\sigma(w^*_L \cdot x_L + w_{L^\perp}^*\cdot x_{L^\perp})$. 
Notice that since $r\le d/\log(d)$, with probability at least $2/3$, 
$\norm{w^*_L} \le 1/\log(d)$ and $\abs{T} \le 1/\log(1/d)$. Denote by 
$D'$ the uniform distribution over the ball in $L^\perp$ 
with radius $r=\norm{\sqrt{1-\norm{w^*_L}^2}}\norm{x_{L^\perp}}$. 
For every fixed $x_{L^\perp}$, we have 
\begin{align*}
    \Var_{w_{L^\perp}^*\sim D}(\sigma(T + w_{L^\perp}^*\cdot x_{L^\perp})) = \Var_{x\sim D'}(\sigma(T + e_1\cdot x)),
\end{align*}
where $e_1$ is the first standard basis vector in $L^\perp$. 
Since $L$ has dimension $d(1-O(1/\log(d)))$, 
we know that when $\norm{x_{L^\perp}} \ge \Omega(\sqrt{d})$, 
and thus $r>\Omega(\sqrt{d})$, 
$ \Var_{x\sim D'}(\sigma(T + e_1\cdot x)) \ge \Omega(1)$. 
Since $x_{L^\perp}$ and $x_L$ are independent, 
we know that with probability at least $2/3$, 
$\norm{x_{L^\perp}} \ge \Omega(\sqrt{d})$. 
Thus, for every learning algorithm, 
if it makes fewer than $d/\log(d)$ queries,
with probability at least $2/3$, the expected error 
of the output hypothesis is at least 
$\E_{w_{L^\perp}^*\sim D} \E_{x \sim N(0,I)} (h(x) - \sigma(w^*_L \cdot x + w_{L^\perp}^*))^2 \ge \Omega(1)$.
\end{proof}

\subsection{Proof of \Cref{th lb compute}}\label{app lb compute}

We restate the theorem below. 

\begin{theorem}\label{th lb compute re}
For any active learning algorithm $\A$, there is an activation function $h^*$ that labels $S$ with bias $p$ such that if $\A$ makes less than $\Tilde{O}(d/(p\log(m)))$ label queries over $S$, a set of $m$ \iid points drawn from $N(0,I)$, then with probability at least $2/3$ the hypothesis $\hat{h}$ output by $\A$ has error more than $\Tilde{O}(p)$ with respect to $h^*$.
\end{theorem}

To begin with, we establish the following lemma that reduces the 
learning problem to a slightly easier problem of finding examples with 
non-zero labels from a pool of unlabeled examples.

\begin{lemma}\label{lm reduction lb}
Suppose there is an active learning algorithm that can make $r$ label queries 
over a pool $S$ of $m \ge \poly(d/p)$ examples drawn from $N(0,I)$ 
and learn any ReLU activation function $h^*(x) = \sigma(w^*\cdot x-t^*)$  with bias $p$ up to error $\Tilde{O}(p)$ with probability at least $2/3$. 
Then there is an algorithm such that given a pool of $2m$ random examples $S$ 
drawn from the standard Gaussian distribution with hidden labels 
by some ReLU activation function $h^*(x) = \sigma(w^*\cdot x-t^*)$ with bias $p$, 
it makes $r+O(d)$ queries and finds $d$ examples with non-zero labels. 
from $S$ with probability $1/2$.
\end{lemma}

\begin{proof}[Proof of \Cref{lm reduction lb}]
Let $\A$ be such a learning algorithm. We select a random set of $m$ examples $S_1$ and give it to $\A$. We know that with probability $1/2$, we learn a hypothesis $h$ such that $\E_{x\sim N(0,I)}\left(h(x)-h^*(x)\right)^2 \le \Tilde{O}(p)$. We first show that if $x\sim N(0,I)$, then with probability at least $\Omega(p)$, $h(x)>1/(2t)$. This is because otherwise
\begin{align*}
    &\E_{x\sim N(0,I)} \left(h(x)-h^*(x)\right)^2\Ind\{h(x)<1/(2t),h^*(x)>1/t\}\\
    \ge& \Omega(1/t^2) \Pr_{x \sim N(0,I)}\left(h(x)<1/(2t),h^*(x)>1/t\right) = \Tilde{\Omega}(p).
\end{align*}
On the other hand, if $x\sim N(0,I)$ and $h(x)>(1/2t)$, then with probability at least $1/2$, $h^*(x)>0$. Suppose this is not correct, we have 
\begin{align*}
   & \E_{x\sim N(0,I)} \left(h(x)-h^*(x)\right)^2\Ind\{h(x)>1/(2t),h^*(x)\le 0\} \\
   &\ge \Omega(1/t^2) \Pr_{x \sim N(0,I)}\left(h(x)>1/(2t),h^*(x)\le 0\right) = \Tilde{\Omega}(p).
\end{align*}
 Since $m$ is at least $\poly(d,1/p)$, we know that with enough high probability, at least $\Omega(d)$ examples will satisfy $h(x)>1/(2t)$ and at least a constant fraction of these examples will satisfy $h^*(x)>0$. Thus, given such a $h$ with probability at least $3/4$, we can find $d$  examples with non-zero label in $S$ by randomly querying $O(d)$ examples with prediction $h(x)>1/(2t)$.
\end{proof}

Based on this, we can give the proof of \Cref{th lb compute}.

\begin{proof}[Proof of \Cref{th lb compute}]
Consider the problem where an algorithm wants to find $k$ examples 
with non-zero labels from $m$ unlabeled examples by making $r$ queries. 
By \Cref{lm reduction lb}, it is sufficient for us to prove the hardness 
of such a problem by finding suitable parameters $k,r$.

Consider any deterministic learning algorithm $\A$. 
Given a pool of $m$ unlabeled examples, 
we describe $\A$ in the following way. 
For every $w^*$, the implementation of $\A$ can be described 
as a path $P = \left((x^{(1)},y^{(1)}),\dots,(x^{(r)},y^{(r)})\right)$, 
with length at most $r$. In particular, a path, along which $\A$ 
successfully finds $k$ examples with non-zero, can be uniquely 
represented as the indices $i_1<\dots<i_k$ and the 
corresponding label $(y^{(i_1)},\dots,y^{(i_k)})$. 
Next, we consider a fixed tuple of $k$ indices $i_1,\dots,i_k$ 
and denote by $S_c$ the set of all 
$(x^{(i_1)},y^{(i_1)}),\dots,(x^{(i_k)},y^{(i_k)})$ 
that make algorithm $\A$ succeed at positions $i_1<\dots<i_k$. 
By choosing $w^* \sim \s^{d-1}$ uniformly, 
we will show that the probability of $S_c$ is very small.
Notice that for any $(x^{(i_1)},y^{(i_1)}),\dots,(x^{(i_k)},y^{(i_k)})\in S_c$, 
it corresponds to a unique event $(x^{(i_1)},q^{(i_1)}),\dots,(x^{(i_k)},q^{(i_k)})$, 
where $q^{(i_j)} = w^* \cdot x^{(i_j)}$ 
for $j \in [k]$, 
since $y_{ij}>0$ and $q_{ij} = t^* + y_{ij}$. 
We denote the marginal density on this event as $f((x^{(i_1)},q^{(i_1)}),\dots,(x^{(i_k)},q^{(i_k)}))$. 
Notice that $\Pr(S_c) = \int_{q_{i_1},\cdot,q_{i_k}>t^*}f((x^{(i_1)},q^{(i_1)}),\dots,(x^{(i_k)},q^{(i_k)}))$. 
So, it is sufficient to upper bound this integral.

Denote by $L$ the subspace spanned by $\{x^{(i_1)},\dots,x^{(i_k)}\}$. 
We consider a basis $\{b_1,\dots,b_k\}$ of $L$ as follows. 
$b_1 = x^{(i_1)}/\norm{x^{(i_1)}}$ and 
$b_j = \proj_{L_{j-1}} x^{(i_j)}/\norm{\proj_{L_{j-1}} x^{(i_j)}}$, 
where $L_{j-1} = span\{x^{(i_1)},\dots,x^{(i_{j-1})}\}$. 
We know that there is a unique $w_L:=\proj_L(w^*)$ 
such that $w_L\cdot x^{(i_j)} = q^{(i_j)}$ for $j \in [k]$. 
This implies that events 
$(x^{(i_1)},y^{(i_1)}),\dots,(x^{(i_k)},y^{(i_k)})$ 
can be precisely described as a $k$-dimensional vector 
$v(w_L): = (w_L\cdot b_1,\dots,w_L\cdot b_k)$.
Denote by $L_0 = span(e_1,\dots,e_k)$. 
We have
\begin{align*}
    f((x^{(i_1)},q^{(i_1)}),\dots,(x^{(i_k)},q^{(i_k)})) =p(\proj_{L}(w^*) = w_L)=p(\proj_{L_0}(w^*) = v(w_L)) \;,,
\end{align*}
where $p(\proj_{L}(w^*) =  w_L)$ is the density function 
of the event that $\proj_{L}(w^*) =  w_L$, $w^* \sim \s^{d-1}$ 
Consider the set $S_w$ of all possible $v(w_L)$ 
such that $w_L$ corresponds to a realization $((x^{(i_1)},q^{(i_1)}),\dots,(x^{(i_k)},q^{(i_k)}))$. 
We notice that, given any $v(w_L) \in S_w$, 
we can uniquely recover the corresponding 
$(x^{(i_1)},q^{(i_1)}),\dots,(x^{(i_k)},q^{(i_k)}) \in S_c$ 
via $\A$. This implies
\begin{align*}
    \Pr\left(S_c\right)= \int_{q_{i_1},\cdot,q_{i_k}>t^*}f((x^{(i_1)},q^{(i_1)}),\dots,(x^{(i_k)},q^{(i_k)})) = \int_{S_w} p(\proj_{L_0}(w^*) = v(w_L)) = \Pr_{w^*}(S_w). 
\end{align*}
To upper bound $\Pr(S_w)$, it is sufficient to upper bound 
the probability of a super-set of $S_w$. 
Since $w^*\cdot x^{(i_j)} = q^{(i_j)}>t^*$ for $j \in [k]$, 
we know that 
\begin{align*}
    v:= (w^*\cdot x^{(i_1)},\dots, w^*\cdot x^{(i_k)})^\top
\end{align*}
satisfies 
$\norm{v}^2 = \sum_{j\in [k]}(q^{(i_j)})^2 \ge  k(t^*)^2$. 
This implies that, for every realization, 
the square of the norm of the projection of $w^*$ 
onto the subspace, $\norm{w_L}^2$, is 
\begin{align*}
   B&:= (w^*)^{\top}A^{\top}(AA^{\top})^{-1} Aw^*= v^{\top} (AA^{\top})^{-1} v \ge \norm{v}^2/\norm{AA^{\top}}_2 \\
   &\ge k(t^*)^2/\norm{AA^{\top}}_2 \;,
\end{align*}
where $A \in \R^{k \times d}$ is the matrix with row vectors $x^{(i_1)},\dots,x^{(i_k)}$.

We next make use of the following structural lemma from \cite{diakonikolas2024active}.
\begin{lemma}\label{lm high probability}
    Let $S \subseteq \R^d$ be a set of $m$ examples drawn \iid from $N(0,I)$. 
    Let $t^*>C>0$ for a sufficiently large constant $C$ and $k=O(d/\log(m)(t^*)^4)$. 
    Then, with probability at least $2/3$, for every $k$-tuple of examples $\{x_1,\dots,x_k\} \subseteq S$, $\norm{AA^{\top}-dI}_2 \le d/(t^*)^2$, 
    where $A \in \R^{k \times d}$ be a matrix with row vectors $x_1,\dots,x_k$.
\end{lemma}
That is to say, we are able to bound $\norm{AA^{\top}}_2$ 
by $d(1+O(1/(t^*)^2))$ and thus every vector $v(w_L)$ has squared norm at least $(k(t^*)^2/d(1+O(1/(t^*)^2)))$.
Since $w^*$ is uniformly chosen 
from the unit sphere, by Lemma B.1 in \cite{kontonis2024gain}, 
the square norm of $w^*$ projected onto a fixed $k$-dimensional 
subspace is a random variable drawn 
from a Beta distribution $B(\frac{k}{2},\frac{d-k}{2})$.
By Lemma 2.2 in \cite{dasgupta2003elementary}, we have
\begin{align*}
    \Pr\left(\norm{w^*_{L_0}}^2 \ge k(t^*)^2/d(1+O(1/(t^*)^2))) \right) & \le \exp\left(-\frac{k}{2}(\frac{d(t^*)^2}{d(1+O(1/(t^*)^2)))}-1-\log(\frac{d (t^*)^2}{d(1+O(1/(t^*)^2)))}))\right) \\
    & =\left( \sqrt{\frac{(t^*)^2d}{d(1+O(1/(t^*)^2)))}} \exp(-\frac{1}{2}(\frac{d(t^*)^2}{d(1+O(1/(t^*)^2)))}-1))\right)^k \\
    & \le \left(O( (t^*) \exp(-\frac{(t^*)^2}{2}(1-O(1/(t^*)^2)) ) \right)^k \\
    & = \left(O((t^*)^2 \frac{1}{t^*} \exp(-\frac{(t^*)^2}{2}))\right)^k
    \le \left(O(p\log(1/p))\right)^k.
\end{align*} 
This implies that by choosing $k=O(d/\log(m)(t^*)^4)$, 
for every $k$ tuple of indices, the probability of success 
for $\A$ is at most $\left(O(p\log(1/p))\right)^k$. 
Since there are at most
$\binom{r}{k}$ such tuples, by the union bound,
\begin{align*}
    \binom{r}{k}\alpha^k \le \left(\frac{er}{k} O(p\log(1/p))\right)^k \le 2/3 \;,
\end{align*}
if $r \le O(k/p\log(1/p)) = O(d/(p\log(m)\polylog(1/p))$.
This implies that if we make less than 
$r \le O(k/p\log(1/p)) = O(d/(p\log(m)\polylog(1/p))$ queries, 
then it is hard to find $k=O(d/\log(m)(t^*)^4)$.    
\end{proof}

\end{document}